\pdfoutput=1

\documentclass[11pt]{article}

\usepackage[preprint]{acl}
\usepackage{multirow}
\usepackage{float}
\usepackage{tabularx}
\usepackage{times}
\usepackage{latexsym}

\usepackage[T1]{fontenc}

\usepackage[utf8]{inputenc}
\usepackage{booktabs} 

\usepackage{microtype}
\usepackage{amsthm}
\usepackage{caption}
\usepackage{inconsolata}

\usepackage{enumitem}

\usepackage{graphicx}

\usepackage{xcolor}      
\usepackage{hyperref}    

\usepackage[table,xcdraw]{xcolor}

\usepackage{amssymb}
\usepackage{amsmath}

\usepackage{subcaption}
\usepackage{adjustbox}
\usepackage{hyperref}
\usepackage{cuted}
\usepackage{capt-of}

\usepackage{pifont}
\newcommand{\cmark}{\ding{51}}
\newcommand{\xmark}{\ding{55}}

\usepackage{xcolor}
\usepackage{colortbl}

\definecolor{openended}{RGB}{230, 240, 255}  
\definecolor{closeended}{RGB}{255, 245, 230}  
\usepackage[most]{tcolorbox}
\usepackage{listings}
\tcbuselibrary{listings,skins,breakable}

\newtcblisting{promptbox}[2][]{%
  enhanced,
  breakable,
  listing only,
  title={#2},
  colback=blue!3,          
  colbacktitle=blue!10,    
  colframe=blue!40,        
  coltitle=black,          
  fonttitle=\bfseries,
  boxrule=0.6pt,
  arc=2pt,
  left=6pt,right=6pt,top=4pt,bottom=4pt,
  listing options={%
    basicstyle=\ttfamily\footnotesize,
    breaklines=true,
    columns=fullflexible,
    showstringspaces=false
  },
  #1
}

\definecolor{darkblue}{HTML}{000099}

\newtheorem{assumption}{Assumption}

\newtheorem{corollary}{Corollary}
\newtheorem{lemma}{Lemma}
\newtheorem{theorem}{Theorem}

\title{DermoGPT: Open Weights and Open Data for Morphology-Grounded Dermatological Reasoning MLLMs}

\author{
\textbf{Jinghan Ru}\textsuperscript{1}\thanks{\,Equal Contribution.} \quad
\textbf{Siyuan Yan}\textsuperscript{2}\footnotemark[1]\thanks{\,Project Leader.} \quad
\textbf{Yuguo Yin}\textsuperscript{1} \quad
\textbf{Yuexian Zou}\textsuperscript{1}\footnotemark[3] \quad
\textbf{Zongyuan Ge}\textsuperscript{2}\thanks{\,Corresponding Authors: zongyuan.ge@monash.edu, zouyx@pku.edu.cn.}
\\[2mm]
\textsuperscript{1}School of Electronic and Computer Engineering, Peking University \\
\textsuperscript{2}Faculty of Information Technology, Monash University, Melbourne, Australia \\
}

\begin{document}
\addtocontents{toc}{\protect\setcounter{tocdepth}{-1}} 
\maketitle






\begin{abstract}
Multimodal Large Language Models (MLLMs) show promise for medical applications, yet progress in dermatology lags due to limited training data, narrow task coverage, and lack of clinically-grounded supervision that mirrors expert diagnostic workflows. We present a comprehensive framework to address these gaps. First, we introduce \textbf{DermoInstruct}, a large-scale morphology-anchored instruction corpus comprising 211,243 images and 772,675 trajectories across five task formats, capturing the complete diagnostic pipeline from morphological observation and clinical reasoning to final diagnosis. Second, we establish \textbf{DermoBench}, a rigorous benchmark evaluating 11 tasks across four clinical axes: \textit{Morphology}, \textit{Diagnosis}, \textit{Reasoning}, and \textit{Fairness}, including a challenging subset of 3,600 expert-verified open-ended instances and human performance baselines. Third, we develop \textbf{DermoGPT}, a dermatology reasoning MLLM trained via supervised fine-tuning followed by our Morphologically-Anchored Visual-Inference-Consistent (MAVIC) reinforcement learning objective, which enforces consistency between visual observations and diagnostic conclusions. At inference, we deploy Confidence-Consistency Test-time adaptation (CCT) for robust predictions. Experiments show DermoGPT significantly outperforms 16 representative baselines across all axes, achieving state-of-the-art performance while substantially narrowing the human-AI gap. \textit{DermoInstruct, DermoBench and DermoGPT will be made publicly available at \href{https://github.com/mendicant04/DermoGPT}{\textcolor{blue}{https://github.com/mendicant04/DermoGPT}} upon acceptance.}
\end{abstract}

\section{Introduction}
Skin diseases impose a substantial global burden, yet specialist access remains limited~\citep{hay2014global}. Dermatological diagnosis requires differentiating hundreds of fine-grained conditions across modalities via systematic clinical reasoning\citep{morphology1}. While Multimodal Large Language Models (MLLMs) show promise\citep{gemini,qwen3vl}, existing medical MLLMs~\citep{huatuogpt-v,skingpt,skinr1} struggle with dermatology's specialized requirements due to limited training data, narrow task scopes, and lack of interpretable reasoning mechanisms aligned with clinical practice.

As summarized in Table~\ref{tab:dermo_benchmark_comparison}, current resources exhibit three systemic limitations hindering clinical viability. \textbf{First, insufficient scale and diversity}: Existing resources like DermaSynth~\citep{dermasynth} and MM-Skin \citep{mmskin} typically cover only 2--3 tasks with limited samples. This scarcity fails to capture the long-tail visual complexity of the hundreds of conditions, severely limiting generalization. \textbf{Second, limited task formulations}: Existing instruction data and benchmarks predominantly rely on close-ended Multiple-Choice Question Answering (MCQAs)~\citep{dermavqa}, inadequate for evaluating open-ended generation and multi-step reasoning required in clinical consultations. \textbf{Third, ungrounded clinical reasoning}: Unlike end-to-end models \citep{panderm,make} that map pixels directly to labels, expert dermatologists adhere to a ``morphology-first'' paradigm, parsing lesion morphology attributes to construct reasoning chains before diagnosis~\citep{morphology1,morphology2}. Current datasets lack supervision for this \textit{morphology $\to$ reasoning $\to$ diagnosis} trajectory, yielding ungrounded systems prone to hallucinations inconsistent with visual evidence.

\begin{table*}[ht]
\footnotesize
  \centering
  \resizebox{0.95\textwidth}{!}{%
    \begin{tabular}{l|cc|rrr|cccc}
      \toprule
      \multirow{2}{*}{\textbf{Dataset / Benchmark}} &
      \multicolumn{2}{c|}{\textbf{Type}} &
      \multicolumn{3}{c|}{\textbf{Scale}} &
      \multicolumn{4}{c}{\textbf{Features}} \\
      \cmidrule(lr){2-3} \cmidrule(lr){4-6} \cmidrule(lr){7-10}
      &
      \textbf{Bench.} &
      \textbf{Train} &
      \textbf{\#Tasks} &
      \textbf{\#Images} &
      \textbf{\#VQA Pairs} &
      \textbf{Multi-modal} &
      \textbf{Morph. CoT} &
      \textbf{CoT} &
      \textbf{Fairness} \\
      \midrule
      SkinCon \citep{skincon}       & \xmark & \cmark & 2 & 3,886   & \textendash       & \xmark & \xmark & \xmark & \xmark \\
      SkinCap \citep{skincap}       & \xmark & \cmark & 1 & 4,000   & \textendash       & \xmark & \xmark & \xmark & \xmark \\
      SkinCaRe \citep{skincare}     & \xmark & \cmark & 2 & 7,041   & 7,041     & \xmark & \xmark & \cmark & \xmark\\
      DermaSynth \citep{dermasynth} & \xmark & \cmark & 2 & 45,205  & 92,020    & \cmark & \xmark & \cmark & \xmark \\
      MM-Skin \citep{mmskin}        & \xmark & \cmark & 3 & 11,039  & 27,412    & \cmark & \xmark & \xmark & \xmark \\
      DermaVQA \citep{dermavqa}     & \cmark & \cmark & 1 & 3,434   & 1,488     & \cmark & \xmark & \xmark & \xmark \\
      DermBench \citep{dermbench}   & \cmark & \xmark & 1 & 4,000   & 4,500     & \cmark & \xmark & \cmark & \xmark \\
      \midrule
      \rowcolor{gray!15}
      \textbf{DermoInstruct (Ours)} & \xmark & \cmark & \textbf{4} & \textbf{211,243} & \textbf{772,675}   & \cmark & \cmark & \cmark & \xmark \\
      \rowcolor{gray!15}
      \textbf{DermoBench (Ours)}    & \cmark & \xmark & \textbf{11} & \textbf{12,371}  & \textbf{33,999}    & \cmark & \cmark & \cmark & \cmark \\
      \bottomrule
    \end{tabular}%
  }
  \vspace{-2mm}
  \caption{Comparison of instruction datasets and benchmarks for dermatology MLLMs. Our datasets significantly expand task diversity and introduce morphology-grounded chain-of-thought reasoning (\textbf{Morph. CoT}) and fairness evaluation, addressing key gaps in existing resources.}
  \label{tab:dermo_benchmark_comparison}
  \vspace{-5mm}
\end{table*}

To address these gaps, we propose a holistic framework centered on morphology-grounded reasoning. We first introduce \textbf{DermoInstruct}, a large-scale morphology-anchored instruction corpus unifying 14 heterogeneous public datasets under a shared diagnostic ontology with 9 superclasses and 325 fine-grained subclasses. The dataset contains 211,243 images and 772,675 instruction trajectories spanning 5 task formats: \textit{free-text morphological description, structured attribute generation, clinically grounded Chain-of-Thought reasoning, flat diagnosis,} and \textit{multi-turn hierarchical diagnosis}. This structured diversity ensures the model learns the complete diagnostic trajectory from lesion observation to morphology extraction to diagnostic reasoning, rather than mere label prediction. We also establish \textbf{DermoBench}, a comprehensive evaluation suite with 11 tasks across 4 clinical axes: \textit{Morphology}, \textit{Diagnosis}, \textit{Reasoning}, and \textit{Fairness} (Figure~\ref{fig:dermobench_fig1} and Table~\ref{tab:dermobench_tasks}). For rigorous evaluation, we constructed 3,600 open-ended instances from a 900-case core image set with line-by-line specialist revision to guarantee morphological fidelity and reasoning validity, providing ``Gold Standard" ground truth. We also benchmarked expert dermatologist performance as a clinical ceiling, enabling precise quantification of the Human-AI gap.


Building on these resources, we develop \textbf{DermoGPT}, a dermatology-specialized MLLM initialized from Qwen3-VL-8B. The training proceeds through two phases. First, Supervised Fine-Tuning (SFT) on DermoInstruct establishes foundational diagnostic capabilities. Second, a novel Morphologically-Anchored Visual-Inference-Consistent (MAVIC) reward aligns the model with clinical reasoning trajectories. MAVIC utilizes Group Relative Policy Optimization (GRPO) \cite{deepseekmath} to penalize logical disconnects between generated visual morphology descriptions and diagnostic conclusions, enforcing the ``morphology-first'' reasoning trajectory. At inference, a Confidence-Consistency Test-time adaptation (CCT) scheme aggregates predictions to improve generalization. DermoGPT significantly outperforms 16 baselines across all 11 tasks, particularly in morphology understanding and reasoning consistency, narrowing the Human-AI gap.

Our contributions are three-fold: (1) \textbf{DermoBench Benchmark}: The first unified suite evaluating the full clinical pipeline beyond MCQAs for dermatology. Validated against an expert-verified core set and human baselines, it exposes systemic reliability gaps in current MLLMs. (2) \textbf{DermoInstruct Dataset}: The largest ontology-aware corpus unifying 14 sources into structured multi-task trajectories, providing the essential supervision for versatile, clinically-aligned reasoning. (3) \textbf{DermoGPT}: The first clinical-aligned reasoning MLLM in dermatology utilizing the MAVIC and CCT. This approach yields substantial improvements, significantly narrowing the human-AI gap in both diagnostic accuracy and reasoning.

\begin{figure*}[t]
  \centering
  \includegraphics[width=0.95\linewidth]{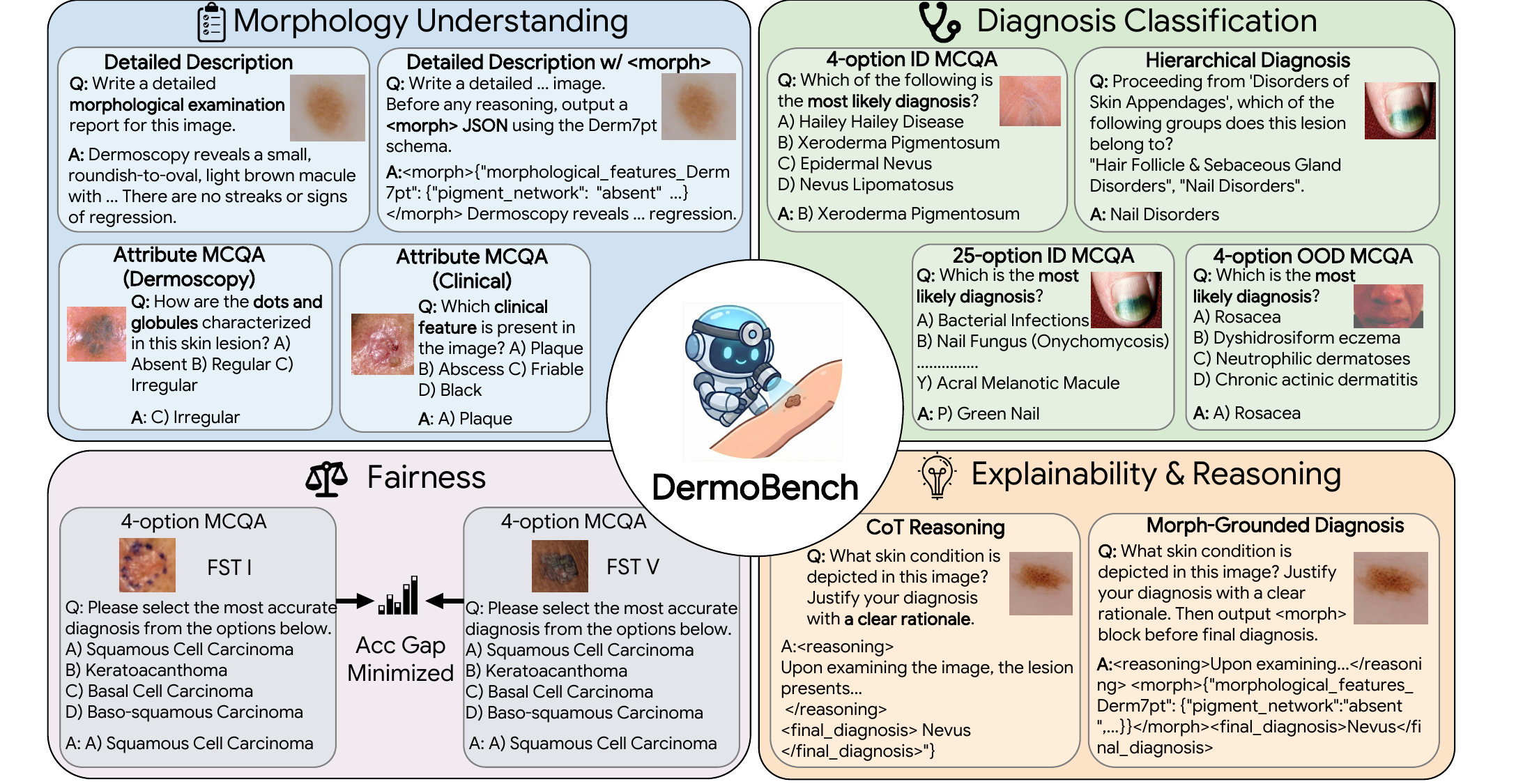}
  \vspace{-0.2cm}
  \caption{Overall architecture of DermoBench. DermoBench contains 11 subtasks spanning four axes: \textbf{Morphology} (Task~1.1 Detailed Description; Task~1.2 Morph-grounded Description; Task~1.3 Dermoscopic Attribute MCQA; Task~1.4 Clinical Attribute MCQA), \textbf{Diagnosis} (Task~2.1 4-option ID MCQA; Task~2.2 25-option ID MCQA; Task~2.3 hierarchical diagnosis; Task~2.4 4-option OOD MCQA), \textbf{Reasoning} (Task~3.1 CoT reasoning; Task~3.2 Morph-grounded Reasoning), and \textbf{Fairness} (Task~4). Note that the same set of images is used across all open-ended tasks (Tasks 1.1, 1.2, 3.1, and 3.2).}
  \label{fig:dermobench_fig1}
  \vspace{-0.3cm}
\end{figure*}
\section{Related Work}
\label{sec:related_work}

\noindent \textbf{Dermatology MLLMs and Reasoning.}
The landscape of dermatology AI has evolved from closed-set classification \citep{alsuwaidan2023deep, panderm} to open-ended multimodal reasoning. Early works relied on discrimitive model with limited label spaces \citep{derm1m, derm7pt}. Recently, specialized MLLMs such as SkinGPT-4 \citep{skingpt}, SkinGPT-R1 \citep{skingptr1}, and Skin-R1 \citep{skinr1} have adapted general foundation models to dermatology via instruction tuning. While these models demonstrate improved dialogue capabilities, they typically treat diagnostic reasoning as a latent, black-box process. Unlike our \textbf{DermoGPT}, which enforces an explicit \textit{Morphology $\rightarrow$ Reasoning $\rightarrow$ Diagnosis} workflow via concept bottleneck, existing approaches lack fine-grained grounding, often leading to hallucinations where visual evidence contradicts diagnostic conclusions.

\vspace{2pt}
\noindent\textbf{Dermatology Training Data and Benchmarks.}
The paradigm of dermatology AI has shifted from standard classification to large-scale vision-language alignment, exemplified by Derm1M~\citep{derm1m} and subsequent instruction-tuned MLLMs~\citep{skingpt, skinr1}. However, current approaches rely on small-scale instruction data with limited task diversity. Furthermore, evaluation remains underdeveloped—while DermBench~\citep{dermbench} assesses diagnostic narratives, it lacks rigorous workflow verification. To address these gaps, we introduce DermoInstruct, an expert-curated dataset with 772K morphology-grounded instruction pairs, and DermoBench, a multi-axis testbed that evaluates the full clinical workflow from morphology and diagnosis to OOD robustness and fairness.
\begin{figure*}[t]
  \centering
  \resizebox{0.95\linewidth}{!}{
    \subcaptionbox{}{\includegraphics[height=4.7cm]{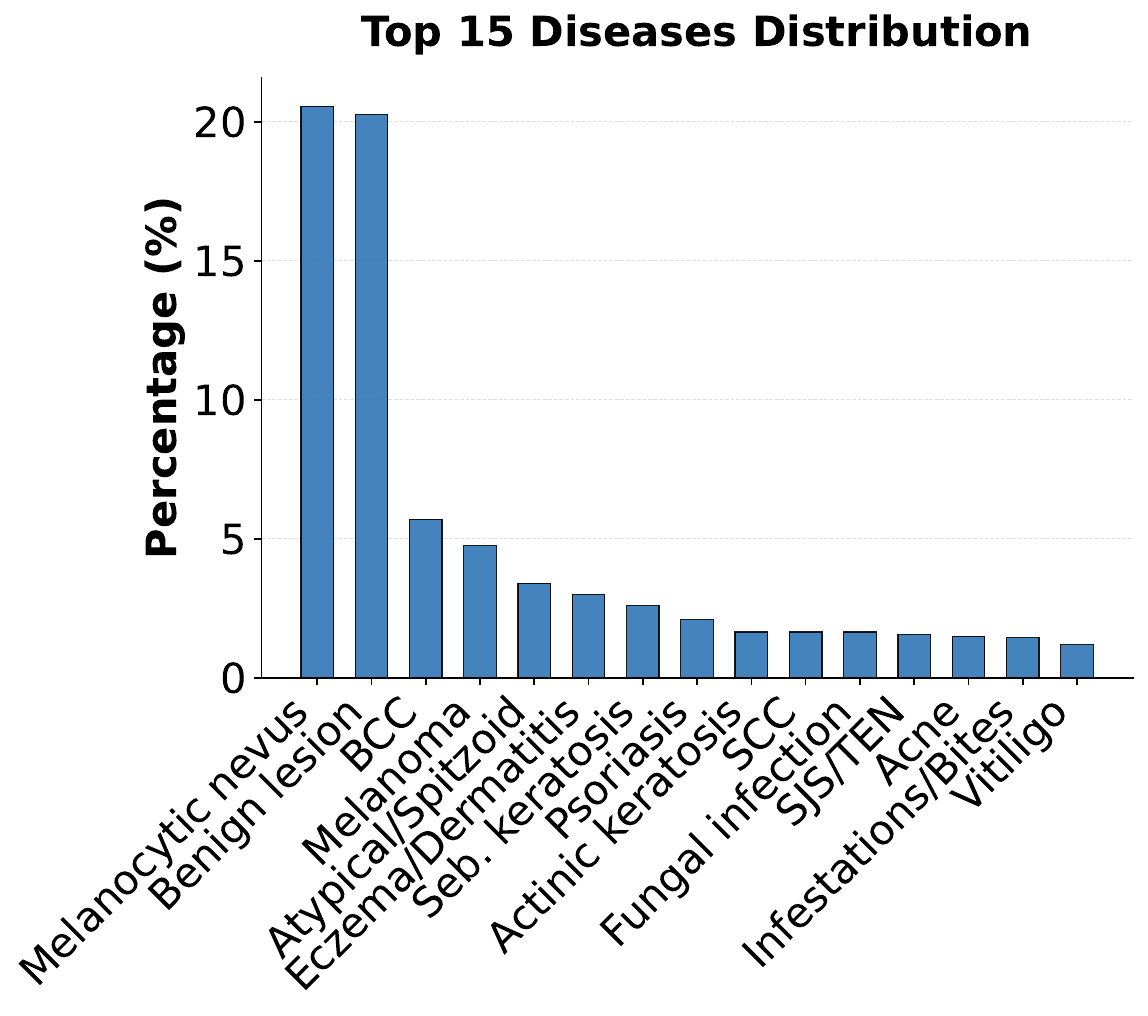}}
    \hfill
    \subcaptionbox{}{\includegraphics[height=4.7cm]{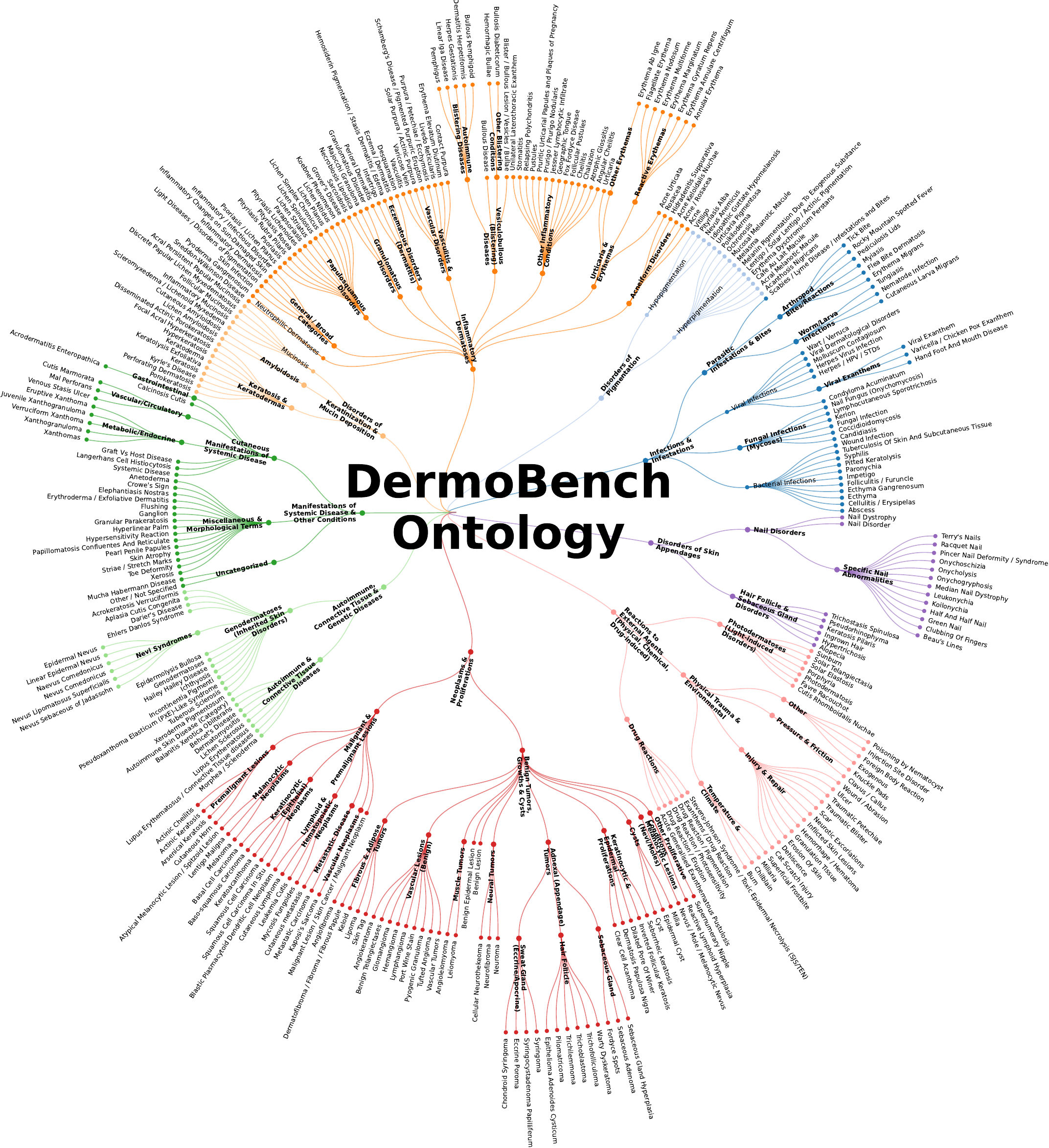}}
    \hfill
    \subcaptionbox{}{\includegraphics[height=4.7cm]{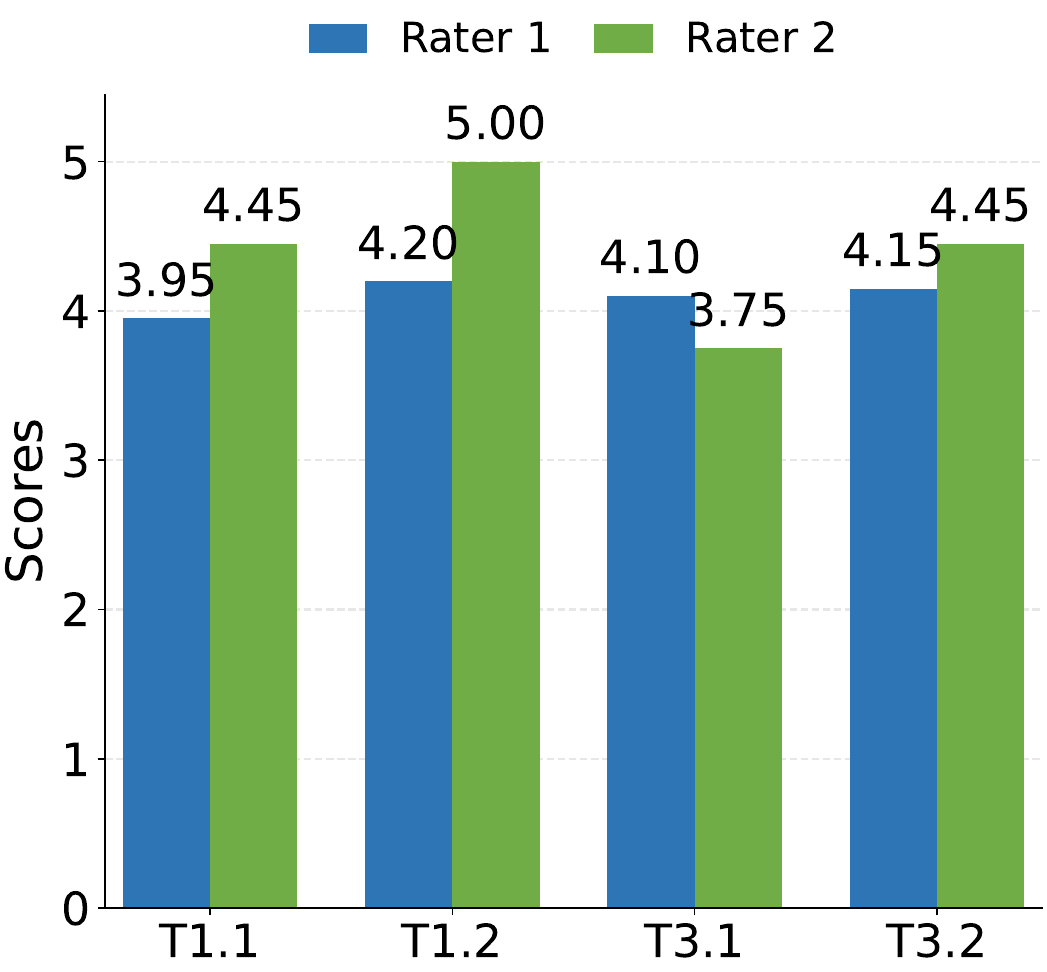}}
  }
  \vspace{-0.2cm}
  \caption{Overview of DermoBench. (a) Distribution of the top 15 diseases. (b) A unified ontology organizes 325 fine-grained diagnoses in DermoBench and DermoInstruct into 9 top-level super-classes. Zoom in for details. (c) Human ratings of LLM-as-a-Judge quality. 0 stands for ``strongly disagree'', and 5 represents ``strongly agree''}
  \label{fig:2}
  \vspace{-0.4cm}
\end{figure*}
\section{DermoInstruct}
To address the scarcity of clinically grounded training resources, we introduce DermoInstruct. Unlike prior works, this corpus is constructed to operationalize the \textbf{``morphology-first''} diagnostic workflow, providing high-quality supervision aligned with a unified ontology.
\subsection{DermoInstruct Curation}
\label{sec:curation}
The construction pipeline employs a four-step strategy to ensure both data scale and clinical rigor.

\noindent\textbf{(1) Aggregation \& Rigorous Cleaning:} We aggregated 14 public datasets spanning clinical and dermoscopic modalities. To strictly prevent data leakage, we implemented a \textit{patient-level split}. We further applied perceptual hashing (pHash, Hamming distance $\le 2$) to remove near-duplicate images, resulting in 211,243 distinct, high-quality images (see Appendix~\ref{sec:appendix-related} for source details).

\noindent\textbf{(2) Ontology Induction:} Addressing the label fragmentation issue across heterogeneous sources, we employed GPT-5 to normalize 903 raw diagnostic strings into canonical clusters. These clusters were rigorously reviewed by two dermatologists to merge synonyms and resolve ambiguities, yielding a unified ontology of 9 superclasses and 325 fine-grained subclasses (Figure~\ref{fig:2}b; zoom in).

\noindent\textbf{(3) Morphology-grounded Reasoning Synthesis:} To transcend the limitations of naive CoT, we implemented a Clinically-Aligned Reasoning Synthesis pipeline that mirrors the expert diagnostic workflow: \textit{Observation $\to$ Abstraction $\to$ Deduction}. We prompted Gemini-2.5-Flash~\citep{gemini} via a strict dependency-aware protocol (detailed prompts could be found in Appendix~\ref{app:prompts}):
(i) \textbf{Morphological Inspection:} First, generate detailed descriptions of salient lesion structures (e.g., borders, symmetry) to simulate visual examination.
(ii) \textbf{Schema-Based Anchoring:} Explicitly map these visual findings to standardized medical terminologies (seven-point checklist \citep{derm7pt} for dermoscopy, general dermatology guidlines \citep{skincon} for clinical images). This acts as a ``concept bottleneck,'' \cite{cbm}, anchoring pixel data to verifiable medical facts.
(iii) \textbf{Evidence-Informed Diagnosis:} Finally, synthesize a reasoning chain that is rigorously conditioned on these extracted attributes. This enforces a reasoning trajectory where the model must justify the diagnosis via morphological evidences (e.g., ``presence of atypical network implies higher risk of melanoma''), ensuring the reasoning is transparent, interpretable, and clinically coherent.

\noindent\textbf{(4) Diagnosis VQA Construction:} Complementing the open-ended reasoning, we leveraged the unified ontology to synthesize structured decision-making tasks that test the model's diagnostic precision. For \textit{Flat MCQAs}, we enforced clinical hardness by sampling distractors exclusively from sibling nodes or nearest neighbors (i.e., clinical mimics), demanding fine-grained discrimination beyond random guessing. For \textit{Hierarchical Instructions}, we modeled diagnosis as a sequential root-to-leaf traversal with an adaptive correction mechanism: if the reasoning trajectory deviates, corrective prompts inject expert guidance to realign the diagnostic path, simulating the interactive pedagogy of medical training.

\subsection{DermoInstruct Data Analysis}
\label{sec:data_analysis}
The final corpus comprises 211,243 multimodal images and 772,675 instructions (646k used for training after holding out DermoBench evaluation splits; see Appendix~\ref{app:646k}). As illustrated in Figure~\ref{fig:2}, the dataset features a realistic long-tail disease distribution (Fig.~\ref{fig:2}a) organized under our unified ontology of 9 superclasses and 325 subclasses (Fig.~\ref{fig:2}b). The instruction data across 4 major task dimensions spans 5 formats forming a complete diagnostic loop: (1) \textit{Free-text morphological description}; (2) \textit{Structured attribute generation} (for concept bottleneck training); (3) \textit{Clinically grounded CoT reasoning}; (4) \textit{Flat diagnosis}; and (5) \textit{Multi-turn hierarchical diagnosis}. This structured diversity ensures the model learns to look, reason, and deduce, rather than just memorize labels.

\begin{figure*}[ht]
  \centering
  \includegraphics[width=0.9\linewidth]{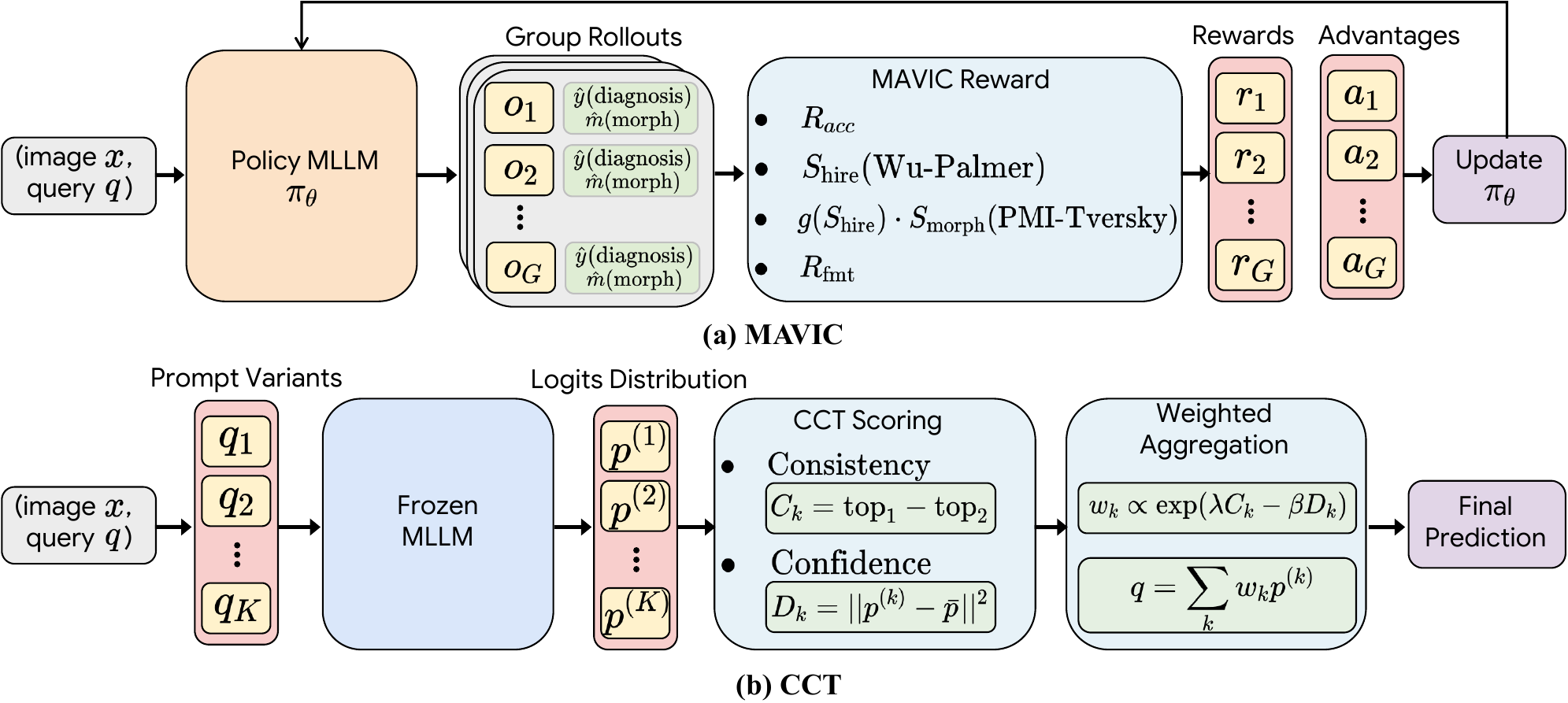}
  \caption{Method overview of MAVIC and CCT. (a) MAVIC integrates diagnosis accuracy, taxonomy-level similarity, gated morphology agreement, and format validity into a GRPO-style group reward to enforce morphology-first alignment. (b) CCT is a decoding-only test-time aggregation that reweights prompt-variant distributions by confidence and cross-variant consistency, requiring no parameter updates.}
  \label{fig:vis}
  \vspace{-3mm}
\end{figure*}
\section{DermoBench}

\subsection{Benchmark Construction} 
We construct DermoBench, a comprehensive evaluation suite comprising 33,999 VQA pairs spanning 11 subtasks across 4 dimensions: \textit{Morphology}, \textit{Diagnosis}, \textit{Reasoning}, and \textit{Fairness} (Table~\ref{tab:dermobench_tasks} and Appendix Figure~\ref{fig:3-move-supp}). The benchmark consists of 3,600 open-ended instances from a 900-case core image set (enabling cross-task consistency evaluation across T1.1, T1.2, T3.1, T3.2) and 30,399 closed-ended MCQAs. Each open-ended sample underwent strict line-by-line dermatologist revision to serve as gold-standard references. Two reasoning tasks (T1.2, T3.2) require structured morphological evidence before diagnosis to prevent ungrounded predictions. Independent sanity checks by two dermatologists confirmed high annotation quality with mean scores of 3.88–4.60 in a 5-scale score across tasks (Appendix Figure~\ref{fig:3-move-supp}b). The closed-ended component comprises 12,533 diagnoses, 654 fairness and 17,212 attribute-related MCQAs across 7 subtasks, including Out-of-Distribution tasks (T2.4). All images are isolated from training data to prevent leakage. Please refer to Appendix~\ref{app:dermobench_tasks} for details about task definitions and data sources.

\subsection{Evaluation Metrics} We adopt distinct metrics \citep{dentalgpt,oralgptomni} tailored to the nature of each subtask. For closed-ended questions, we use standard accuracy. For open-ended tasks, we employ an LLM-as-a-Judge protocol (using Gemini-2.5-Pro), which compares model outputs against human-curated references to generate mean fidelity scores. Judge consistency was validated through model substitution experiments and human sanity checks. Crucially, to quantify the real-world utility gap, we invited board-certified dermatologists to complete all tasks. Their performance serves as the clinical ceiling, allowing us to precisely measure where MLLMs fall short compared to human experts. Complete LLM-as-a-judge protocol are in Appendix \ref{app:judge_prompts}.

\section{DermoGPT}
\label{sec:method}

We aim to develop models that follow the dermatological reasoning chain \emph{morphology $\rightarrow$ reasoning $\rightarrow$ diagnosis} with explicitly verifiable intermediate steps. We propose MAVIC (Morphologically-Anchored Visual-Inference-Consistent) reward, an end-to-end computable reward function requiring no external judge, achieving morphology-first alignment during RL training. We further introduce CCT (Confidence--Consistency Test-time adaptation), a plug-and-play decoding strategy that enhances OOD generalization without fine-tuning. Hyperparameters and implementation details are documented in Appendix~\ref{app:train_details}.

\subsection{MAVIC: Morphologically-Anchored Visual-Inference-Consistent Reward}
\label{sec:mavic}

We begin with multi-task supervised fine-tuning (SFT) on DermoInstruct using Qwen3-VL-8B-Instruct~\citep{qwen3vl}. We optimize cross-entropy loss for 1 epoch with LoRA (rank 64, $\alpha=64$, dropout 0.05) while freezing the LLM and training the vision tower and projector, obtaining DermoGPT-SFT. To enable automatic verification of morphological evidence, we adopt a concept bottleneck framework~\citep{cbm} that compels the model to output structured morphological features following the ``seven-point checklist''~\cite{derm7pt} and ``general dermatology guideline''~\cite{skincon} schema. This structured output enables direct computation of morphology-level rewards without external judges—a key departure from prior approaches that rely on costly LLM-as-a-judge pipelines.

However, RL training for open-ended morphology descriptions faces a critical challenge: lack of directly verifiable reward signals. Diagnosis-only rewards are sparse and encourage shortcut learning that bypasses morphological evidence. To address this, we design MAVIC reward with the following components. Given an image and instruction, we sample $G$ completions from the current policy following GRPO~\citep{deepseekmath} and compute the following reward components for each rollout:

\begin{table}[t]
\centering
\resizebox{\linewidth}{!}{
\begin{tabular}{l l l r}
\toprule
\textbf{Axis} & \textbf{Task} & \textbf{Type} & \textbf{\#Pairs} \\
\midrule
Morphology &
T1.1 Detailed Description & Open-ended & 900 \\
& T1.2 Morph-grounded Description & Open-ended & 900 \\
& T1.3 Dermoscopic attribute MCQA & MCQA & 5,530 \\
& T1.4 Clinical attribute MCQA & MCQA & 11,682 \\
\midrule
Diagnosis &
T2.1 ID 4-way MCQA & MCQA & 2,000 \\
& T2.2 ID 25-way MCQA & MCQA & 2,000 \\
& T2.3 Hierarchical diagnosis & MCQA (multi-step) & 2,000 \\
& T2.4 OOD 4-way MCQA & MCQA & 6,533 \\
\midrule
Reasoning &
T3.1 CoT reasoning & Open-ended & 900 \\
& T3.2 Morph-grounded reasoning & Open-ended & 900 \\
\midrule
Fairness &
T4 Skin-type fairness MCQA & MCQA & 654 \\
\bottomrule
\end{tabular}}
\vspace{-2mm}
\caption{DermoBench tasks, sizes, and data sources.}
\vspace{-3mm}
\label{tab:dermobench_tasks}
\end{table}

\noindent\textbf{(1) $R_{\text{acc}}$}: Standard 0-1 reward for tasks with unique ground-truth (e.g., MCQAs).

\noindent\textbf{(2) $S_{\text{hier}} \in [0,1]$}: Hierarchical similarity over the diagnostic ontology using Wu-Palmer function~\citep{wupalmer}. This differentiates completely incorrect diagnoses from predictions correct at superclass level, mitigating sparse rewards while encouraging coarse-to-fine diagnostic alignment.

\noindent\textbf{(3) $S_{\text{morph}} \in [0,1]$}: Morphology similarity computed via PMI-weighted Tversky matching on structured outputs (Derm7pt/SkinCon attributes).

\noindent\textbf{(4) Gating $g(\cdot)$ and $R_{\text{fmt}}$}: To prevent models from exploiting template-style morphology outputs when diagnoses diverge from ground truth, we progressively unlock morphology rewards only when diagnostic alignment is reasonable:
\begin{equation}
 \vspace{2mm}
 \footnotesize
g(S_{\text{hier}})=\sigma\!\big(k\cdot(S_{\text{hier}}-\mu)\big),
\label{eq:gating}
\vspace{-3mm}
\end{equation}
where $\mu$ is the median $S_{\text{hier}}$ within each batch (adaptive difficulty threshold). $R_{\text{fmt}}$ verifies JSON schema validity and critical tags to ensure auditable outputs. The total MAVIC reward is:
\begin{equation}
 \vspace{2mm}
 \footnotesize
R = R_{\text{acc}} + \lambda_{\text{hier}} S_{\text{hier}} + \lambda_{\text{morph}}\, g(S_{\text{hier}})\, S_{\text{morph}} + R_{\text{fmt}},
\label{eq:mavic}
\end{equation}
with $\lambda_{\text{hier}}=\lambda_{\text{morph}}=1$ by default. We optimize the standard GRPO objective using MAVIC rewards to obtain DermoGPT-RL. Complete implementation detail is in Appendix~\ref{app:mavic_details}.

\subsection{Confidence--Consistency Test-time Adaptation}
\label{subsec:tta}
To further improve generalization under distribution shifts, we note that trivial deterministic decoding often yields unstable predictions on out-of-distribution (OOD) samples, yet full test-time fine-tuning is infeasible in clinical workflows. Thus, we propose CCT, a purely decoding-level strategy that enhances OOD robustness through weighted aggregation of multiple stochastic rollouts, without updating model parameters. The key insight is that reliable predictions should be both confident and consistent across sampling variations, aligning with dermatological practice where diagnostic certainty requires stable evidence.

\begin{table*}[ht]
\centering
\resizebox{\textwidth}{!}{%
\begin{tabular}{l|c|>{\columncolor{openended}}c>{\columncolor{openended}}c>{\columncolor{closeended}}c>{\columncolor{closeended}}c|c|>{\columncolor{closeended}}c>{\columncolor{closeended}}c>{\columncolor{closeended}}c|c|>{\columncolor{closeended}}c>{\columncolor{closeended}}c>{\columncolor{closeended}}c>{\columncolor{closeended}}c|c|>{\columncolor{openended}}c>{\columncolor{openended}}c|c|>{\columncolor{closeended}}c}
\toprule
\multirow{3}{*}{\textbf{Model}} & \multirow{3}{*}{\textbf{Params}} & \multicolumn{5}{c|}{\textbf{Task 1: Morphology}} & \multicolumn{9}{c|}{\textbf{Task 2: Diagnosis}} & \multicolumn{3}{c|}{\textbf{Task 3: Reasoning}} & \textbf{Task 4} \\
\cmidrule(lr){3-7} \cmidrule(lr){8-16} \cmidrule(lr){17-19} \cmidrule(lr){20-20}
& & \textbf{T1.1} & \textbf{T1.2} & \textbf{T1.3} & \textbf{T1.4} & \textbf{Avg.} & \multicolumn{4}{c|}{\textbf{In-Distribution (ID)}} & \multicolumn{5}{c|}{\textbf{Out-of-Distribution (OOD)}} & \textbf{T3.1} & \textbf{T3.2} & \textbf{Avg.} & \textbf{Fair.} \\
& & \scriptsize{(Desc)} & \scriptsize{(Struct)} & \scriptsize{(D7pt)} & \scriptsize{(SkinCon)} & \scriptsize{(T1)} & \textbf{4-cls} & \textbf{25-cls} & \textbf{Hier.} & \textbf{Avg.} & \textbf{Derm1M} & \textbf{DDI} & \textbf{D7pt} & \textbf{SNU} & \textbf{Avg.} & \scriptsize{(CoT)} & \scriptsize{(M-CoT)} & \scriptsize{(T3)} & \scriptsize{(Score)} \\
\midrule
\multicolumn{20}{l}{\textit{General Purpose MLLMs}} \\
GPT-4o-mini & / & 34.55 & 51.80 & 41.19 & 61.09 & 47.16 & 59.50 & 34.75 & 65.90 & 53.38 & 52.12 & 58.54 & 56.48 & 59.17 & 56.57 & 42.83 & 51.65 & 47.24 & 94.06 \\
Claude-Sonnet-4.5-Thinking & / & 36.75 & 55.90 & 29.73 & 59.20 & 45.40 & 55.35 & 34.15 & 63.40 & 50.97 & 53.64 & 52.90 & 50.40 & 68.75 & 56.42 & 43.54 & 54.37 & 48.95 & 91.40 \\
Gemini-2.5-Flash & / & 40.08 & 53.48 & 39.28 & 66.59 & 49.86 & 72.60 & 47.20 & 70.31 & 63.37 & 66.33 & 59.15 & 53.96 & 65.42 & 61.21 & 48.92 & 58.49 & 53.70 & 79.89 \\
GLM-4.5V & 106B & 36.85 & 42.75 & 45.50 & 52.03 & 44.28 & 63.65 & 28.85 & 52.39 & 48.30 & 45.51 & 48.17 & 43.08 & 57.08 & 48.46 & 44.19 & 53.28 & 48.73 & 93.59 \\
Qwen2.5-VL-72B & 72B & 27.97 & 49.35 & 52.91 & 60.51 & 47.69 & 61.50 & 35.95 & 53.93 & 50.46 & 54.63 & 54.88 & 58.36 & 66.67 & 58.63 & 40.39 & 49.71 & 45.05 & \textbf{97.32} \\
QVQ-72B-Preview & 72B & 22.38 & 41.02 & 49.77 & 59.20 & 43.09 & 64.65 & 47.30 & 57.25 & 56.40 & 60.53 & 53.66 & 56.92 & 62.92 & 58.51 & 51.56 & 54.14 & 52.85 & 86.26 \\
Llama-3.2-90B & 90B & 28.20 & 44.43 & 35.84 & 49.19 & 39.41 & 47.85 & 51.65 & 51.20 & 50.23 & 44.76 & 49.09 & 37.14 & 49.58 & 45.14 & 44.61 & 56.14 & 50.38 & 91.31 \\
Llama-3.2-11B & 11B & 12.33 & 38.48 & 39.13 & 29.93 & 29.97 & 29.25 & 16.50 & 35.98 & 27.58 & 25.50 & 21.80 & 26.90 & 42.92 & 29.28 & 36.16 & 38.29 & 37.22 & 53.85 \\
Nemotron-Nano & 12B & 18.93 & 29.09 & 38.72 & 59.20 & 36.49 & 47.25 & 25.60 & 40.17 & 37.67 & 44.12 & 39.48 & 36.84 & 52.08 & 43.14 & 31.90 & 37.40 & 34.65 & 92.40 \\
Qwen3-VL-32B & 32B & 50.30 & 57.43 & 46.15 & 60.67 & 53.64 & 64.25 & 38.05 & 64.08 & 55.46 & 48.13 & 57.93 & 63.11 & \textbf{69.58} & 59.69 & 55.04 & 53.85 & 54.45 & 81.78 \\
Qwen3-VL-8B (Base) & 8B & 33.18 & 46.05 & 40.43 & 62.06 & 45.43 & 67.20 & 45.35 & 44.77 & 52.44 & 52.67 & 51.07 & 59.10 & 55.42 & 54.31 & 47.53 & 53.43 & 50.48 & 89.37 \\
\midrule
\multicolumn{20}{l}{\textit{Medical/Dermatology Specialized}} \\
HuatuoGPT-Vis-7B & 7B & 18.15 & 34.50 & 33.82 & 38.15 & 31.15 & 51.60 & 26.05 & 46.10 & 41.25 & 31.40 & 36.13 & 41.64 & 47.92 & 39.27 & 39.41 & 43.98 & 41.69 & 76.80 \\
LLaVA-Med-v1.5 & 7B & 23.07 & 29.73 & 40.15 & 56.42 & 37.34 & 49.65 & 32.40 & 43.29 & 41.78 & 41.38 & 36.74 & 33.63 & 37.08 & 37.21 & 38.33 & 46.19 & 42.26 & 60.48 \\
SkinVL-PubMM & 7B & 27.82 & 42.63 & 43.62 & 61.31 & 43.84 & 57.15 & 38.75 & 52.19 & 49.36 & 51.12 & 48.93 & 58.95 & 54.58 & 53.40 & 42.92 & 54.62 & 48.77 & 83.04 \\
Lingshu-32B & 32B & 14.94 & 44.85 & 43.47 & 52.39 & 38.91 & 53.45 & 38.40 & 49.11 & 46.99 & 30.29 & 34.91 & 32.24 & 45.83 & 35.82 & 44.41 & 49.55 & 46.98 & 75.44 \\
Lingshu-7B & 7B & 16.44 & 40.74 & 43.92 & 46.08 & 36.80 & 49.55 & 31.90 & 43.43 & 41.64 & 25.95 & 32.16 & 33.88 & 40.00 & 33.00 & 47.16 & 49.30 & 48.23 & 61.58 \\
\midrule
\arrayrulecolor{black!60}\hline
\textbf{DermoGPT-SFT} & 8B & 41.74 & 49.11 & 53.69 & 75.56 & 55.02 & 89.55 & 64.30 & 77.91 & 77.25 & 68.91 & 62.80 & 65.88 & 59.17 & 64.19 & 62.57 & 63.34 & 62.95 & 91.12 \\
\textbf{DermoGPT-SFT + CCT} & 8B & 43.49 & 50.96 & 54.10 & 75.92 & 56.12 & 89.75 & 64.45 & 78.06 & 77.42 & 70.65 & \textbf{64.33} & 65.58 & 61.25 & 65.45 & 63.73 & 65.31 & 64.52 & 92.41 \\
\textbf{DermoGPT-RL} & 8B & 43.93 & 59.29 & 56.53 & 76.67 & 59.10 & \textbf{90.30} & 64.60 & 79.12 & 78.01 & 69.68 & 62.80 & 68.59 & 60.00 & 65.27 & 66.04 & 65.48 & 65.76 & 93.49 \\
\textbf{DermoGPT-RL + CCT} & 8B & \textbf{44.76} & \textbf{60.33} & \textbf{56.94} & \textbf{77.22} & \textbf{59.81} & 89.60 & \textbf{65.40} & \textbf{79.12} & \textbf{78.04} & \textbf{71.56} & 62.96 & \textbf{70.13} & 61.25 & \textbf{66.48} & \textbf{67.74} & \textbf{66.64} & \textbf{67.19} & 93.88 \\
\arrayrulecolor{black!60}\hline
\arrayrulecolor{black}
\midrule
Human Performance & - & 73.36 & 79.27 & 83.00 & 92.00 & 81.90 & 85.00 & 77.00 & 87.54 & 83.18 & 94.00 & 86.00 & 89.00 & 93.00 & 90.50 & 82.15 & 78.41 & 80.28 & 94.00 \\
\bottomrule
\end{tabular}%
}
\caption{\textbf{Main Results on DermoBench.} We evaluate models across four dimensions, and report each model's parameter count when publicly available (\textbf{Params}; ``/'' denotes unknown). \colorbox{openended}{\small Blue} columns indicate open-ended generation tasks (description and structured output), while \colorbox{closeended}{\small orange} columns indicate close-ended classification/scoring tasks. White columns represent aggregate metrics. \textbf{CCT} denotes our confidence--consistency test-time adaptation module. \textbf{Bold} indicates the best result in each column.}
\label{tab:main_results}
\vspace{-0.3cm}
\end{table*}

\subsubsection{Confidence--Consistency Ensemble}
\label{subsubsec:tta-ensemble}

At each decoding step $t$ for input $(x,\mathrm{query})$, we sample $K$ rollouts yielding token distributions $p^{(1)}_t,\dots,p^{(K)}_t \in \Delta^{V-1}$. For each rollout $r$, we compute:

\noindent\textbf{Confidence $C_r$ (margin-based):} Let $p^{(r)}_{t,(1)}$ and $p^{(r)}_{t,(2)}$ denote the highest and second-highest probabilities in $p^{(r)}_t$. We define $C_r = p^{(r)}_{t,(1)} - p^{(r)}_{t,(2)} \in [0,1]$. A larger margin indicates a more confident prediction. For discrete answer tasks, we compute this over option tokens; for free-form generation, over the full vocabulary.

\noindent\textbf{Consistency $D_r$ (deviation from barycenter):} We compute the empirical barycenter $\bar{p}_t = \frac{1}{K}\sum_{j=1}^K p^{(j)}_t$ and set $D_r = \frac{1}{2}\,\|p^{(r)}_t - \bar{p}_t\|_2^2$. Rollouts that deviate significantly from $\bar{p}_t$ (large $D_r$) are downweighted exponentially.

We construct the aggregated distribution via weighted combination:
\begin{equation}
 \vspace{2mm}
 \footnotesize
q_t = \sum_{r=1}^K w_r p^{(r)}_t, \quad
w_r = \frac{\exp(\lambda C_r - \beta D_r)}{\sum_{j=1}^K \exp(\lambda C_j - \beta D_j)},
\label{eq:cct}
\vspace{-2mm}
\end{equation}
where $\lambda$ and $\beta$ control the relative importance of confidence and consistency. The weighting exponentially suppresses outlier rollouts (high $D_r$) while favoring confident predictions (high $C_r$), ensuring predictions are both stable and confident, critical for clinical reliability. The next token is sampled from $q_t$, and this process repeats for each step. In practice, we set $K=8$ and $\lambda=\beta=1.0$.

\subsubsection{Theoretical Guarantee}
\label{subsubsec:tta-theory-main}

To formalize the robustness of this weighting scheme, we establish the following guarantee under distribution contamination; full proofs appear in Appendix~\ref{app:theory}.

\begin{theorem}[Robustness of CCT, informal]
\label{thm:tta-robust}
Let $\{p^{(r)}_t\}_{r=1}^K$ be sampled from a mixture where fraction $(1-\varepsilon)$ comes from a ``good'' component concentrated near $p^\star_t$, and fraction $\varepsilon$ comes from an arbitrary ``bad'' component ($\varepsilon<\tfrac{1}{2}$). Under bounded variance assumptions, there exist constants $\varepsilon_{\mathrm{eff}}, C_U, \gamma_{\mathrm{eff}} > 0$ such that:
\begin{equation}
 \vspace{2mm}
 \footnotesize
\bigl\|q_t - p^\star_t\bigr\|_2 \le \varepsilon_{\mathrm{eff}} + C_U + \mathrm{const} \cdot \exp\!\bigl(-\beta \gamma_{\mathrm{eff}} + \lambda\bigr).
\label{eq:theorem}
\vspace{-2mm}
\end{equation}
\end{theorem}

The bound shows that corrupted rollouts' influence decays exponentially with $\beta$, keeping $q_t$ near $p^\star_t$ when $\beta$ is sufficiently large relative to $\lambda$. This theoretical guarantee explains why CCT remains robust even when a substantial fraction (up to $\varepsilon < 50\%$) of rollouts are corrupted by distribution shifts—the aggregation automatically suppresses outliers without requiring knowledge of the corruption distribution.
\section{Experiments}
\label{sec:experiments}

\noindent \textbf{Performance on Closed-Ended Tasks.}
We evaluate model accuracy across Dermoscopic/Clinical Attribute Recognition (T1.3--1.4), Diagnosis (including In-Distribution 4-cls/25-cls/Hierarchical MCQA and OOD MCQA; T2), and Fairness (T4). Results demonstrate that our DermoGPT-SFT baseline alone establishes a new state-of-the-art, validating the high quality of our instruction data. On In-Distribution (ID) diagnosis (T2 Avg), SFT achieves 77.25\%, surpassing its base model (Qwen3-VL-8B; 52.44\%) and the strongest commercial baseline Gemini-2.5-Flash (63.37\%) by substantial margins; notably, it excels in Hierarchical Diagnosis (77.91\% vs. Gemini 70.31\%) and Clinical Attribute Recognition (T1.4: 75.56\% vs. Gemini 66.59\%). Building on this foundation, our subsequent modules steadily improve robustness: the RL stage enhances OOD performance from 64.19\% (SFT) to 65.27\%, and the CCT module further elevates it to 66.48\% by mitigating domain shifts. Consequently, our final DermoGPT-RL+CCT establishes a comprehensive new state-of-the-art, significantly outperforming Gemini-2.5-Flash across all axes: it improves ID and OOD diagnostic accuracy by +14.67\% and +5.27\%, respectively; crucially, it simultaneously achieves an exceptional Fairness score of 93.88 (Task 4), surpassing Gemini (79.89) by nearly 14\%, effectively minimizing diagnostic disparities across diverse skin tones. 
\begin{figure*}[ht]
  \centering
  \includegraphics[width=\linewidth]{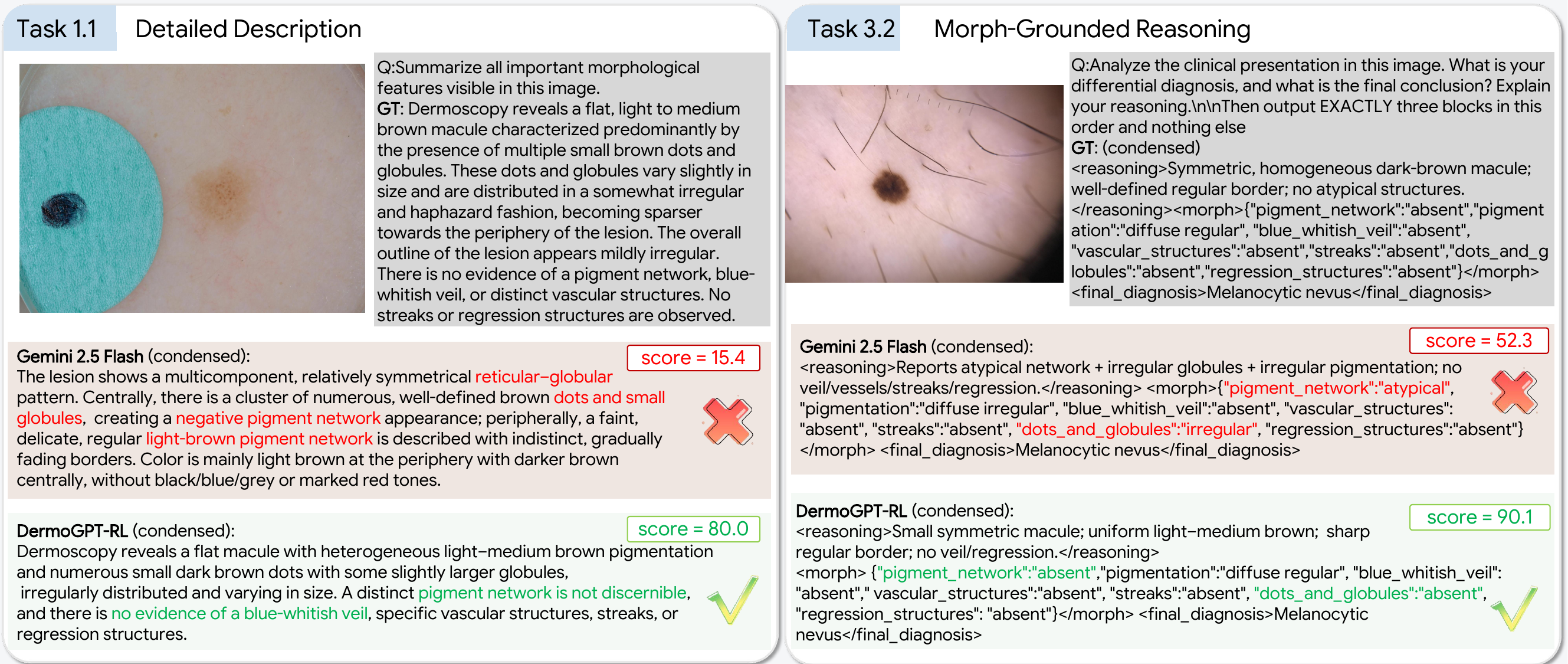}
  \vspace{-0.6cm}
  \caption{\textbf{Qualitative comparison on DermoBench.} Left: Task~1.1 (Detailed Description). Right: Task~3.2 (Morph-Grounded Reasoning with ultra-short structured outputs). Compared to Gemini-2.5 Flash, DermoGPT-RL better matches the reference morphology and achieves higher scores.}
  \vspace{-0.4cm}
    \label{fig:vis}
\end{figure*}

\vspace{2mm}
\noindent \textbf{Open-Ended Morphology \& Reasoning.}
In Open-Ended Morphology and Reasoning tasks (T1.1, T1.2, T3), DermoGPT-RL+CCT demonstrates superior generation quality over both general-purpose MLLMs and existing medical-specialized models. Notably, previous medical specialized models (e.g., HuatuoGPT-Vis-7B) score only 41.69\% on the Reasoning axis, lower than most general MLLMs, suggesting that naive fine-tuning without morphological constraints produces "black-box" classifiers rather than genuine reasoning agents. In contrast, our model scores 67.19\% on the Reasoning axis on average (T3), outperforming Gemini-2.5-Flash (53.70\%) by over 13.49\%; this verifies that our Concept Bottleneck design effectively reduces hallucination by grounding reasoning in explicit morphological evidence. Human sanity checks confirmed high reliability of LLM-Judge scoring ($>4.0/5.0$, Figure~\ref{fig:2}c). Despite these algorithmic advances, a significant Human-AI gap persists, particularly in Detailed Description (T1.1; 73.36 vs 44.76), highlighting that capturing fine-grained visual nuances remains a critical challenge.

\begin{table}[t]
\centering
\small
\setlength{\tabcolsep}{4pt}
\begin{tabular}{lcccc}
\toprule
Setting & T1.1 & T1.2 & T3.1 & T3.2 \\
\midrule
SFT only & 41.74 & 49.11 & 62.57 & 63.34 \\
GRPO (acc+fmt) & 35.13 & 41.20 & 61.34 & 59.88 \\
w/o $S_{\text{morph}}$ & 39.65 & 48.09 & 65.40 & 65.27 \\
w/o $S_{\text{hier}}$ & 42.59 & 50.11 & 63.96 & 65.02 \\
w/o gate ($g{=}1$) & 43.26 & 56.03 & \textbf{66.71} & 63.89 \\
PMI$\to$uniform & 42.56 & 56.98 & 57.32 & 56.64 \\
Full MAVIC & \textbf{43.93} & \textbf{59.29} & 66.04 & \textbf{65.48} \\
\bottomrule
\end{tabular}
\vspace{-2mm}
\caption{MAVIC ablations under GRPO setup ($K{=}8$). Higher is better for all metrics.}
\vspace{-2mm}
\label{tab:mavic_ablation}
\end{table}

\begin{table}[t]
\centering
\small
\setlength{\tabcolsep}{3pt}
\begin{tabular}{lcccc}
\toprule
Method & ID MCQA & OOD MCQA & Hier. & Fair. \\
\midrule
Single ($K{=}1$) & 77.80 & 65.27 & 79.63 & 93.49 \\
Vote ($K{=}4$) & 78.10 & 65.83 & 79.15 & 93.50 \\
MeanProb ($K{=}4$) & 77.95 & 65.69 & 79.51 & 93.32 \\
ConfOnly ($K{=}4$) & 78.40 & 66.47 & 79.49 & 93.09 \\
ConsOnly ($K{=}4$) & 78.35 & \textbf{66.59} & 79.82 & 93.58 \\
CC (Ours, $K{=}4$) & \textbf{78.80} & 66.27 & \textbf{80.31} & \textbf{93.76} \\
\bottomrule
\end{tabular}
\vspace{-2mm}
\caption{Ablation of confidence--consistency components on 900-case core set. Higher is better for all metrics.}
\vspace{-5mm}
\label{tab:tta_component_ablation}
\end{table}
\vspace{2mm}
\noindent\textbf{Ablation Study.} We further dissect component contributions on the core set and OOD benchmarks. Please refer to Appendix~\ref{app:tta-ablation} for more results.

\noindent\textbf{(1) MAVIC Reward Analysis.} 
We first investigate the necessity of morphology-guided rewards (Table~\ref{tab:mavic_ablation}). Naively applying standard RL with only accuracy and format rewards (GRPO(acc+fmt)) proves detrimental, degrading performance below SFT baseline across all reasoning tasks. This indicates that unconstrained RL encourages metric gaming rather than genuine clinical reasoning. Incorporating morphological similarity ($S_{\text{morph}}$) and hierarchical diagnosis rewards ($S_{\text{hier}}$) steadily improves performance. Crucially, the full MAVIC setup with gating mechanism ($g$) achieves peak performance (65.48 on T3.2). Ablating the gate drops performance to 63.89, confirming that difficulty-aware gating prevents the model from bypassing morphological evidence to make uninformed diagnostic guesses.

\noindent\textbf{(2) CCT Test-Time Adaptation Analysis.} 
We evaluate test-time inference with $K$ prompt variants and find that Confidence–Consistency (CC) aggregation consistently outperforms standard ensemble baselines (Majority Vote, MeanProb). As shown in Table~\ref{tab:tta_component_ablation}, on Task 2.1, neither signal alone is sufficient: ConfOnly (78.40\%) and ConsOnly (78.35\%) both underperform CC (78.80\%), indicating complementary robustness cues. We also observe test-time scaling: as $K$ increases from 2 to 8, OOD performance rises from 65.82\% to 66.48\%, supporting TTS for improved reliability.

\noindent\textbf{Qualitative Analysis.}
Fig.~\ref{fig:vis} validates DermoGPT's reasoning superiority over Gemini-2.5-Flash, which exhibits hallucinated morphology concepts (Task 1.1) and inconsistent reasoning between observations and diagnoses (Task 3.2). MAVIC-guided training enables DermoGPT to maintain strict alignment, achieving significantly higher accuracy in feature description and diagnostic consistency.

\section{Conclusion}
\label{sec:conclusion}
We present a comprehensive framework for dermatology MLLMs grounded in morphology-first clinical reasoning. Our unified data--benchmark--model suite—comprising DermoInstruct, DermoBench, and DermoGPT—enables systematic training and evaluation across diverse dermatological tasks, significantly advancing the state-of-the-art while narrowing the human--AI performance gap. This work establishes a foundation for developing clinically-viable dermatology AI systems that mirror expert diagnostic workflows.

\section*{Limitations}
\label{sec:limitations}
Despite substantial progress, several limitations warrant discussion. First, while DermoGPT significantly narrows the human--AI gap, performance disparities persist across all tasks, highlighting the inherent difficulty of clinical-grade diagnostic reasoning. Second, although our benchmark is comprehensive, it may not fully capture the complexity of real-world clinical scenarios, such as patient-level holistic analysis \cite{panderm} or cases requiring longitudinal patient histories. Third, despite integrating expert knowledge during data curation, the morphology-grounded reasoning chains remain susceptible to noise, particularly in ambiguous cases where visual features alone are insufficient for definitive diagnosis. Finally, computational constraints limited our exploration of larger model architectures and full parameter fine-tuning, both of which may further improve performance.
\bibliography{custom}

\clearpage
\onecolumn
\begin{center}
    \LARGE Appendix
\end{center}

\appendix
\addtocontents{toc}{\protect\setcounter{tocdepth}{3}} 

\begingroup
  \renewcommand{\contentsname}{Contents}
  \tableofcontents
\endgroup
\clearpage

\twocolumn

\appendix

\section{Source Datasets and Extended Related Work}
\label{sec:appendix-related}

\begin{table*}[t]
  \centering
  \small
  \begin{tabular}{p{2.6cm}p{2.0cm}p{3.5cm}p{1.1cm}p{5.0cm}}
    \toprule
    Dataset & Modality & Population / setting & Scale & Notes \\
    \midrule
    Daffodil~\citep{daffodil} & Dermoscopic & Bangladesh hospital & S &
    Biopsy-proven dermoscopy dataset. \\
    \midrule
    DermNet~\citep{dermnet} & Clinical & Global web atlas & L &
    Expert-curated clinical photos. \\
    \midrule
    Fitzpatrick17k~\citep{f17k} & Clinical & US outpatient clinics & M &
    Includes Fitzpatrick skin-type labels. \\
    \midrule
    ISIC Archive~\citep{isic} & Dermoscopic & Multi-center dermoscopy & L &
    Standard benchmark for dermoscopic lesions. \\
    \midrule
    MIDAS~\citep{midas} & Clinical \& dermoscopic & Multi-institution NEJM AI dataset & M &
    Paired clinical/dermoscopy with biopsy labels. \\
    \midrule
    PAD-UFES-20~\citep{pad} & Clinical & Brazilian teledermatology & S-M &
    Smartphone photos with rich metadata. \\
    \midrule
    PASSION~\citep{passion} & Clinical & Sub-Saharan Africa & M &
    Smartphone images emphasizing pigmented skin. \\
    \midrule
    PUMCH~\citep{pumch} & Clinical & Chinese tertiary hospital & M &
    Broad inflammatory and neoplastic diseases. \\
    \midrule
    SCIN~\citep{scin} & Clinical & US crowdsourced users & M &
    Diverse smartphone photos with demographics. \\
    \midrule
    SD-198~\citep{sd198} & Clinical & China dermatology clinic & S-M &
    198-category long-tail dataset. \\
    \midrule
    BCN20000~\citep{bcn20000} & Dermoscopic & Barcelona tertiary center & M-L &
    Large European dermoscopy cohort. \\
    \midrule
    HAM10000~\citep{ham10000} & Dermoscopic & Austria \& Australia & M &
    Classic dermoscopy benchmark. \\
    \midrule
    Derm12345~\citep{derm12345} & Dermoscopic & Turkish hospital & M &
    40-class dermoscopic dataset. \\
    \midrule
    MILK10k~\citep{milk10k} & Clinical \& dermoscopic & ISIC multimodal cohort & M &
    Paired clinical/dermoscopy with metadata. \\
    \bottomrule
  \end{tabular}
  \caption{Summary of the fourteen source dermatology datasets used to construct DermoInstruct and DermoBench. ``Scale'' is qualitative (S: $<$5k images, M: 5k-20k, L: $>$20k).}
  \label{tab:source-datasets}
\end{table*}

\subsection{Source Dermatology Datasets}
\label{app:datasets}

To construct DermoInstruct and DermoBench, we aggregate
fourteen public or institutionally curated dermatology datasets
covering clinical photographs, dermoscopic images, and
smartphone or teledermatology photos from diverse healthcare
systems:
Daffodil~\citep{daffodil},
DermNet~\citep{dermnet},
Fitzpatrick17k~\citep{f17k},
ISIC Archive~\citep{isic},
MIDAS~\citep{midas},
PAD-UFES-20~\citep{pad},
PASSION~\citep{passion},
PUMCH~\citep{pumch},
SCIN~\citep{scin},
SD-198~\citep{sd198},
BCN20000~\citep{bcn20000},
HAM10000~\citep{ham10000},
Derm12345~\citep{derm12345}, and
MILK10k~\citep{milk10k}.
Note that images hosted on the ISIC platform that are not part of these named subsets are grouped into ``ISIC Archive'' collection. These datasets span pigmented and non-pigmented lesions,
benign and malignant conditions, a wide range of anatomic
sites and skin tones, and both controlled and real-world
acquisition conditions.
We briefly summarize their scope in
Table~\ref{tab:source-datasets}; the main paper focuses on
the unified ontology and task construction built on top of
these sources.

Across these datasets, we harmonize heterogeneous diagnosis labels into
a unified hierarchy of superclasses and subclasses, and map existing
attribute schemas (e.g., dermoscopic structures, pigmentation patterns)
into a common morphology ontology used consistently throughout
DermoInstruct and DermoBench.

\subsection{Leakage prevention and de-duplication.}
\label{app:646k}
We split data at the \emph{patient} level (all images from the same \texttt{patient\_id} are confined to a single split), allowing multiple cases per patient in the test set but ensuring no patient overlap with training.
We exclude images from DDI, SCIN, PAD, SkinCon, and Derm7pt from training and reserve them for evaluation-only settings.
Finally, we apply near-duplicate filtering with perceptual hashing (pHash; Hamming distance $\le 2$) to remove visually redundant images.
In total, we retain 646,018 pairs for training after leakage controls and de-duplication.

\subsection{Dermatology Benchmarks and Vision-Language Models}
\label{sec:appendix-derm-benchmarks}

Traditional deep-learning systems for dermatology have focused on
single-image diagnosis of a limited set of conditions, often trained
and evaluated on individual datasets such as HAM10000 or ISIC,
and commonly framed as closed-set classification
tasks~\citep{kshirsagar2022deep,alsuwaidan2023deep,noronha2023deep,ddi}.
Recent work has begun to emphasize both fairness and robustness,
highlighting disparities across skin tones and acquisition conditions
and calling for more diverse benchmarks~\citep{f17k,sd198,scin,ddi}.

In parallel, several multimodal and vision-language dermatology
datasets and models have emerged.
MAKE~\citep{make} pre-trains a dermatology VLM with multi-aspect
knowledge, and PanDerm~\citep{panderm} proposes a dermatology
vision foundation model trained on large-scale multimodal data.
SkinGPT-4~\citep{skingpt}, Skin-R1~\citep{skinr1}, and
SkinGPT-R1~\citep{skingptr1} explore instruction-tuning and
reasoning-style training for dermatology LLMs.
DermBench~\citep{dermbench} and DermaVQA~\citep{dermavqa}
provide evaluation datasets for diagnostic narratives and
question answering,
while SkinCap~\citep{skincap} and SkinCaRe~\citep{skincare}
enrich image-text pairs with medical captions and chain-of-thought
reasoning.
More recently, Derm1M~\citep{derm1m} and DermaSynth~\citep{dermasynth}
scale dermatology vision-language data to the million-sample regime.

Beyond dermatology, there is a growing ecosystem of multimodal medical
benchmarks and foundation models, such as GEMEX for chest
X-ray VQA~\citep{gemex},
PathGen for pathology image-text pairs~\citep{pathgen},
EndoBench for endoscopy~\citep{endobench}, and
VisionUnite for ophthalmology~\citep{visionunite}.
Compared to these efforts, DermoBench is specifically
designed to evaluate dermatology MLLMs along a morphology
$\to$ reasoning $\to$ diagnosis axis, with fairness and robustness
explicitly foregrounded.

\subsection{Concept Bottleneck Models and Morphology-Grounded Reasoning}
\label{sec:appendix-cbm}

Concept bottleneck models (CBMs) explicitly insert an
interpretable concept layer between raw features and task
predictions: the model first predicts a vector of
human-understandable concepts and then predicts the final label
from those concepts~\citep{cbm}.
Such models allow users to inspect and intervene on the
intermediate concept predictions, improving transparency and
enabling richer human-model interaction.
Subsequent work has studied robustness, intervention strategies,
and automatic discovery of concepts,
but the core idea remains to align model internals with
domain-relevant abstractions.

Dermatology is naturally aligned with the CBM paradigm, because
clinical practice is organized around lesion morphology.
Dermatologists rely on structured morphology descriptors in both
clinical and dermoscopic settings~\citep{morphology1,morphology2,morphology3},
and recent datasets such as the SkinCon schema and the dermoscopic
seven-point checklist provide explicit morphology annotations
for skin lesions~\citep{skincon,derm7pt}.
Our benchmark instantiates a \emph{soft} concept bottleneck for
dermatology: Task~1 evaluates morphology descriptions and attributes,
Task~3 assesses chain-of-thought reasoning grounded in these concepts,
and Task~2 measures whether diagnoses are consistent with both.
Rather than inserting a fixed-dimensional concept layer into a single
network, we expose morphology, reasoning, and diagnosis as separate
but tightly coupled tasks, and exploit cross-task consistency as both
a training signal (via DermoInstruct) and an evaluation criterion
(via DermoBench).

\subsection{Reinforcement Learning, GRPO, and Instruction-Tuned MLLMs}
\label{sec:appendix-grpo}

Reinforcement learning has become a central tool for enhancing the
reasoning capabilities of large language models beyond standard
supervised fine-tuning.
DeepSeekMath~\citep{deepseekmath}, for example, combines
continued pre-training on math-heavy corpora with RL and introduces
\emph{Group Relative Policy Optimization} (GRPO), a variant of PPO
that replaces a learned value-function critic with a group-based
baseline over multiple sampled trajectories.
GRPO-style objectives have quickly been adopted in reasoning-focused
LLMs because they are sample-efficient, remove the need for a separate
critic network, and work well with verifiable or heuristic reward
signals.

Our MAVIC framework is inspired by this line of work but tailors the
reward design to dermatology.
Instead of rewarding only final correctness, we combine multiple
terms capturing hierarchical diagnosis correctness, proximity in the
ontology, morphology-grounded agreement with Task~1 outputs, and
format constraints.
This connects GRPO-style RL with clinical desiderata such as
lesion understanding and cross-skin-type robustness, and is
complementary to standard instruction-tuning with LoRA
adaptation~\citep{lora} and chain-of-thought prompting~\citep{cot}
used in general-purpose MLLMs such as Qwen3-VL and Gemini
\citep{qwen3vl,gemini,tang} and in domain-specific models such as
SkinGPT-4 and Skin-R1~\citep{skingpt,skinr1,skingptr1,vargpt,lisa}.

\begin{figure*}[ht]
  \centering
  
  \subcaptionbox{}{\includegraphics[height=4.3cm]{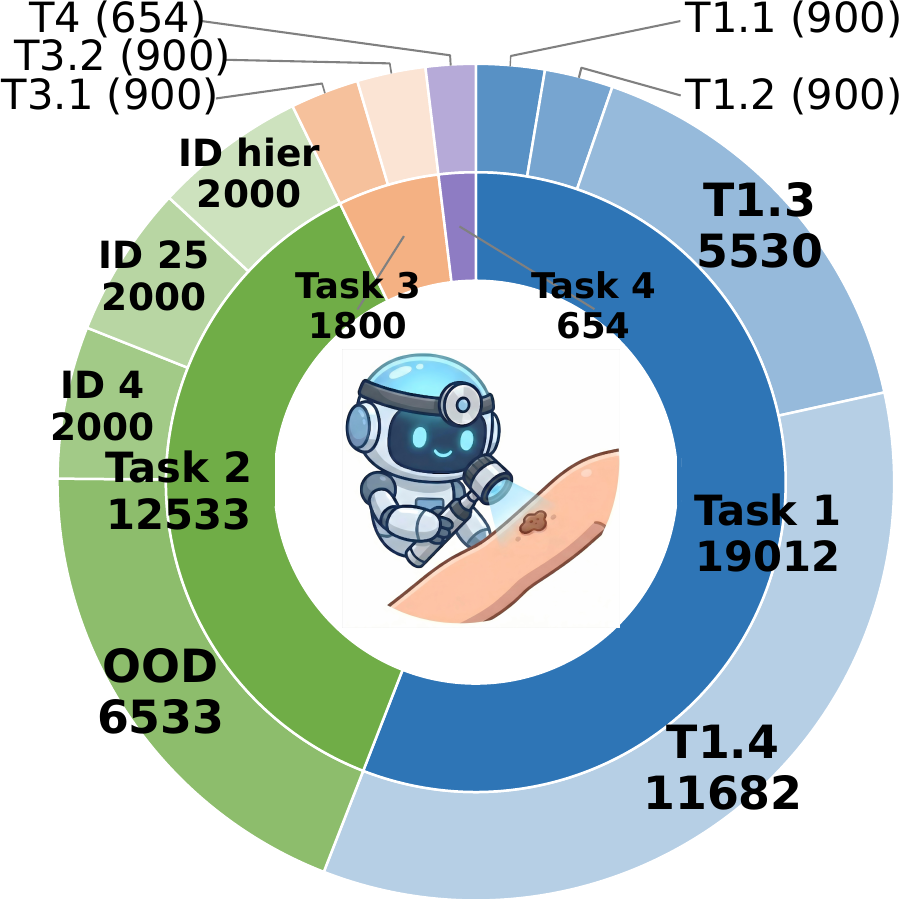}}
  \hfill
  \subcaptionbox{}{\includegraphics[height=4.3cm]{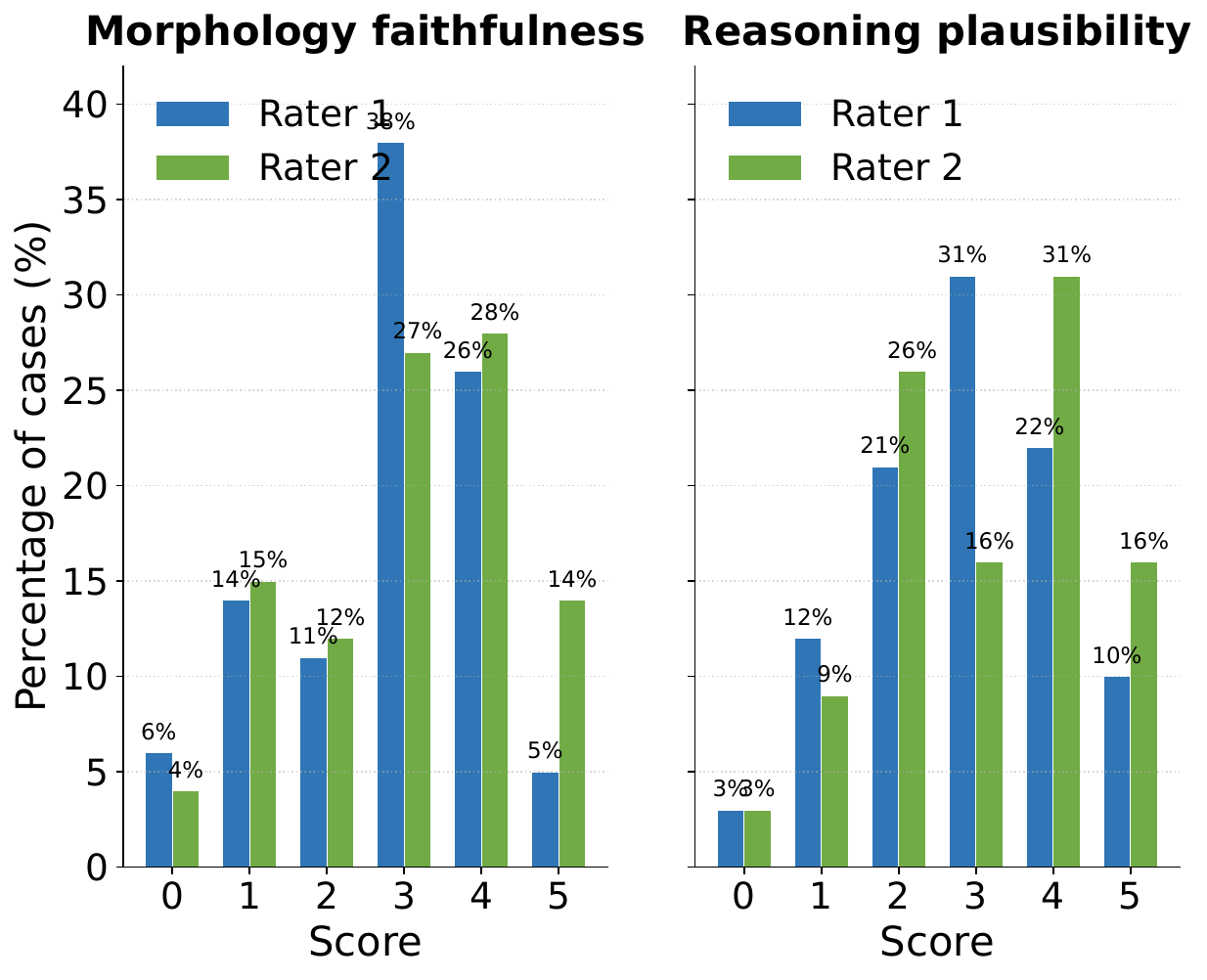}}
  \hfill
  \subcaptionbox{}{\includegraphics[height=4.5cm]{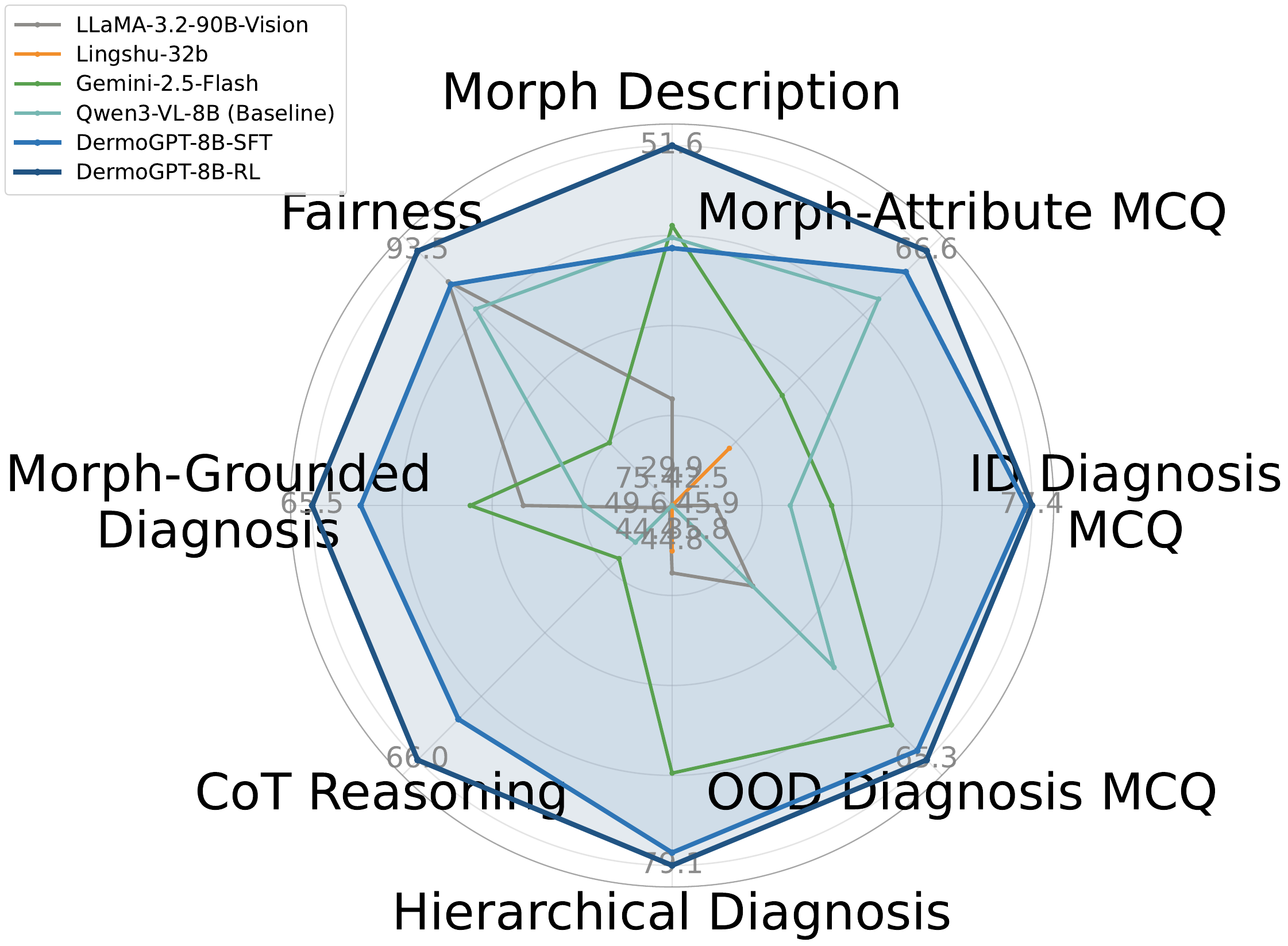}}

  \caption{Benchmark statistics and key evaluation dimensions of DermoBench and DermoInstruct.
(a) Task-wise and sub-task-wise distribution of VQA pairs.
(b) Human ratings of synthesized morphological features and CoT of DermoInstruct.
(c) Performance of representative MLLMs.}
  \label{fig:3-move-supp}
  \vspace{-0.3cm}
\end{figure*}

\subsection{Test-Time Adaptation and Test-Time Scaling}
\label{sec:appendix-tta-tts}

\paragraph{Test-time adaptation.}
Test-time adaptation (TTA) adapts a pre-trained model to unlabeled
test data at deployment time, typically to mitigate covariate shifts
without full re-training.
Classical domain adaptation methods ~\citep{iosda, tent} update batch-normalization statistics or
minimize prediction entropy, while more recent work explores online
adaptation, pseudo-labeling, and robustness under dynamic
streams~\citep{tta_survey, rnlm, atri}.
For vision-language models, recent methods study both optimization
based and optimization-free strategies.
ZERO~\citep{zero} shows that a surprisingly strong VLM TTA baseline
can be obtained by aggressive test-time augmentation, temperature-$0$
prediction, and confidence-based marginalization, requiring only a
single batched forward pass and no backpropagation.
These results demonstrate that much of the benefit of prompt-tuning
style TTA can be captured by carefully designed test-time inference
procedures.

Our CCT framework is complementary to
these methods.
Instead of updating model parameters, we adapt how the model is
\emph{queried} and how multiple stochastic predictions are aggregated:
we sample multiple responses under morphology- and diagnosis-focused
prompting, then aggregate them using confidence- and
consistency-based weighting across tasks, images, and augmentations.
This can be seen as a lightweight, domain-specific TTA scheme that
relies on cross-task dermatology priors rather than parameter updates.

\paragraph{Test-time scaling.}
Test-time scaling (TTS) refers to improving model performance by
allocating more compute at inference time without changing model
parameters~\citep{lisa, s1}.
In the LLM literature, canonical examples include chain-of-thought
prompting with self-consistency, where multiple reasoning paths are
sampled and the majority answer is selected, and best-of-$n$
sampling guided by task-specific scorers~\citep{cot}.
Such techniques can substantially improve reasoning quality but
incur linear cost in the number of samples.

Our CCT procedure can be interpreted as a specialized TTS scheme for
dermatology MLLMs.
By combining multi-sample decoding with confidence- and
consistency-based aggregation across morphology, reasoning, and
diagnosis tasks, CCT leverages the structure of DermoBench to
stabilize predictions under distribution shifts (e.g., across
devices or skin-tone groups) while keeping computation modest
relative to naive best-of-$n$ sampling.

\section{DermoBench Task Definitions and Data Sources}
\label{app:dermobench_tasks}

\subsection{Task Overview and Sample Statistics}
\label{app:dermobench_overview}

Table~\ref{tab:dermobench_tasks} summarizes all DermoBench subtasks, data sources, and sample sizes.
The complete benchmark contains 33,999 VQA-style samples, distributed as follows: Task 1 has 19,012 samples; Task 2 has 12,533; Task 3 has 1,800; and Task 4 has 654.

\subsection{Training Isolation and Leakage Control (Clean Separation)}
\label{app:dermobench_split}

To ensure credible evaluation results, DermoBench implements the following isolation strategies:

\paragraph{(1) Image-level Isolation.} Unless explicitly stated, DermoBench images are sourced from datasets unused in DermoInstruct or from strictly held-out splits of the same source datasets. Critical morphology evaluation datasets such as Derm7pt and SkinCon are designated as evaluation-exclusive sources, with no images or labels utilized for training.

\paragraph{(2) Text-level Isolation.} Reference texts for all open-ended tasks (T1.1/T1.2/T3.1/T3.2)—including morphological reports, attribute JSONs, reasoning chains, and diagnostic statements—are excluded from training corpora to prevent artificially inflated performance through answer memorization.

\paragraph{(3) Question/Template-level Isolation.}
Both multiple-choice and open-ended tasks employ minimal sets of semantically equivalent templates. We perform rigorous deduplication checks between training and evaluation template sets, and provide complete template inventories with cryptographic hashes for reproducibility upon release (see Appendix~\ref{app:prompts}).

\subsection{Morphology Understanding (Task 1.x)}
\label{app:task1}

\subsubsection{T1.1--T1.2: Open-Ended Morphology Evaluation on 900-Case Core Set}
\label{app:task1_open}

\paragraph{Input.}
A single clinical or dermoscopic image + instruction.

\paragraph{Output and Format Constraints.}
\begin{itemize}
\item \textbf{T1.1 (Morph report)}: Generate a structured morphological examination report covering key aspects including lesion type, color, border, surface/scales, and distribution.
\item \textbf{T1.2 (Morph JSON + report)}: In addition to the report, output a JSON object wrapped in \texttt{<morph>}...\texttt{</morph>} tags. Dermoscopic images follow Derm7pt checklist fields; clinical images follow SkinCon fields.
\end{itemize}

\paragraph{Gold Standard Construction (Core Process).}
We first use a strong VLM to generate for each core set image: (i) morphological report, (ii) attribute JSON, and (iii) diagnostic reasoning with final diagnosis (for Task 3.x). Dermatologists then conduct line-by-line review and revision to ensure (a) textual descriptions align with visible evidence in images, (b) JSON field values conform to clinical terminology and definitions, and (c) consistency between descriptions and diagnoses. Detailed review guidelines, conflict resolution examples, and final consistency checks are provided in Appendix~\ref{app:core900_protocol}.

\subsubsection{T1.3: Dermoscopic Attribute MCQA}
\label{app:task1_derm7pt}

\paragraph{Data Source.}
The dermoscopic test split of Derm7pt is used for evaluation. Although Derm7pt provides training splits, we exclude all its images and labels from training.

\paragraph{Question Construction.}
Each question queries one attribute from the Derm7pt checklist (e.g., pigment network, streaks, etc.), with options corresponding to valid states for that attribute. Question templates and option generation rules are specified in Appendix~\ref{app:prompts}.

\subsubsection{T1.4: SkinCon Attribute Multiple-Choice Questions (Clinical Attribute MCQA)}
\label{app:task1_skincon}

\paragraph{Data Source.}
SkinCon does not provide an official test split; we treat all its annotated samples as evaluation-only, generating MCQAs from its morphological annotations.
Question and option construction follow the same principles as above, with fields and value spaces determined by the SkinCon schema.

\subsection{Diagnosis classification (Task 2.x)}
\label{app:task2}

\subsubsection{In-distribution (ID) diagnosis (T2.1--T2.3)}
\label{app:task2_id}

\paragraph{Data sources and partitioning.}
The ID diagnosis evaluation set was constructed by extracting strictly held-out images from the same 14 source datasets as DermoInstruct (completely isolated from training instruction pairs, see Appendix~\ref{app:dermobench_split}).

\paragraph{T2.1: 4-way MCQA (leaf-level).}
The correct option is a fine-grained leaf-node diagnosis; distractors are preferentially sampled from neighboring nodes/siblings under the same parent node in the unified ontology to enhance "clinical confusability."

\paragraph{T2.2: 25-way MCQA (coarse-grained triage).}
The 325 leaf-node diagnoses are collapsed into 25 coarse-grained categories with stronger clinical significance, creating a fixed option menu to simulate real-world triage scenarios.

\paragraph{T2.3: Hierarchical diagnosis.}
A single diagnosis is decomposed into sequential decisions along the ontology path (root$\to$leaf).
Each question corresponds to one step along the path, with both per-level accuracy and path-level metrics measured.

\subsubsection{Out-of-distribution (OOD) diagnosis (T2.4)}
\label{app:task2_ood}

\paragraph{Data sources.}
Evaluation partitions from multiple external dermoscopy/clinical datasets are used, including Derm1M educational split, Derm7pt, DDI, and SNU134.

\paragraph{Key setting: Non-aligned label spaces.}
Unlike ID tasks, OOD tasks construct MCQAs within each dataset's \textbf{original label space}:
We do not map ground-truth labels or options to a unified ontology.
Consequently, models must simultaneously handle visual distribution shifts and label space mismatches, preventing inflated scores from "interpolating" on a unified taxonomy.

\paragraph{MCQA construction.}
For each sample, the original dataset label serves as the correct option; distractors are sampled from the same dataset's label set (potentially weighted by class frequency or confusability).

\subsection{Reasoning (Task 3.x)}
\label{app:task3}

\subsubsection{T3.1: CoT reasoning}
\label{app:task3_cot}

\paragraph{Data and objective.}
We use the same 900-case core set as in T1.1/T1.2.
Models must output reasoning text enclosed in \texttt{<reasoning>}...\texttt{</reasoning>} tags,
connecting visible evidence with candidate diagnoses, and provide the final diagnosis within \texttt{<final\_diagnosis>} tags.

\subsubsection{T3.2: Morph-grounded reasoning}
\label{app:task3_mground}

Building upon T3.1, models are additionally required to output \texttt{<morph>} JSON (with the same schema as in T1.2).
This setup explicitly tests: \emph{whether the morphological evidence documented by the model sufficiently supports its reasoning chain and final diagnosis}.

\paragraph{Consistency check (analysis dimension).}
Beyond open-ended scoring, we additionally perform automated "morphology$\leftrightarrow$diagnosis consistency" checks on the core set:
For example, contradictions are counted when the model declares critical negative features in the JSON (e.g., \emph{no pigment network}) but cites contradictory evidence in its reasoning.

\subsection{Fairness (Task 4.x)}
\label{app:task4}

\paragraph{Data and grouping.}
We reuse the DDI-based 4-way MCQAs and group images according to Fitzpatrick skin type (FST I--V).

\paragraph{Fairness metric.}
Let $\mathrm{Acc}_k$ denote the model accuracy for each group. Fairness is defined as:
\[
\mathrm{Fairness} \;=\; \frac{\min_k \mathrm{Acc}_k}{\max_k \mathrm{Acc}_k}.
\]
This metric achieves higher values when overall performance is high and performance gaps across skin tone groups are small.
In addition to this primary metric, we also report per-group accuracies to avoid misinterpretations where "ratios mask absolute performance differences".

\subsection{Gold standard annotation protocol for the 900-case core set}
\label{app:core900_protocol}

\begin{figure*}[ht]
  \centering
  \includegraphics[width=\linewidth]{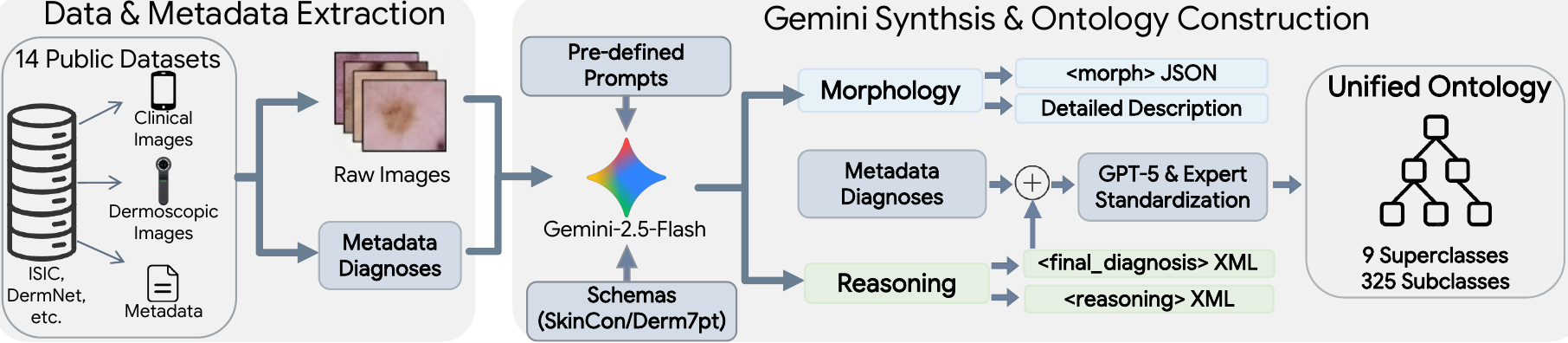}
  \label{fig:vis1}
  \caption{Construction pipeline for DermoInstruct. We aggregate 14 source datasets, apply leakage controls and de-duplication, then generate morphology- and reasoning-grounded instruction pairs using a SOTA multimodal LLM (Gemini-2.5-Flash).
  The final training subset of DermoInstruct dataset contains 646k high-quality image-instruction pairs.
  }
\end{figure*}

\paragraph{Step 1: Draft generation.}
Three types of drafts are generated for each image: (i) morphological report, (ii) attribute JSON (Derm7pt/SkinCon schema), and (iii) reasoning chain + final diagnosis.

\paragraph{Step 2: Clinical line-by-line revision.}
Two dermatologists conduct line-by-line review and revision of the drafts, with focus on correcting:
(a) invisible or exaggerated morphological descriptions; (b) JSON field values inconsistent with definitions; (c) reasoning inconsistent with morphological evidence; (d) diagnoses unsupported by the evidence chain.

\paragraph{Step 3: Consistency and format validation.}
We perform format validation (tag/JSON parseability) and consistency checks (morphology$\leftrightarrow$reasoning$\leftrightarrow$diagnosis) for all samples.
If conflicts are detected, we return to Step 2 and iterate until all checks pass.

\paragraph{Step 4: Quality spot-checking and documentation.}
A random subset undergoes dual review by two annotators, with common error patterns documented and revision guidelines updated to ensure annotation consistency and scalability.

\subsection{Concept bottleneck tasks}
\label{app:concept_bottleneck}

\paragraph{Task motivation.}
T1.2 and T3.2 enforce the output of standardized morphological concepts (the \texttt{<morph>} JSON), using “interpretable morphological evidence” as a diagnostic intermediate bottleneck. This extends evaluation from merely “whether the answer is correct” to “whether the evidence chain is auditable and self-consistent.”

\paragraph{Output format and ordering constraints.}
Both tasks require outputting parseable JSON enclosed within \texttt{<morph>}...\texttt{</morph>} tags:
(i) T1.2: the \texttt{<morph>} tag is placed before the morphological report;
(ii) T3.2: output the \texttt{<morph>} JSON right after the \texttt{<reasoning>} paragraph, and finally the \texttt{<final\_diagnosis>}.

\section{Prompt Templates and Example Outputs for DermoInstruct}
\label{app:prompts}

\subsection{Morphology and Reasoning Supervision}

We obtain morphology-centric supervision by querying a SOTA multimodal LLM, Gemini-2.5-Flash \citep{gemini}, with a small set of templates for every image. For each case, the model is asked to (i) describe the lesion in free text, (ii) output a structured set of morphology attributes, and (iii) perform step-by-step diagnostic reasoning that ends in a final diagnosis chosen from a candidate list. This provides a unified image-to-text pipeline whose outputs are reused across DermoInstruct and DermoBench.

We distinguish clinical and dermoscopic images only through the morphology schema. For clinical photographs, prompts align Gemini’s outputs with the 48 SkinCon concepts \citep{skincon}, returning a short report plus a JSON object indicating which attributes are present. For dermoscopic photographs, we instead condition on the seven-point checklist \citep{derm7pt} to obtain an analogous JSON over dermoscopic structures and a brief dermoscopy report. In both cases, the model must first commit to morphology before predicting any disease label. We then augment each case with chain-of-thought (CoT) supervision \citep{cot}: given the image, the morphology JSON, and a small candidate set of diagnoses derived from metadata and our ontology (Sec.~2.1.3), Gemini produces a reasoning paragraph and a \texttt{<final\_diagnosis>} tag selecting one fine-grained diagnosis.
\subsection{Morphology JSON Prompts}

\subsubsection{Clinical Images (SkinCon)}
\begin{promptbox}{SkinCon clinical prompt (system + user)}
PROMPT_DICT = { 
 "system_prompt": "You are an expert in dermatology. Your task is to perform a detailed visual analysis of a provided skin lesion image (clinical or dermoscopic). You will be given an image of a skin lesion and a predefined list of 48 standardized clinical concepts from the SkinCon dataset. Your task is to analyze the image, describe it clinically, and then map the observed features to the provided SkinCon concepts. Any features you observe that are not on the list must be categorized separately. Your output must be a single, clean JSON object and nothing else.", 
 "user_prompt": "Analyze the provided skin lesion image using the established SkinCon vocabulary. First, perform a detailed, step-by-step visual assessment. Second, generate a single, valid JSON object as your final and ONLY output. Do not include any text, explanations, or markdown formatting outside of the JSON object.\n\n### SkinCon Morphological Concepts List\nHere are the 48 standardized concepts you MUST use for classification:\n1. Abscess\n2. Acuminate\n3. Atrophy\n4. Black\n5. Blue\n6. Brown(Hyperpigmentation)\n7. Bulla\n8. Burrow\n9. Comedo\n10. Crust\n11. Cyst\n12. Dome-shaped\n13. Erosion\n14. Erythema\n15. Excoriation\n16. Exophytic/Fungating\n17. Exudate\n18. Fissure\n19. Flat topped\n20. Friable\n21. Gray\n22. Induration\n23. Lichenification\n24. Macule\n25. Nodule\n26. Papule\n27. Patch\n28. Pedunculated\n29. Pigmented\n30. Plaque\n31. Poikiloderma\n32. Purple\n33. Purpura/Petechiae\n34. Pustule\n35. Salmon\n36. Scale\n37. Scar\n38. Sclerosis\n39. Telangiectasia\n40. Translucent\n41. Ulcer\n42. Umbilicated\n43. Vesicle\n44. Warty/Papillomatous\n45. Wheal\n46. White(Hypopigmentation)\n47. Xerosis\n48. Yellow\n\n### Required JSON Output Structure\nThe JSON object MUST contain exactly three keys:\n1. detailed_description: (String) A comprehensive clinical narrative of the lesion's morphology, including primary lesion type, color, shape, border, surface, and texture.\n2. morphological_features_skincon: (Array of Strings) A list of all observed features that EXACTLY MATCH one or more terms from the 48 SkinCon concepts provided above.\n3. morphological_features_others: (Array of Strings) A list of important observed features that are NOT found in the SkinCon list. If all features are covered by the SkinCon list, this array should be empty [].\n\n ### Examples for Guidance\n\n**INPUT:** [Image of a Psoriasis Plaque] \n**REQUIRED JSON OUTPUT:**\n{\n \"detailed_description\": \"The image shows a sharply demarcated, erythematous plaque with a raised, indurated surface. The lesion is ovoid and its borders are well-defined. The surface is covered by a thick layer of silvery-white, lamellar scales. The perilesional skin appears unremarkable.\",\n \"morphological_features_skincon\": [\n \"Plaque\",\n \"Erythema\",\n \"Scale\"\n ],\n \"morphological_features_others\": [\n \"well-demarcated\",\n \"silvery-white\"\n ]\n}\n\n---\n### YOUR TASK\nNow, for the image I have provided, please perform the same analysis and generate the JSON output. Remember, the JSON object is the only thing you should return.\n" 
}
\end{promptbox}

\begin{promptbox}[colback=gray!5,colframe=gray!60]{SkinCon example JSON}
{""detailed_description"": ""The image shows multiple digits (toes) affected by severe onychodystrophy and prominent periungual inflammation. The nail plates are markedly thickened, opaque, and display significant discoloration, predominantly yellow and brownish hues. Many nails exhibit onycholysis, appearing separated from the nail bed, often with underlying subungual hyperkeratosis. The surrounding periungual skin and distal phalanges are diffusely erythematous, swollen, and indurated, indicative of chronic inflammation. Localized areas of scaling and subtle crusting are also observed on the inflamed periungual tissue."", ""morphological_features_skincon"": [ ""Yellow"", ""Brown(Hyperpigmentation)"", ""Erythema"", ""Scale"", ""Induration"", ""Crust"" ], ""morphological_features_others"": [ ""Onychodystrophy"", ""Onycholysis"", ""Subungual hyperkeratosis"", ""Nail thickening"", ""Opaque nails"", ""Swelling"", ""Periungual inflammation"" ] }
\end{promptbox}

\subsubsection{Dermoscopic Images (Derm7pt)}
\begin{promptbox}{Derm7pt dermoscopic prompt (system + user)}
PROMPT_DICT = { 
 "system_prompt": "You are an expert in dermatology. Your task is to perform a detailed visual analysis of a provided dermoscopic image. You will analyze the image and classify its features according to the 7-point checklist, assigning the single most fitting morphological label to each of the seven criteria. Your output must be a single, clean JSON object and nothing else.", 
 "user_prompt": "Analyze the provided skin lesion image using the established Derm7pt vocabulary. First, perform a detailed, step-by-step visual assessment. Second, for each of the 7 criteria, select the single most appropriate label from the lists provided below. Finally, generate a single, valid JSON object as your final and ONLY output. Do not include any text, explanations, or markdown formatting outside of the JSON object.\n\n### Derm7pt Morphological Concepts and Labels\nYou MUST classify the lesion by selecting exactly one label for each of the 7 criteria:\n\n1. **pigment_network**: [\"absent\", \"typical\", \"atypical\"]\n2. **blue_whitish_veil**: [\"absent\", \"present\"]\n3. **vascular_structures**: [\"absent\", \"arborizing\", \"comma\", \"hairpin\", \"within regression\", \"wreath\", \"dotted\", \"linear irregular\"]\n4. **pigmentation**: [\"absent\", \"diffuse regular\", \"localized regular\", \"diffuse irregular\", \"localized irregular\"]\n5. **streaks**: [\"absent\", \"regular\", \"irregular\"]\n6. **dots_and_globules**: [\"absent\", \"regular\", \"irregular\"]\n7. **regression_structures**: [\"absent\", \"blue areas\", \"white areas\", \"combinations\"]\n\n### Required JSON Output Structure\nThe JSON object MUST contain exactly three keys:\n1. detailed_description: (String) A comprehensive clinical narrative of the lesion's morphology, including primary lesion type, color, shape, border, surface, and texture, justifying your label choices.\n2. morphological_features_Derm7pt: (Object) An object where each key is one of the 7 Derm7pt criteria and its value is a single (String) label selected from the lists above.\n3. morphological_features_others: (Array of Strings) A list of important observed features that are NOT part of the 7-point checklist classification (e.g., symmetry, specific colors). If none, this array should be empty [].\n\n### Examples for Guidance\n\n**INPUT:** [Dermoscopic image of a melanoma]\n**REQUIRED JSON OUTPUT:**\n{\n \"detailed_description\": \"Dermoscopy reveals a chaotic and asymmetrical lesion. The pigment network is thickened and irregular, with variable hole sizes and abrupt cut-offs at the periphery, classifying it as 'atypical'. Irregular streaks are visible radiating from the main body. There are multiple blotches of dark brown and black pigment concentrated in one quadrant, consistent with 'localized irregular' pigmentation. Additionally, a peppering of various-sized gray-black dots and globules is present, indicating an 'irregular' pattern. The lesion also features both scar-like white areas and peppercorn-like blue areas, which points to 'combinations' of regression structures. Abnormal linear irregular vessels are noted. A blue-whitish veil is absent.\",\n \"morphological_features_Derm7pt\": {\n \"pigment_network\": \"atypical\",\n \"blue_whitish_veil\": \"absent\",\n \"vascular_structures\": \"linear irregular\",\n \"pigmentation\": \"localized irregular\",\n \"streaks\": \"irregular\",\n \"dots_and_globules\": \"irregular\",\n \"regression_structures\": \"combinations\"\n },\n \"morphological_features_others\": [\n \"asymmetry\",\n \"chaotic appearance\",\n \"color variegation (dark brown, black, gray-black, white, blue)\"\n ]\n}\n\n---\n### YOUR TASK\nNow, for the image I have provided, please perform the same analysis and generate the JSON output. Remember, the JSON object is the only thing you should return.\n" 
}
\end{promptbox}

\subsection{Chain-of-Thought Reasoning Prompt}
\begin{promptbox}{CoT reasoning prompt (system + user)}
PROMPT_DICT = { 
 "system_prompt": "You are an expert dermatologist AI, acting as a clinical consultant. Your primary task is to analyze a skin lesion image and generate a concise clinical reasoning narrative. You will be provided with potential clinical concepts (which may not be entirely accurate) and a confirmed diagnosis. You must critically evaluate the visual evidence in the image to explain how it supports the diagnosis, adhering to a strict XML format for your output.", 
 "user_prompt_template": "Analyze the provided image and its context. Your entire output must be a structured response containing a reasoning block (<reasoning>) and a final diagnosis block (<final_diagnosis>).\n\n### Input Context\n* **Image:**\n\n* **Potential Clinical Concepts:** {clinical_concepts}\n* **Confirmed Diagnoses:** {diagnoses}\n\n### Your Task\nYour response MUST follow these three rules precisely:\n1. **First,** provide a step-by-step clinical rationale explaining how the visual evidence in the image leads to the confirmed diagnosis. Your explanation should be from the perspective of an expert explaining the case to a colleague. Ground your reasoning in the visual features of the lesion (e.g., shape, color, border, texture, specific structures). Use the 'Potential Clinical Concepts' as a guide, but your primary justification must come from the image itself. Enclose this entire process within <reasoning> and </reasoning> tags.\n2. **Second,** provide the most specific diagnosis from the 'Confirmed Diagnoses' list inside <final_diagnosis> and </final_diagnosis> tags.\n3. **Third,** ensure there is absolutely NO extra text, explanation, or markdown formatting outside of these two required XML tags.\n\n### Example for Guidance\n\n**INPUT CONTEXT:**\n* **Image:** [Dermoscopic image of a melanoma]\n* **Potential Clinical Concepts:** [\"Asymmetry\", \"Irregular Border\", \"Color Variegation (Brown, Black, Blue-Gray)\", \"Atypical Pigment Network\"]\n* **Confirmed Diagnoses:** [\"Malignant\", \"Malignant Melanoma\"]\n\n**REQUIRED OUTPUT:**\n<reasoning>Upon examining the image, the lesion exhibits several hallmark features concerning for malignancy. There is clear asymmetry in its overall shape and the border is poorly defined and irregular, with notches and blurred edges in several areas. I observe significant color variegation, with multiple shades of brown and black, as well as a focal blue-gray area, which is a strong indicator of pigment regression or deep melanin. These observations align with the classic clinical signs for melanoma. The combination of these visual findings provides a strong basis for diagnosing this lesion as a malignant melanoma, differentiating it from a benign nevus.</reasoning><final_diagnosis>Malignant Melanoma</final_diagnosis>\n\n---\n### YOUR TASK\nNow, for the image, concepts, and diagnoses I have provided, generate the response in the required format." 
}
\end{promptbox}
\begin{promptbox}[colback=gray!5,colframe=gray!60]{CoT example XML}
<reasoning>Upon visual inspection, the image displays an erythematous, ill-defined plaque with an irregular shape on the skin. Centrally, there are multiple confluent erosions and ulcerations, appearing moist and suggestive of serous exudate. A yellowish-brown crust is also visible within this central eroded area, indicating dried serous fluid or possibly a secondary bacterial component. The presence of acute erythema, clustered erosions, exudate, and crusting is highly characteristic of an acute viral infection, such as Herpes Simplex Virus. This morphology strongly supports the diagnosis as fitting within the category of "other Viral Infections," as distinct from typical warts or molluscum contagiosum which present differently.</reasoning><final_diagnosis>Warts Molluscum and other Viral Infections</final_diagnosis>
\end{promptbox}




\subsection{Diagnosis VQA prompt templates}
\label{app:diagnosis_prompts}

Using the ontology described above, we synthesize diagnosis VQA items in two forms. First, for flat four-way MCQA questions, we sample one ground-truth diagnosis and three ontology-consistent distractors (typically siblings or closely related conditions), and render them as options A--D. The question stem is drawn at random from a small pool of interchangeable prompts that ask the model to choose the most likely diagnosis. This yields diverse yet semantically equivalent formulations while keeping the underlying label space fixed.

Second, for hierarchical diagnosis VQA, we traverse the ontology level by level. At each step, we present the image and a set of candidate categories, and instantiate one of several templated prompts for (i) selecting a top-level superclass, (ii) refining the choice within its subcategories, and (iii) choosing a final leaf diagnosis. Additional declarative prompts are used to convert the completed path into a natural-language statement of the final diagnosis, and a small set of ``human correction'' prompts supports expert editing when the automatically proposed path is incorrect.

Together, these instruction types give dense supervision over both \emph{what} diagnosis to output and \emph{how} to traverse and correct a hierarchical diagnostic reasoning process.

\begin{promptbox}{PROMPTS for 4-way diagnosis MCQA}
PROMPTS = [
    "Observe this skin image. Which of the following diagnoses is the most likely?",
    "Based on the skin lesion shown in this image, please select the most accurate diagnosis from the options below.",
    "Which of the following diagnoses best matches the skin condition shown in this image?",
    "Considering the clinical presentation of the skin lesion in the image, which of the following is the most likely diagnosis?"
]
\end{promptbox}

\begin{promptbox}{TOP\_LEVEL\_PROMPTS\_GEN}
TOP_LEVEL_PROMPTS_GEN = [
    "Based on the clinical image, identify the most fitting major dermatological category from the following list: {options_list}.",
    "Observe the skin lesion. Which of these high-level classifications best describes it? Here are the possibilities: {options_list}.",
    "Please provide a broad categorization for the skin condition shown. Your answer should be one of the following: {options_list}."
]
\end{promptbox}

\begin{promptbox}{SUB\_LEVEL\_PROMPTS\_GEN}
SUB_LEVEL_PROMPTS_GEN = [
    "Correct, the condition is a form of '{parent_category}'. Now, specify the sub-category from this list: {options_list}.",
    "Proceeding from '{parent_category}', which of the following groups does this lesion belong to? {options_list}.",
    "Understood. Let's refine the diagnosis within '{parent_category}'. Please choose the most accurate description from the following: {options_list}."
]
\end{promptbox}

\begin{promptbox}{FINAL\_LEVEL\_PROMPTS\_GEN}
FINAL_LEVEL_PROMPTS_GEN = [
    "We've classified this under '{parent_category}'. Now, provide the definitive diagnosis from the choices available: {options_list}.",
    "Excellent. To finalize, please state the specific diagnosis for '{parent_category}', which should be one of the following: {options_list}.",
    "Perfect. Based on our hierarchical classification ending with '{parent_category}', please identify the definitive diagnosis from this list: {options_list}."
]
\end{promptbox}

\begin{promptbox}{DECLARATIVE\_PROMPTS}
DECLARATIVE_PROMPTS = [
    "Following the diagnostic path to '{parent_category}', the evidence points to a single definitive diagnosis, which is {final_diagnosis}.",
    "Correct. The reasoning has led us to '{parent_category}', which contains only one specific condition. Therefore, the diagnosis must be {final_diagnosis}.",
    "Excellent. Since '{parent_category}' is the most specific category and it corresponds to a single diagnosis, we can conclude the condition is {final_diagnosis}."
]
\end{promptbox}

\begin{promptbox}{HUMAN\_CORRECTION\_PROMPTS}
HUMAN_CORRECTION_PROMPTS = [
    "That's not quite right. While '{wrong_choice}' is a possibility, the visual evidence points more strongly to '{correct_choice}'. Let's proceed with the correct category.",
    "Actually, that's incorrect. A closer look reveals features more consistent with '{correct_choice}'. Please correct the path.",
    "Incorrect. The diagnosis should be '{correct_choice}', not '{wrong_choice}'. Let's continue from the right category.",
    "I disagree. '{correct_choice}' is the more accurate classification here. Let's use that one instead."
]
\end{promptbox}

\section{LLM-as-a-Judge Prompts}
\label{app:judge_prompts}
We use a text-only LLM-as-a-Judge protocol: the judge \emph{does not see the image} and evaluates by comparing the \textbf{REFERENCE} text versus the \textbf{CANDIDATE} text under a strict dermatology morphology rubric.
All tasks output a scalar \texttt{final\_overall} in $[0,100]$ and we report \texttt{mean\_final\_overall} in the main paper.

\subsection{Task 1.1 (Morph Description)}

\begin{promptbox}{Task 1.1 -- SYSTEM PROMPT}
You are a strict, no-nonsense clinical dermatology evaluator.
You DO NOT see the image; evaluate ONLY by comparing the REFERENCE vs the CANDIDATE text.
Use dermatology morphology standards. Avoid rewarding verbosity; penalize contradictions and invented findings.
Focus on: anatomical site, number/arrangement, primary lesion types, color, shape, borders, surface features, size/extent,
distribution/pattern, and special/contextual features (e.g., pen markings, dermoscopic 7-point structures if applicable).
Return STRICT JSON only.
\end{promptbox}

\begin{promptbox}{Task 1.1 -- USER PROMPT TEMPLATE}
[Task Prompt]
\{task\_prompt\}

[REFERENCE]
\{reference\}

[CANDIDATE]
\{candidate\}

Evaluate as follows:
1) Decompose REFERENCE into <=25 atomic CLAIMS.
2) For each CLAIM, label wrt CANDIDATE: Supported, PartiallySupported, Contradicted, Missing, or Vague.
3) Identify any EXTRA INCORRECT statements in CANDIDATE.
4) Score:
   recall_like = (Supported + 0.5*PartiallySupported) / max(1, total_ref_claims)
   precision_penalty = min(1.0, (Contradicted + ExtraIncorrect) / max(1, total_ref_claims))
   overall [0-100] = round(100 * max(0, recall_like - 0.5*precision_penalty), 1)
   Provide rubric sub-scores (accuracy, completeness, consistency) in [0,1].

JSON ONLY. Schema:
\{
  "claims": [\{"text":"...","label":"Supported|PartiallySupported|Contradicted|Missing|Vague"\}],
  "counts": \{"supported":0,"partial":0,"contradicted":0,"missing":0,"vague":0,"extra\_incorrect":0,"total\_ref\_claims":0\},
  "rubric": \{"accuracy":0.0,"completeness":0.0,"consistency":0.0\},
  "overall": 0.0,
  "short\_feedback": "<=40 words concise justification"
\}
\end{promptbox}

\subsection{Task 1.2 (Morph Content + Narrative)}

\begin{promptbox}{Task 1.2 -- SYSTEM PROMPT}
You are a strict dermatology evaluator for Task 1.2 (morph content + narrative).
You DO NOT see the image. Focus on CONTENT, not formatting.
Both REFERENCE and CANDIDATE may or may not wrap the morph JSON in <morph> tags.
Do NOT penalize missing tags, extra whitespace, or minor ordering/format differences.
If a JSON block is present anywhere, treat the FIRST JSON object as the morph content.
If no JSON is present, infer the morph feature set from the surrounding text.
Schemas you may encounter:
- SkinCon: \{"morphological\_features\_skincon": [<feature strings>]\}
- Derm7pt: \{"morphological\_features\_Derm7pt": \{pigment\_network, blue\_whitish\_veil, vascular\_structures, pigmentation, streaks, dots\_and\_globules, regression\_structures\}\}
For the narrative comparison, use dermatology morphology standards (site, number/arrangement, primary lesion types, color, shape, borders, surface features, size/extent, distribution/pattern, special/context).
Also check CROSS-CONSISTENCY between the CANDIDATE morph content and CANDIDATE narrative.
Return STRICT JSON only.
\end{promptbox}

\begin{promptbox}{Task 1.2 -- USER PROMPT TEMPLATE}
You will be given REFERENCE and CANDIDATE texts.
Each may contain a morph JSON (SkinCon or Derm7pt) with or without <morph> tags,
possibly followed by a narrative paragraph. Do NOT penalize formatting.
Rules:
- If a JSON object appears anywhere, treat the FIRST JSON object as the morph content.
- If no JSON is found, infer the morph feature set from the surrounding text (best-effort).
- Use synonyms tolerance for semantic matching.

[Task Prompt]
\{task\_prompt\}

[REFERENCE]
\{reference\}

[CANDIDATE]
\{candidate\}

Your tasks:
1) MORPH SEMANTICS (content-first): Compare CANDIDATE-morph vs REFERENCE-morph semantically (synonyms allowed).
   Count supported/missing/contradicted/extra and give a semantic score in [0,1].
   If CANDIDATE has no explicit JSON, infer its morph set from the candidate text.

2) TEXT (NARRATIVE): Compare REFERENCE-narrative vs CANDIDATE-narrative using morphology standards.
   Extract <=25 atomic claims from the REFERENCE-narrative; for each, label CANDIDATE as Supported/PartiallySupported/Contradicted/Missing/Vague.
   Provide rubric sub-scores (accuracy, completeness, consistency) in [0,1] and overall [0,100] using:
   recall_like = (Supported + 0.5*PartiallySupported) / max(1, total_ref_claims)
   precision_penalty = min(1.0, (Contradicted + ExtraIncorrect) / max(1, total_ref_claims))
   overall = round(100 * max(0, recall_like - 0.5*precision_penalty), 1)

3) CROSS-CONSISTENCY: Judge if the CANDIDATE narrative contradicts the CANDIDATE morph content.
   Output a penalty in [0,1] (0=no issue, 1=severe) and short notes.

Output STRICT JSON:
\{
  "morph\_semantic": \{
    "schema": "SkinCon" | "Derm7pt" | "Unknown",
    "supported": 0, "missing": 0, "contradicted": 0, "extra": 0,
    "score\_semantic": 0.0,
    "notes": "<=60 words"
  \},
  "text\_judge": \{
    "claims": [\{"text":"...","label":"Supported|PartiallySupported|Contradicted|Missing|Vague"\}],
    "counts": \{"supported":0,"partial":0,"contradicted":0,"missing":0,"vague":0,"extra\_incorrect":0,"total\_ref\_claims":0\},
    "rubric": \{"accuracy":0.0,"completeness":0.0,"consistency":0.0\},
    "overall": 0.0,
    "short\_feedback": "<=40 words"
  \},
  "cross\_consistency": \{"penalty": 0.0, "notes": "<=40 words"\}
\}
\end{promptbox}

\subsection{Task 3.1 (Reasoning + Final Diagnosis)}

\begin{promptbox}{Task 3.1 -- SYSTEM PROMPT}
You are a strict dermatology evaluator for Task 3 (reasoning + final diagnosis).
You DO NOT see the image; evaluate ONLY the textual content. Ignore formatting and tags.
Goal: robustly extract (A) the candidate's reasoning and (B) the candidate's final diagnosis,
then score (1) REASONING ALIGNMENT vs the GT reasoning and (2) DIAGNOSIS SIMILARITY vs the GT final diagnosis.
Penalize contradictions and hallucinated findings. Do not reward verbosity. Return STRICT JSON only.
\end{promptbox}

\begin{promptbox}{Task 3.1 -- USER PROMPT TEMPLATE}
[Task Prompt]
\{task\_prompt\}

[GROUND\_TRUTH\_RAW]
\{reference\}

[CANDIDATE\_RAW]
\{candidate\}

Evaluate with these steps (format-agnostic; focus on content):
A) Extraction (be robust even if the candidate is unstructured):
   - From GROUND_TRUTH_RAW, extract:
     gt_reasoning: inside <reasoning>...</reasoning> if present; else best-effort summary.
     gt_final_dx: inside <final_diagnosis>...</final_diagnosis> if present; else best-effort label.
   - From CANDIDATE_RAW, extract:
     cand_reasoning: the explanation/rationale (anywhere).
     cand_final_dx: the single most likely final diagnosis term/phrase.

B) Reasoning Alignment:
   - Decompose gt_reasoning into <=25 atomic claims.
   - For each claim, label wrt cand_reasoning: Supported | PartiallySupported | Contradicted | Missing | Vague.
   - Compute reasoning_score [0-100] using the same recall/penalty formula.

C) Diagnosis Similarity (graded, not binary):
   - Decide relation: Exact | Synonym | Parent | Child | Sibling/CloseDifferential | SameSuperfamily | UnrelatedPlausible | WrongSystem | Nonsense/NoAnswer.
   - Map to similarity in [0,1] and compute diagnosis_score [0-100].

D) Overall:
   - overall [0-100] = round(0.5 * reasoning_score + 0.5 * diagnosis_score, 1)

STRICT JSON ONLY (use the specified schema in the paper).
\end{promptbox}

\subsection{Task 3.2 (Morph-grounded Reasoning)}

\begin{promptbox}{Task 3.2 -- SYSTEM PROMPT}
You are a strict dermatology evaluator for Task 3.2 (reasoning + morph JSON + final diagnosis).
You DO NOT see the image. Focus on CONTENT, not formatting.
Both REFERENCE and CANDIDATE may or may not wrap the morph JSON in <morph> tags.
Do NOT penalize missing tags, extra whitespace, or ordering differences.
If a JSON object appears anywhere, treat the FIRST JSON object as the morph content.
If no JSON is present, infer the morph feature set from the surrounding text.
SCHEMA SELECTION RULE: Detect the schema used by REFERENCE. Compare and output using the SAME schema.
\end{promptbox}

\begin{promptbox}{Task 3.2 -- USER PROMPT TEMPLATE}
You will be given REFERENCE and CANDIDATE texts containing three conceptual parts: <reasoning>, <morph> JSON, and <final_diagnosis>.
Be format-agnostic; extract content even when tags are missing or order differs.

Allowed schemas:
- Derm7pt (object with EXACT keys): pigment_network, blue_whitish_veil, vascular_structures, pigmentation, streaks, dots_and_globules, regression_structures
- SkinCon (array of strings only): \{"morphological\_features\_skincon": [ ... ]\} from a CLOSED set.

SCHEMA SELECTION:
- Detect the schema used by REFERENCE (Derm7pt vs SkinCon). Use that schema for extraction/normalization and comparison. Do NOT switch schemas.

[Task Prompt]
\{task\_prompt\}

[REFERENCE]
\{reference\}

[CANDIDATE]
\{candidate\}

Tasks:
A) EXTRACTION: reasoning, morph (normalized to REFERENCE schema), final_dx for both sides.
B) REASONING ALIGNMENT: compute reasoning_score [0-100].
C) MORPH SEMANTICS: score_semantic in [0,1].
D) DIAGNOSIS SIMILARITY: diagnosis_score [0-100].
E) CROSS-CONSISTENCY: penalty in [0,1] if candidate reasoning contradicts candidate morph JSON.

STRICT JSON ONLY (use the specified schema in the paper).
\end{promptbox}

\subsection{Judge Reliability and Human Sanity Check}
\label{app:judge_validation}

\subsubsection{Judge sensitivity on the 900-case core set}
Table~\ref{tab:judge_sensitivity_900} reports \texttt{mean\_final\_overall} on the 900-case core set when swapping the judge between Gemini-2.5-Pro (main paper default) and GPT-5.
This comparison is intended as a robustness check for evaluator choice rather than a replacement of the main evaluation protocol.

\begin{table}[h]
  \centering
  \small
  \resizebox{\columnwidth}{!}{
  \begin{tabular}{l l c c c c}
    \toprule
    Candidate model & Judge & T1.1 & T1.2 & T3.1 & T3.2 \\
    \midrule
    Qwen3-VL-8B & Gemini-2.5-Pro & 33.18 & 46.05 & 47.53 & 53.43 \\
    Qwen3-VL-8B & GPT-5          & 37.73 & 43.92 & 51.08 & 59.81 \\
    GPT-4o-mini & Gemini-2.5-Pro & 34.55 & 51.80 & 42.83 & 51.65 \\
    GPT-4o-mini & GPT-5          & 31.32 & 47.82 & 45.28 & 49.17 \\
    \bottomrule
  \end{tabular}
  }
  \caption{Judge sensitivity on the 900-case core set (reported as \texttt{mean\_final\_overall} in $[0,100]$).}
  \label{tab:judge_sensitivity_900}
\end{table}

\subsubsection{Aggregate-level inter-judge agreement metrics}
Using the 8 paired items in Table~\ref{tab:judge_sensitivity_900} (2 candidate models $\times$ 4 tasks), we compute rank/absolute agreement metrics between GPT-5 and Gemini-2.5-Pro judge scores.
Results indicate strong agreement at the level of model-task means.

\begin{table}[h]
  \centering
  \small
  \begin{tabular}{l c}
    \toprule
    Metric & Value \\
    \midrule
    Pearson $r$ & 0.883 \\
    Spearman $\rho$ & 0.857 \\
    Mean difference (GPT-5 $-$ Gemini) & +0.65 \\
    Mean absolute difference (MAE) & 3.60 \\
    \bottomrule
  \end{tabular}
  \caption{Inter-judge agreement between GPT-5 and Gemini-2.5-Pro computed over the 8 paired model-task means in Table~\ref{tab:judge_sensitivity_900}.}
  \label{tab:judge_agreement_metrics}
\end{table}

\subsection{Human sanity check (20 cases)}
We further sample 20 cases from \texttt{Qwen3-VL-8B + Gemini-2.5-Pro} and ask clinicians to rate whether the judge scoring and feedback are reasonable on a $0$--$5$ scale (higher is more reasonable).
Figure~\ref{fig:2}c summarizes the reasonableness ratings.


\section{Training Details}
\label{app:train_details}
\subsection{Hyperparameters}

\paragraph{Backbone and precision.}
We initialize from \texttt{Qwen3-VL-8B-Instruct}, train with Deepspeed ZeRO-2,
and use BF16 with TF32 enabled. FlashAttention-2 is used unless stated
otherwise. Gradient checkpointing is enabled in both stages.

\paragraph{Stage 1: Supervised fine-tuning (SFT).}
We perform one epoch of multi-task SFT on the merged instruction data.
We enable LoRA adapters with rank $r{=}64$, $\alpha{=}64$, dropout $0.05$, and
exclude \texttt{lm\_head} and \texttt{embed\_tokens} from LoRA injection.
We freeze the language model backbone (\texttt{freeze\_llm=True}), while keeping
the vision tower and merger trainable (\texttt{freeze\_vision\_tower=False},
\texttt{freeze\_merger=False}).
We set per-device batch size to $8$ on $8$ GPUs with gradient accumulation
steps $2$ (global batch size $128$). We train with learning rate $1\mathrm{e}{-4}$,
and optionally use module-specific learning rates for the vision tower
($2\mathrm{e}{-6}$) and the merger ($1\mathrm{e}{-5}$). Weight decay is $0.1$,
warmup ratio is $0.03$, and we use a cosine scheduler. Images are resized by
pixel constraints with \texttt{image\_min\_pixels} $=256\cdot32^2$ and
\texttt{image\_max\_pixels} $=1280\cdot32^2$. Unless otherwise specified, we use
the training framework's default AdamW-type optimizer settings.

\paragraph{Stage 2: GRPO with MAVIC reward.}
We further optimize the SFT checkpoint with GRPO using group size
$K{=}\texttt{num\_generations}=8$. We train for one epoch with per-device batch
size $32$ and gradient accumulation steps $3$. We sample completions with
temperature $1.0$, top-$p$ $1.0$, and top-$k$ $50$, using maximum prompt length
$4096$ and maximum completion length $640$. We set learning rate to
$1\mathrm{e}{-6}$, weight decay to $0.1$, warmup ratio to $0.03$, and cosine
scheduler. We use \texttt{beta}=0.1 for GRPO's KL regularization. In this stage,
we freeze the vision tower, language model, and merger, and train only LoRA
adapters (LoRA rank $16$, $\alpha=32$, dropout $0.05$, excluding
\texttt{lm\_head} and \texttt{embed\_tokens}). Images are constrained by
\texttt{image\_min\_pixels} $=256\cdot28^2$ and \texttt{image\_max\_pixels}
$=1280\cdot28^2$.

\begin{table}[ht]
\centering
\small
\resizebox{\columnwidth}{!}{
\begin{tabular}{lcc}
\toprule
Hyperparameter & SFT & GRPO \\
\midrule
GPUs & 8 & 8 \\
Epochs & 1 & 1 \\
Per-device batch & 8 & 32 \\
Grad.\ accumulation & 2 & 3 \\
Global batch & 128 & 768 \\
LoRA rank / $\alpha$ & 64 / 64 & 16 / 32 \\
LoRA dropout & 0.05 & 0.05 \\
Backbone frozen? & LLM frozen & LLM/Vision/Merger frozen \\
LR & $1\mathrm{e}{-4}$ & $1\mathrm{e}{-6}$ \\
Vision LR / Merger LR & $2\mathrm{e}{-6}$ / $1\mathrm{e}{-5}$ & -- \\
Weight decay & 0.1 & 0.1 \\
Warmup / Scheduler & 0.03 / cosine & 0.03 / cosine \\
Group size $K$ & -- & 8 \\
Sampling & -- & $T{=}1.0$, top-$p{=}1.0$, top-$k{=}50$ \\
Max prompt / completion & -- & 4096 / 640 \\
KL coef.\ (\texttt{beta}) & -- & 0.1 \\
\bottomrule
\end{tabular}}
\caption{Key hyperparameters for SFT and RL training.}
\label{tab:train_hparams}
\end{table}

\subsection{MAVIC Implementation Details}
\label{app:mavic_details}

\paragraph{Morphology representation (tokens).}
Each completion must contain a structured morphology field encoded as JSON under
a \texttt{<morph>} tag. For dermoscopic images, we use Derm7pt-style attributes
\citep{derm7pt}; for clinical images, we use SkinCon-style attributes
\citep{skincon}. We binarize morphology into a vector
$\mathbf m\in\{0,1\}^{F}$, where each dimension $f$ corresponds to an attribute
indicator. For Derm7pt, we expand categorical states into attribute-state
indicators (e.g., \texttt{streaks\_irregular}); for SkinCon, each label is an
indicator.

\paragraph{PMI-based weights (precomputed lookup).}
Because each training sample has a known leaf diagnosis $y$, we precompute
diagnosis-conditioned weights $w_f(y)$ \emph{once} before RL training. We estimate
PMI with log and $\epsilon=10^{-5}$ smoothing and keep negative values:
\begin{equation}
\label{eq:pmi}
\mathrm{PMI}(m_f;y)
=
\log\frac{\hat p(m_f{=}1,y)+\epsilon}{\hat p(m_f{=}1)\hat p(y)+\epsilon}.
\end{equation}
We then normalize per diagnosis with a softmax over features:
\begin{equation}
\label{eq:pmi_softmax}
w_f(y)=\frac{\exp(\mathrm{PMI}(m_f;y))}{\sum_{f'}\exp(\mathrm{PMI}(m_{f'};y))}.
\end{equation}
During RL, $w_f(y)$ is obtained by table lookup.

\paragraph{Morphology similarity $S_{\text{morph}}$.}
Let $P$ and $G$ be the predicted and ground-truth sets of active morphology
indicators. We compute a PMI-weighted Tversky score with $\alpha=0.7,\beta=0.3$:
\begin{equation}
\begin{aligned}
\mathrm{TP} &= \sum_f w_f \mathbf{1}[\hat{m}_f = 1 \land m_f = 1],\\
\mathrm{FP} &= \sum_f w_f \mathbf{1}[\hat{m}_f = 1 \land m_f = 0],\\
\mathrm{FN} &= \sum_f w_f \mathbf{1}[\hat{m}_f = 0 \land m_f = 1].
\end{aligned}
\vspace{-3mm}
\end{equation}
\begin{equation}
S_{\mathrm{morph}}(\hat{\mathbf{m}}, \mathbf{m})
= 
\frac{\mathrm{TP}}{\mathrm{TP} + \alpha \mathrm{FP} + \beta \mathrm{FN}}.
\end{equation}

\paragraph{Hierarchy similarity $S_{\text{hier}}$.}
We map a diagnosis to its taxonomy path (ancestors) and append the leaf label to
the end of the path. We compute Wu--Palmer similarity:
\begin{equation}
\label{eq:wup}
S_{\text{hier}}=
\frac{2\cdot \mathrm{depth}(\mathrm{LCA}(\text{path}_{pred},\text{path}_{gt}))}
{|\text{path}_{pred}|+|\text{path}_{gt}|}.
\end{equation}
When parsing model outputs, we canonicalize strings and use alias/fuzzy matching
(threshold $0.8$) to map predictions to taxonomy leaves.

\paragraph{Soft gate.}
Within each GRPO sampling group (size $K$), we set $\mu$ as the median $S_{\text{hier}}$ and apply the sigmoid gate with $k=10$.

\paragraph{Format term $R_{\text{fmt}}$.}
$R_{\text{fmt}}\in\{0,1\}$ indicates whether the completion satisfies required
tag structure and JSON validity: (i) presence of required tags
(e.g., \texttt{<morph>} and, for reasoning tasks, \texttt{<final\_diagnosis>});
(ii) parseable JSON under \texttt{<morph>}; (iii) exactly one valid schema
(Derm7pt \emph{or} SkinCon); (iv) schema matches image modality; and (v) tag
ordering constraints when applicable. Invalid outputs receive $R_{\text{fmt}}=0$.

\paragraph{Hyperparameters.}
We use $\lambda_{\text{hier}}=\lambda_{\text{morph}}=1$, $\alpha=0.7$,
$\beta=0.3$, $\epsilon=10^{-5}$, fuzzy threshold $0.8$, and gate slope $k=10$.


\section{Ablation Study}
\label{app:tta-ablation}

\subsection{Impact of MAVIC Reward Components}
As shown in Table~\ref{tab:mavic_ablation}, using standard reinforcement learning rewards alone (acc+fmt) actually degrades performance on T3.2 (59.88). Incorporating morphological similarity reward $S_{\text{morph}}$ and hierarchical diagnosis reward $S_{\text{hier}}$ steadily improves scores to 65.48. Crucially, the combination of $S_{\text{morph}}$ with the logical gating mechanism $g(S_{\text{hier}})$ effectively prevents models from bypassing pathological features to make uninformed diagnostic guesses.

\subsection{Ablation of Confidence--Consistency Components}
\paragraph{Setup.}
We evaluate test-time adaptation (TTA) under the same deterministic decoding setting as the main paper (temperature $=0$). The only source of diversity is prompt paraphrasing: we use $K$ prompt variants per example (including the original prompt), and aggregate MCQA option probabilities derived from the first-step logits.

\paragraph{Baselines.}
We compare against standard, simpler ensemble decoding variants:
(i) \textbf{Single} ($K{=}1$), no TTA;
(ii) \textbf{Vote}, majority vote over predicted option letters across prompts;
(iii) \textbf{MeanProb}, unweighted averaging of option probability vectors $\mathbf{p}_r$;
(iv) \textbf{ConfOnly}, weights based on confidence margin only ($\beta{=}0$);
(v) \textbf{ConsOnly}, weights based on consistency only (drop $\tilde{C}_r$ term);
(vi) \textbf{CC (Ours)}, full confidence--consistency weighting.


\paragraph{Sensitivity to $K$ and hyperparameters.}
We further vary the number of prompt variants $K$ and the confidence exponent $\alpha$ / consistency weight $\beta$.

\begin{table}[t]
  \centering
  \small
  \begin{tabular}{lcc}
    \toprule
    $K$ & Task2.4 (OOD) $\uparrow$ & Task4 (Fair.) $\uparrow$ \\
    \midrule
    2  & 65.82 & 93.81 \\
    4  & 66.27 & 93.76 \\
    8  & \textbf{66.48} & \textbf{93.88} \\
    \bottomrule
  \end{tabular}
  \caption{Sensitivity to the number of prompt variants $K$.}
  \label{tab:tta_k_sensitivity}
\end{table}


\paragraph{Takeaway.}
Across datasets, the gains of CC aggregation cannot be explained solely by using more prompts ($K$), and persist after controlling for simpler voting/averaging baselines, supporting the claim that \emph{confidence} and \emph{consistency} provide complementary signals for robust MCQA aggregation.

\section{Theoretical Analysis}
\label{app:theory}


We provide a probabilistic model explaining why our CCT can suppress outlier rollouts and remain close to an underlying ``ideal''
token distribution.

\paragraph{Setup.}
Fix a decoding step $t$. For notational simplicity, we omit the superscript
and write $p_r \in \Delta^{V-1}$ for the token distribution of the $r$-th
rollout at this step, where $\Delta^{V-1}$ is the probability simplex in
$\mathbb{R}^V$. For any $p \in \Delta^{V-1}$ we have
\begin{equation}
  \|p\|_2 \;\le\; 1,
  \label{eq:simplex-bounded}
\end{equation}
and hence for any $p,p^*\in\Delta^{V-1}$,
\begin{equation}
  \|p - p^*\|_2^2 \;\le\; 2.
  \label{eq:simplex-distance-bound}
\end{equation}

At this time step, our method forms a weighted ensemble
\begin{equation}
q = \sum_{r=1}^K w_r p_r,
\quad
w_r =
\frac{\exp(\lambda C_r - \beta D_r)}
{\sum_{j=1}^K \exp(\lambda C_j - \beta D_j)} .
\label{eq:weighted-ensemble}
\end{equation}

where
\begin{itemize}
  \item $C_r \in [0,1]$ is a margin-based confidence score, derived from
        the top-1 vs.\ top-2 probability gap of $p_r$;
  \item $D_r = \frac{1}{2}\,\|p_r - \bar p\|_2^2$ is the squared
        $\ell_2$-distance to the empirical barycenter
        $\bar p := \frac{1}{K}\sum_{j=1}^K p_j$;
  \item $\lambda \ge 0$ controls the strength of the confidence term,
        and $\beta > 0$ controls how aggressively we downweight outliers.
\end{itemize}
Intuitively, $D_r$ penalizes rollouts that deviate from the main cluster,
while $C_r$ slightly favors locally confident rollouts among those that
are consistent.

We now formalize this intuition via a contamination model.

\subsection{Huber Contamination on the Simplex}

We assume that the rollouts at a fixed decoding step are i.i.d.\ samples
from a mixture of a ``clean'' (good) component and a contaminated (bad)
component.

\begin{assumption}[Huber contamination on the simplex]
\label{assumption:huber}
There exists an unknown target distribution $p^* \in \Delta^{V-1}$ such
that each rollout distribution $p_r$ is drawn i.i.d.\ from
\begin{equation}
\begin{aligned}
p_r &\sim (1-\varepsilon)\,\mathcal{D}_G
      + \varepsilon\,\mathcal{D}_B,
\end{aligned}
\label{eq:mixture-model}
\end{equation}
where $r = 1,\dots,K$, $0 \le \varepsilon < \tfrac{1}{2}$, $\mathcal{D}_G$ and $\mathcal{D}_B$ denote the clean and
contaminated components, respectively.

We assume the following moment and separation conditions:
\begin{align}
  \mathbb{E}_{p\sim\mathcal{D}_G} \big[\|p - p^*\|_2^2\big]
  &\;\le\; \sigma^2,
  \label{eq:good-second-moment}
  \\
  \mathbb{E}_{p\sim\mathcal{D}_B} \big[\|p - p^*\|_2^2\big]
  &\;\ge\; \sigma^2 + \Delta^2,
  \label{eq:bad-second-moment}
\end{align}
for some $\sigma^2 > 0$ and $\Delta^2 > 0$. Let
$\mu_G := \mathbb{E}_{\mathcal{D}_G}[p]$ and
$\mu_B := \mathbb{E}_{\mathcal{D}_B}[p]$ be the means of the clean and
contaminated components, respectively. We further assume a signal-to-noise
condition:
\begin{equation}
  \varepsilon\,\|\mu_B - \mu_G\|_2
  \;\le\; c_0\, \Delta
  \quad\text{for some } c_0 < \frac{1}{2}
  \label{eq:snr-condition}
\end{equation}

Finally, we assume that the clean noise level $\sigma$ is sufficiently
small relative to the separation $\Delta$ (and the contamination rate
$\varepsilon$) so that there exists a parameter
$\alpha \in (0,1)$ satisfying simultaneously:
\begin{align}
  R_G(\alpha) := \frac{\sigma}{\sqrt{\alpha}}
  &< R_B := \sqrt{\sigma^2 + \frac{\Delta^2}{2}}, 
  \label{eq:alpha-geom-cond}
  \\
  (1-\varepsilon)(1-\alpha) &> \frac{1}{2},
  \label{eq:alpha-majority-cond}
  \\
  \sigma + c_0 \Delta &\le \eta(R_B - R_G)
\end{align}
for some $\eta \in (0, \frac{1}{2})$. This mild requirement is automatically satisfied whenever the clean
cluster is sufficiently concentrated (small $\sigma$) compared to the
separation $\Delta$ and the contamination rate $\varepsilon$ is moderate.
\end{assumption}

Assumption~\ref{assumption:huber} is a Huber contamination model adapted
to the probability simplex. Conditions
\eqref{eq:good-second-moment}--\eqref{eq:bad-second-moment} ensure that
the clean component concentrates around $p^*$, while the contaminated
component is, on average, farther away. The signal-to-noise condition
\eqref{eq:snr-condition} ensures that the mixture mean is not dominated
by the contaminated component. Conditions
\eqref{eq:alpha-geom-cond}--\eqref{eq:alpha-majority-cond} guarantee that
we can choose a single parameter $\alpha$ that yields both geometric
separation and a strict majority of ``good'' rollouts.

Because $p_r\in\Delta^{V-1}$, all random variables are uniformly bounded
by~\eqref{eq:simplex-bounded}, and standard concentration inequalities
(Hoeffding, Chernoff, and their vector-valued variants) apply directly.

\subsection{High-Probability Geometric Separation}

We now show that, under Assumption~\ref{assumption:huber}, the empirical
sample $\{p_r\}_{r=1}^K$ exhibits a geometric ``good-cluster /
bad-cluster'' separation with high probability. This is precisely the
structure used in deterministic analyses of outlier suppression.

\begin{lemma}[High-probability geometric separation]
\label{lemma:geom-sep}
Suppose Assumption~\ref{assumption:huber} holds and the rollouts
$p_1,\dots,p_K$ are drawn i.i.d.\ from the mixture~\eqref{eq:mixture-model}.
Fix any $\delta\in(0,1)$ and let $\alpha\in(0,1)$ be chosen so that
\eqref{eq:alpha-geom-cond} and \eqref{eq:alpha-majority-cond} hold.
Define
\begin{equation}
  \begin{aligned}
  \varepsilon_{\mathrm{eff}} &:= R_G(\alpha) = \frac{\sigma}{\sqrt{\alpha}}, \\
  \Delta_{\mathrm{eff}} &:= R_B
  = \sqrt{\sigma^2 + \frac{\Delta^2}{2}}.
  \end{aligned}
\end{equation}
Then there exist constants
$\rho_{\mathrm{eff}} \in (\tfrac{1}{2},1)$,
$\eta \in (0,\tfrac{1}{2})$ and a sample size threshold
$K_0 = K_0(\sigma,\Delta,\varepsilon,\alpha,\delta)$ such that the
following holds.

If $K \ge K_0$, then with probability at least $1-\delta$ over the draw of
$\{p_r\}_{r=1}^K$, there exist index sets
$G_{\mathrm{eff}},B_{\mathrm{eff}} \subseteq \{1,\dots,K\}$ with
$G_{\mathrm{eff}}\cap B_{\mathrm{eff}} = \emptyset$ and
$G_{\mathrm{eff}}\cup B_{\mathrm{eff}} \neq \emptyset$ such that:
\begin{enumerate}
  \item (\emph{Effective good cluster})
        \begin{equation}
          \begin{aligned}
          \|p_g - p^*\|_2
          &\le \varepsilon_{\mathrm{eff}},
          \forall\, g \in G_{\mathrm{eff}}, \\
          |G_{\mathrm{eff}}|
          &\ge \rho_{\mathrm{eff}}\, K .
          \end{aligned}
          \label{eq:geom-good-cluster}
        \end{equation}
        where $\rho_{\mathrm{eff}} > \tfrac{1}{2}$.
  \item (\emph{Effective bad cluster is farther})
        \begin{equation}
        \begin{aligned}
        \|p_b - p^*\|_2
        &\ge \Delta_{\mathrm{eff}},
        \forall\, b \in B_{\mathrm{eff}}, \\
        \Delta_{\mathrm{eff}}
        &> \varepsilon_{\mathrm{eff}} .
        \end{aligned}
        \label{eq:geom-bad-cluster}
        \end{equation}

  \item (\emph{Barycenter remains in the attraction basin})
        Let $\bar p := \frac{1}{K}\sum_{r=1}^K p_r$ be the empirical
        barycenter. Then
        \begin{equation}
          \|\bar p - p^*\|_2
          \;\le\;
          \eta\,\big(
            \Delta_{\mathrm{eff}} - \varepsilon_{\mathrm{eff}}
          \big).
          \label{eq:barycenter-attraction}
        \end{equation}
\end{enumerate}
\end{lemma}

\begin{proof}
We proceed in three steps.

\medskip
\noindent\textbf{Step 1: Effective good cluster.}
Consider the random variable
\[
  X_G(p) := \|p - p^*\|_2^2,
\]
for $p \sim \mathcal{D}_G$. By
\eqref{eq:good-second-moment},
$\mathbb{E}_{\mathcal{D}_G}[X_G] \le \sigma^2$, and by
\eqref{eq:simplex-distance-bound}, $0 \le X_G(p) \le 2$ a.s.

By Markov's inequality, for the fixed $\alpha\in(0,1)$ (chosen in the
assumption),
\begin{equation}
  \Pr_{p\sim\mathcal{D}_G}
  \big( X_G(p) > \tfrac{\sigma^2}{\alpha} \big)
  \;\le\;
  \alpha.
\end{equation}
Equivalently,
\begin{equation}
\begin{aligned}
\Pr_{p\sim\mathcal{D}_G}
\Bigl( \|p-p^*\|_2 \le \tfrac{\sigma}{\sqrt{\alpha}} \Bigr)
&=
\Pr_{p\sim\mathcal{D}_G}
\Bigl( X_G(p) \le \tfrac{\sigma^2}{\alpha} \Bigr), \\
&\ge 1 - \alpha .
\end{aligned}
\end{equation}

Recall that we define
\[
  R_G(\alpha) := \frac{\sigma}{\sqrt{\alpha}},
  \varepsilon_{\mathrm{eff}} := R_G(\alpha).
\]

Now consider the mixture $\mathcal{D}$ in \eqref{eq:mixture-model}. The
probability that $p$ is drawn from $\mathcal{D}_G$ and
satisfies $\|p-p^*\|_2 \le R_G(\alpha)$ is at least
\begin{equation}
\begin{split}
\Pr_{p\sim\mathcal{D}}
\Bigl(
p \sim \mathcal{D}_G,\;
\|p-p^*\|_2 \le R_G(\alpha)
\Bigr) \\
\ge (1-\varepsilon)(1-\alpha).
\end{split}
\end{equation}
where we used independence between the mixture component choice and the
conditional distribution.

For each $r\in\{1,\dots,K\}$, define the indicator
\[
  I_r :=
  \mathbf{1}\{
    p_r \sim \mathcal{D}_G
    \text{ and } \|p_r-p^*\|_2 \le R_G(\alpha)
  \}.
\]
Then $(I_r)_{r=1}^K$ are i.i.d.\ Bernoulli random variables with
\begin{equation}
  \mathbb{E}[I_r]
  =
  \Pr_{p_r\sim\mathcal{D}}(I_r=1)
  \;\ge\;
  (1-\varepsilon)(1-\alpha).
\end{equation}
By Hoeffding's inequality, for any $\tau > 0$,
\begin{equation}
\begin{split}
\Pr\Big(
  \frac{1}{K}\sum_{r=1}^K I_r
  \le
  (1-\varepsilon)(1-\alpha) - \tau
\Big) \\
\le
\exp\big( -2K\tau^2 \big).
\end{split}
\label{eq:hoeffding-good}
\end{equation}
Since by Assumption~\eqref{eq:alpha-majority-cond},
$(1-\varepsilon)(1-\alpha) > \tfrac12$, we can choose
$\tau > 0$ such that
\[
  (1-\varepsilon)(1-\alpha) - \tau > \frac12.
\]
Fix such a $\tau$, and define the event
\[
  \mathcal{E}_G
  :=
  \Big\{
    \frac{1}{K}\sum_{r=1}^K I_r
    >
    (1-\varepsilon)(1-\alpha) - \tau
  \Big\}.
\]
Given a target failure probability $\delta\in(0,1)$, choose $K$ large
enough such that
\[
  \exp\big( -2K\tau^2 \big) \le \frac{\delta}{3}.
\]
Then $\Pr(\mathcal{E}_G) \ge 1 - \delta/3$, and on $\mathcal{E}_G$,
\[
  \sum_{r=1}^K I_r
  >
  \big( (1-\varepsilon)(1-\alpha) - \tau \big) K
  \;:=\;
  \rho_{\mathrm{eff}} K
\]
for some $\rho_{\mathrm{eff}} > 1/2$.

Define $G_{\mathrm{eff}}$ to be any subset of indices with $I_g = 1$ for
all $g\in G_{\mathrm{eff}}$ and
$|G_{\mathrm{eff}}| = \sum_{r=1}^K I_r$.
By construction, on $\mathcal{E}_G$ we have
\begin{align}
\|p_g - p^*\|_2
&\le R_G(\alpha)
= \varepsilon_{\mathrm{eff}},
& g &\in G_{\mathrm{eff}}, \label{eq:geom-good-cluster}\\
|G_{\mathrm{eff}}|
&\ge \rho_{\mathrm{eff}}\, K .
\nonumber
\end{align}
so \eqref{eq:geom-good-cluster} holds.

\medskip
\noindent\textbf{Step 2: Effective bad cluster.}
Consider
\[
  X_B(p) := \|p - p^*\|_2^2
\]
for $p\sim\mathcal{D}_B$. By \eqref{eq:bad-second-moment},
\begin{equation}
  \mathbb{E}_{\mathcal{D}_B}[X_B]
  \;\ge\;
  \sigma^2 + \Delta^2,
\end{equation}
and by \eqref{eq:simplex-distance-bound}, we have
$0 \le X_B(p) \le 2$ almost surely.

Fix the threshold
\begin{equation}
  a := \sigma^2 + \frac{\Delta^2}{2}.
\end{equation}
From \eqref{eq:simplex-distance-bound} and
$\mathbb{E}_{\mathcal{D}_B}[X_B] \le 2$, it follows that
$\sigma^2 + \Delta^2 \le 2$, hence $a \le \sigma^2 + \Delta^2 \le 2$ and
in particular $a \le 2$.
Decompose
\begin{align}
\mathbb{E}_{\mathcal{D}_B}[X_B]
&= \mathbb{E}_{\mathcal{D}_B}\!\left[ X_B \mathbf{1}\{X_B < a\} \right] \nonumber\\
&\quad + \mathbb{E}_{\mathcal{D}_B}\!\left[ X_B \mathbf{1}\{X_B \ge a\} \right] \nonumber \\
&\le a \cdot \Pr(X_B < a)
   + 2 \cdot \Pr(X_B \ge a) \nonumber\\
&= a + (2-a)\,\Pr(X_B \ge a)
\end{align}
since $X_B\le 2$ almost surely.
Combining this with $\mathbb{E}_{\mathcal{D}_B}[X_B] \ge \sigma^2 + \Delta^2$ yields
\begin{align}
  \sigma^2 + \Delta^2 &\le a + (2-a)\,\Pr(X_B \ge a) \nonumber \\
  &= \sigma^2 + \frac{\Delta^2}{2} + (2-a)\,\Pr(X_B \ge a),
\end{align}
and hence
\begin{equation}
  \Pr(X_B \ge a)
  \;\ge\;
  \frac{\frac{\Delta^2}{2}}{2-a}
  \;\ge\;
  \frac{\Delta^2}{4}.
\end{equation}

Equivalently,
\begin{equation}
  \Pr_{p\sim\mathcal{D}_B}
  \Big(
    \|p-p^*\|_2 \ge \sqrt{a}
  \Big)
  \;\ge\;
  \frac{\Delta^2}{4}.
\end{equation}
Define
\begin{equation}
  R_B := \sqrt{a}
  = \sqrt{\sigma^2 + \frac{\Delta^2}{2}},
  \qquad
  \Delta_{\mathrm{eff}} := R_B.
\end{equation}
By Assumption~\eqref{eq:alpha-geom-cond}, we have
$\Delta_{\mathrm{eff}} = R_B > R_G(\alpha) = \varepsilon_{\mathrm{eff}}$.

Now consider the mixture $\mathcal{D}$. The probability that
$p\sim\mathcal{D}$ is drawn from $\mathcal{D}_B$ and satisfies
$\|p-p^*\|_2 \ge R_B$ is at least
\begin{equation}
  \Pr_{p\sim\mathcal{D}}
  \big(
    p\text{ from }\mathcal{D}_B,\ \|p-p^*\|_2 \ge R_B
  \big)
  \;\ge\;
  \varepsilon \cdot \frac{\Delta^2}{4}.
\end{equation}
For each $r$, define the indicator
\[
  J_r :=
  \mathbf{1}\{
    p_r \text{ is drawn from }\mathcal{D}_B
    \text{ and } \|p_r-p^*\|_2 \ge R_B
  \}.
\]
Then $(J_r)_{r=1}^K$ are i.i.d.\ Bernoulli random variables with
\begin{equation}
  \mathbb{E}[J_r]
  =
  \Pr_{p_r\sim\mathcal{D}}(J_r=1)
  \;\ge\;
  \varepsilon \cdot \frac{\Delta^2}{4}.
\end{equation}
Applying Hoeffding's inequality again, for any $\tau'>0$,
\begin{equation}
  \Pr\Big(
    \frac{1}{K}\sum_{r=1}^K J_r
    \le
    \varepsilon \frac{\Delta^2}{4} - \tau'
  \Big)
  \;\le\;
  \exp\big( -2K{\tau'}^2 \big).
\end{equation}
Given $\delta$, we may choose $\tau' > 0$ and $K$ large enough so that
$\varepsilon \frac{\Delta^2}{4} - \tau' > 0$ and
$\exp(-2K{\tau'}^2) \le \delta/3$.

Define the event
\[
  \mathcal{E}_B
  :=
  \Big\{
    \frac{1}{K}\sum_{r=1}^K J_r
    >
    \varepsilon \frac{\Delta^2}{4} - \tau'
  \Big\}.
\]
Then $\Pr(\mathcal{E}_B) \ge 1-\delta/3$, and on $\mathcal{E}_B$ there
are at least
\[
  \Big( \varepsilon \frac{\Delta^2}{4} - \tau' \Big)K
\]
indices $r$ such that $J_r=1$.
Define $B_{\mathrm{eff}}$ to be any subset of indices with $J_b=1$ for
all $b\in B_{\mathrm{eff}}$ and
$|B_{\mathrm{eff}}| = \sum_{r=1}^K J_r$.
By construction, for all $b\in B_{\mathrm{eff}}$ we have
$\|p_b - p^*\|_2 \ge R_B = \Delta_{\mathrm{eff}}$, so
\eqref{eq:geom-bad-cluster} holds on $\mathcal{E}_B$.

\medskip
\noindent\textbf{Step 3: Control of the barycenter.}
Let $\mu := \mathbb{E}[p_r]$ be the mean of the mixture $\mathcal{D}$.
From \eqref{eq:mixture-model} we have
\begin{equation}
  \mu = (1-\varepsilon)\mu_G + \varepsilon \mu_B.
\end{equation}
Using Jensen's inequality and \eqref{eq:good-second-moment},
\begin{equation}
  \|\mu_G - p^*\|_2^2
  \;\le\;
  \mathbb{E}_{\mathcal{D}_G} \big[\|p-p^*\|_2^2\big]
  \;\le\;
  \sigma^2,
\end{equation}
so $\|\mu_G - p^*\|_2 \le \sigma$.
Hence
\begin{align}
  \|\mu - p^*\|_2
  &= \big\|
      (1-\varepsilon)(\mu_G - p^*)
      + \varepsilon(\mu_B - p^*)
     \big\|_2
  \nonumber\\
  &\le
  (1-\varepsilon)\|\mu_G - p^*\|_2
  + \varepsilon \|\mu_B - p^*\|_2
  \nonumber\\
  &\le
  \|\mu_G - p^*\|_2
  + \varepsilon \|\mu_B - \mu_G\|_2
  \nonumber\\
  &\le
  \sigma + \varepsilon \|\mu_B - \mu_G\|_2
  \nonumber\\
  &\le
  \sigma + c_0 \Delta,
\end{align}
where we used \eqref{eq:snr-condition} in the last inequality.

Now consider the empirical barycenter
$\bar p = \frac{1}{K}\sum_{r=1}^K p_r$.
Since each $p_r \in \Delta^{V-1}$ with $\|p_r\|_2 \le 1$, the
vector-valued Hoeffding inequality implies that, for any $t>0$,
\begin{equation}
  \Pr\big(
    \|\bar p - \mu\|_2 \ge t
  \big)
  \;\le\;
  2 \exp\big( - c K t^2 \big),
\end{equation}
for some universal constant $c>0$.
Given $\delta$, choose $t>0$ and $K$ large enough such that
$2\exp(-cKt^2) \le \delta/3$.
Define
\[
  \mathcal{E}_M
  :=
  \big\{ \|\bar p - \mu\|_2 \le t \big\}.
\]
Then $\Pr(\mathcal{E}_M) \ge 1-\delta/3$, and on $\mathcal{E}_M$,
\begin{equation}
  \|\bar p - p^*\|_2
  \;\le\;
  \|\bar p - \mu\|_2 + \|\mu - p^*\|_2
  \;\le\;
  t + \sigma + c_0 \Delta.
\end{equation}

We now ensure that this is bounded by a fraction of the gap
$\Delta_{\mathrm{eff}} - \varepsilon_{\mathrm{eff}}
 = R_B - R_G(\alpha) > 0$.
By Assumption~\eqref{eq:alpha-geom-cond}, $R_G(\alpha) < R_B$, so
$\Delta_{\mathrm{eff}} - \varepsilon_{\mathrm{eff}} > 0$.
Fix any $\eta \in (0,\tfrac{1}{2})$.
By increasing $K$, we can make $t$ arbitrarily small, and therefore we can
choose $K$ so large that
\begin{equation}
  t + \sigma + c_0\Delta
  \;\le\;
  \eta \big( R_B - R_G(\alpha) \big)
  =
  \eta \big( \Delta_{\mathrm{eff}} - \varepsilon_{\mathrm{eff}} \big).
\end{equation}
On $\mathcal{E}_M$ we then have
\[
  \|\bar p - p^*\|_2
  \;\le\;
  \eta\big(\Delta_{\mathrm{eff}} - \varepsilon_{\mathrm{eff}}\big),
\]
which is \eqref{eq:barycenter-attraction}.


\medskip
\noindent\textbf{Step 4: Union bound.}
Define
\[
  \mathcal{E}
  :=
  \mathcal{E}_G \cap \mathcal{E}_B \cap \mathcal{E}_M.
\]
By construction and our choices of $K$, we have
\[
  \Pr(\mathcal{E})
  \;\ge\;
  1 - \big( \tfrac{\delta}{3} + \tfrac{\delta}{3} + \tfrac{\delta}{3} \big)
  = 1 - \delta,
\]
and on $\mathcal{E}$ all three properties hold.
This proves the lemma.
\end{proof}

Lemma~\ref{lemma:geom-sep} states that, for sufficiently many rollouts,
with high probability the empirical set behaves as if there were a
deterministic ``good cluster'' and ``bad cluster'' around $p^*$, with the
barycenter $\bar p$ staying within the attraction region of the good
cluster. We next exploit this for robust aggregation.

\subsection{Robust Aggregation via Squared \texorpdfstring{$\ell_2$}{l2}}

We now show that, on the high-probability event of
Lemma~\ref{lemma:geom-sep}, exponential weighting based on the squared
$\ell_2$ distance $D_r$ suppresses contaminated rollouts exponentially.

For the moment, we ignore the confidence term ($\lambda=0$) and consider
pure distance-based weights
\begin{equation}
  w_r \propto \exp(-\beta D_r),
  D_r = \tfrac{1}{2}\,\|p_r - \bar p\|_2^2,
  \label{eq:dist-only-weights-main}
\end{equation}

\begin{theorem}[Robust aggregation under geometric separation]
\label{thm:robust-agg}
Suppose the high-probability event of
Lemma~\ref{lemma:geom-sep} holds, with parameters
$\varepsilon_{\mathrm{eff}},
 \Delta_{\mathrm{eff}},
 \rho_{\mathrm{eff}},
 \eta$
satisfying $\Delta_{\mathrm{eff}} > \varepsilon_{\mathrm{eff}}$ and
$\eta < \tfrac{1}{2}$.
Then there exists a constant $\gamma_{\mathrm{eff}} > 0$, depending only
on these parameters, such that:
\begin{enumerate}
  \item For all $g\in G_{\mathrm{eff}}$ and $b\in B_{\mathrm{eff}}$,
        \begin{equation}
          D_b \;\ge\; D_g + \gamma_{\mathrm{eff}}.
          \label{eq:gap-D}
        \end{equation}
  \item For any $\beta > 0$, the aggregate distribution
        $q = \sum_{r=1}^K w_r p_r$ with
        $w_r \propto \exp(-\beta D_r)$ satisfies
        \begin{align}
          \|q-p^*\|_2
        &\le
        \varepsilon_{\mathrm{eff}}
        + C_U
        + \nonumber \\
        &(\Delta_{\max}-\varepsilon_{\mathrm{eff}})
        \frac{1-\rho_{\mathrm{eff}}}{\rho_{\mathrm{eff}}}
        e^{-\beta\gamma_{\mathrm{eff}}}
        \end{align}
\end{enumerate}
where $C_U$ is a constant. In particular, if $G_{\mathrm{eff}} \cup B_{\mathrm{eff}} = [K]$, the aggregated distribution $q$
converges in $\ell_2$ to the effective good cluster up to radius
$\varepsilon_{\mathrm{eff}}$, and the influence of contaminated rollouts
is exponentially suppressed.
\end{theorem}

\begin{proof}
\textbf{Step 1: Gap in $D_r$.}
By Lemma~\ref{lemma:geom-sep}, for all $g\in G_{\mathrm{eff}}$ we have
$\|p_g - p^*\|_2 \le \varepsilon_{\mathrm{eff}}$ and for all
$b\in B_{\mathrm{eff}}$ we have
$\|p_b - p^*\|_2 \ge \Delta_{\mathrm{eff}}$, and the barycenter satisfies
$\|\bar p - p^*\|_2 \le \eta(\Delta_{\mathrm{eff}} - \varepsilon_{\mathrm{eff}})$.

For any $g\in G_{\mathrm{eff}}$,
\begin{align}
  \|p_g - \bar p\|_2
  &\le
  \|p_g - p^*\|_2 + \|p^* - \bar p\|_2
  \nonumber\\
  &\le
  \varepsilon_{\mathrm{eff}}
  + \eta\big(\Delta_{\mathrm{eff}} - \varepsilon_{\mathrm{eff}}\big),
\end{align}
so
\begin{align}
  D_g
  = \frac{1}{2}\|p_g - \bar p\|_2^2
  \;&\le\;
  \frac{1}{2}
  \big(
    \varepsilon_{\mathrm{eff}}
    + \eta(\Delta_{\mathrm{eff}} - \varepsilon_{\mathrm{eff}})
  \big)^2 \nonumber \\
  &=: D_g^{\max}.
\end{align}
Similarly, for any $b\in B_{\mathrm{eff}}$,
\begin{align}
  \|p_b - \bar p\|_2
  &\ge
  \big|
    \|p_b - p^*\|_2 - \|p^* - \bar p\|_2
  \big|
  \nonumber\\
  &\ge
  \Delta_{\mathrm{eff}}
  - \eta\big(\Delta_{\mathrm{eff}} - \varepsilon_{\mathrm{eff}}\big),
\end{align}
and thus
\begin{align}
  D_b
  = \frac{1}{2}\|p_b - \bar p\|_2^2
  \;&\ge\;
  \frac{1}{2}
  \big(
    \Delta_{\mathrm{eff}}
    - \eta(\Delta_{\mathrm{eff}} - \varepsilon_{\mathrm{eff}})
  \big)^2 \nonumber \\
  &=: D_b^{\min}.
\end{align}

Define
\[
  f(\eta)
  :=
  D_b^{\min} - D_g^{\max}
\]
At $\eta = 0$ we have
\[
  f(0)
  =
  \frac{1}{2}
  \big(
    \Delta_{\mathrm{eff}}^2 - \varepsilon_{\mathrm{eff}}^2
  \big)
  > 0
\]
since $\Delta_{\mathrm{eff}} > \varepsilon_{\mathrm{eff}}$.
The map $\eta \mapsto f(\eta)$ is continuous on $[0,\tfrac{1}{2})$, so
there exists $\eta_0\in(0,\tfrac{1}{2})$ such that $f(\eta)>0$ for all
$\eta\in[0,\eta_0]$.
Lemma~\ref{lemma:geom-sep} guarantees that $\eta$ can be chosen in
$(0,\tfrac{1}{2})$; by further shrinking $\eta$ if necessary we may assume
$\eta\le\eta_0$.
Define
\begin{equation}
  \gamma_{\mathrm{eff}} := f(\eta) > 0.
\end{equation}
It follows that, for all $g\in G_{\mathrm{eff}}$ and $b\in B_{\mathrm{eff}}$,
\[
  D_b \;\ge\; D_b^{\min}
  = D_g^{\max} + \gamma_{\mathrm{eff}}
  \;\ge\; D_g + \gamma_{\mathrm{eff}},
\]
which proves \eqref{eq:gap-D}.

\medskip
\noindent\textbf{Step 2: Exponential suppression and error bound.}
Define the remaining index set
\[
  U_{\mathrm{eff}} := [K]\setminus\big(G_{\mathrm{eff}}\cup B_{\mathrm{eff}}\big),
\]
and the corresponding total weights
\begin{align}
  W_B := \sum_{b\in B_{\mathrm{eff}}} w_b,
  W_G := \sum_{g\in G_{\mathrm{eff}}} w_g, \nonumber \\
  W_U := \sum_{u\in U_{\mathrm{eff}}} w_u,
\end{align}
so that $W_B+W_G+W_U=1$.

Let
\begin{align}
  A := \sum_{g\in G_{\mathrm{eff}}} e^{-\beta D_g},
  B := \sum_{b\in B_{\mathrm{eff}}} e^{-\beta D_b}, \nonumber \\
  C := \sum_{u\in U_{\mathrm{eff}}} e^{-\beta D_u},
  Z := A+B+C.
\end{align}

Then for every $r\in[K]$,
\[
  w_r = \frac{e^{-\beta D_r}}{Z},
  \quad\text{and}\quad
  W_B = \frac{B}{Z}.
\]

Using \eqref{eq:gap-D}, for any $b\in B_{\mathrm{eff}}$ and any $g\in G_{\mathrm{eff}}$,
\[
  e^{-\beta D_b}
  \le
  e^{-\beta(D_g + \gamma_{\mathrm{eff}})}
  =
  e^{-\beta\gamma_{\mathrm{eff}}}\, e^{-\beta D_g}.
\]

Taking $\min$ over $g$ and summing over $b$ gives
\begin{align}
  B
  &\le
  |B_{\mathrm{eff}}|\,e^{-\beta\gamma_{\mathrm{eff}}}
  \min_{g\in G_{\mathrm{eff}}} e^{-\beta D_g} \\
  &\le
  |B_{\mathrm{eff}}|\,e^{-\beta\gamma_{\mathrm{eff}}}\,
  \frac{1}{|G_{\mathrm{eff}}|}
  \sum_{g\in G_{\mathrm{eff}}} e^{-\beta D_g} \\
  &=
  \frac{|B_{\mathrm{eff}}|}{|G_{\mathrm{eff}}|}
  e^{-\beta\gamma_{\mathrm{eff}}}\,A,
\end{align}
and thus
\[
  R := \frac{B}{A}
  \le
  \frac{|B_{\mathrm{eff}}|}{|G_{\mathrm{eff}}|}
  e^{-\beta\gamma_{\mathrm{eff}}}.
\]
Since $|G_{\mathrm{eff}}|\ge \rho_{\mathrm{eff}}K$ and
$|B_{\mathrm{eff}}|\le K-|G_{\mathrm{eff}}|\le (1-\rho_{\mathrm{eff}})K$, we obtain
\[
  R
  \le
  \frac{1-\rho_{\mathrm{eff}}}{\rho_{\mathrm{eff}}}
  e^{-\beta\gamma_{\mathrm{eff}}}.
\]
Moreover, because $Z\ge A+B$,
\[
  W_B=\frac{B}{Z}
  \le
  \frac{B}{A+B}
  =
  \frac{R}{1+R}
  \le R,
\]
so
\[
  W_B
  \le
  \frac{1-\rho_{\mathrm{eff}}}{\rho_{\mathrm{eff}}}
  e^{-\beta\gamma_{\mathrm{eff}}}.
\]

Finally,
\begin{align}
  \|q - p^*\|_2
  &=
  \Big\|
    \sum_{r=1}^K w_r (p_r - p^*)
  \Big\|_2 \nonumber\\
  &\le
  \sum_{r=1}^K w_r \|p_r - p^*\|_2
  \nonumber\\
  &\le
  \varepsilon_{\mathrm{eff}} \sum_{g\in G_{\mathrm{eff}}} w_g
  + \Delta_{\max} \sum_{r\notin G_{\mathrm{eff}}} w_r
  \nonumber\\
  &=
  \varepsilon_{\mathrm{eff}} W_G
  + \Delta_{\max}(W_B+W_U),
  \label{eq:robust-bound-with-U}
\end{align}
where $\Delta_{\max}:=\max_{1\le r\le K}\|p_r-p^*\|_2\le \sqrt{2}$ for distributions on the simplex.

Using $W_G=1-W_B-W_U$, \eqref{eq:robust-bound-with-U} implies
\begin{align}
  \|q-p^*\|_2
  &\le
  \varepsilon_{\mathrm{eff}}(1-W_B-W_U) + \nonumber\\
  &\quad\Delta_{\max}(W_B+W_U)
  \nonumber\\
  &=
  \varepsilon_{\mathrm{eff}}
  +(\Delta_{\max}-\varepsilon_{\mathrm{eff}})(W_B+W_U)
  \nonumber\\
  &\le
  \varepsilon_{\mathrm{eff}}
  +(\Delta_{\max}-\varepsilon_{\mathrm{eff}})W_U
  + \nonumber\\
  &(\Delta_{\max}-\varepsilon_{\mathrm{eff}})
  \frac{1-\rho_{\mathrm{eff}}}{\rho_{\mathrm{eff}}}
  e^{-\beta\gamma_{\mathrm{eff}}}.
  \label{eq:robust-bound-main-with-U}
\end{align}
Defining the residual term
\[
  C_U := (\Delta_{\max}-\varepsilon_{\mathrm{eff}})W_U
  \;\;\;(\le \Delta_{\max}-\varepsilon_{\mathrm{eff}}),
\]
we can rewrite \eqref{eq:robust-bound-main-with-U} in the same final form as
\[
  \|q-p^*\|_2
  \le
  \varepsilon_{\mathrm{eff}}
  + C_U
  +(\Delta_{\max}-\varepsilon_{\mathrm{eff}})
  \frac{1-\rho_{\mathrm{eff}}}{\rho_{\mathrm{eff}}}
  e^{-\beta\gamma_{\mathrm{eff}}}.
\]
\end{proof}

\subsection{Effect of the Margin Term as a Bounded Perturbation}

We now return to the full weighting scheme, which includes a margin-based
confidence term $C_r\in[0,1]$:
\begin{equation}
  s_r \;=\; \lambda C_r - \beta D_r,
  \quad
  w_r \;\propto\; \exp(s_r).
\end{equation}
Since $C_r\in[0,1]$, the margin term perturbs each log-weight by at most
$\lambda$:
\begin{align}
  -\beta D_r
  \;&\le\;
  s_r
  \;\le\;
  -\beta D_r + \lambda
  \Rightarrow \nonumber \\
  e^{-\beta D_r}
  \;&\le\;
  e^{s_r}
  \;\le\;
  e^{\lambda} e^{-\beta D_r}.
  \label{eq:bounded-margin-perturbation}
\end{align}

\begin{corollary}[Robustness with margin-based confidence]
\label{cor:margin-robust}
Under the high-probability event of Lemma~\ref{lemma:geom-sep}, consider
the full weighting scheme
\begin{align}
  w_r \propto \exp\big(\lambda C_r - \beta D_r\big),
  \qquad
  C_r \in [0,1], \nonumber\\
  D_r = \tfrac{1}{2}\|p_r - \bar p\|_2^2. \qquad \quad
\end{align}
Let $U_{\mathrm{eff}} := [K]\setminus(G_{\mathrm{eff}}\cup B_{\mathrm{eff}})$ and
\[
  W_U := \sum_{u\in U_{\mathrm{eff}}} w_u.
\]
Let $\Delta_{\max}:=\max_{1\le r\le K}\|p_r-p^*\|_2$ (for distributions on the simplex, $\Delta_{\max}\le\sqrt{2}$).
Then, for any $\beta>0$,
\begin{align}
  \|q - &p^*\|_2
  \;\le\;
  \varepsilon_{\mathrm{eff}}
  +
  (\Delta_{\max}-\varepsilon_{\mathrm{eff}}) W_U
  + \nonumber\\
  &(\Delta_{\max}-\varepsilon_{\mathrm{eff}})
  \frac{1-\rho_{\mathrm{eff}}}{\rho_{\mathrm{eff}}}
  \exp\!\big(-\beta\gamma_{\mathrm{eff}} + \lambda\big).
  \label{eq:margin-robust-bound-with-U}
\end{align}
In particular, as long as $\beta\gamma_{\mathrm{eff}} > \lambda$, the influence of
$B_{\mathrm{eff}}$ is exponentially suppressed (up to constant factors).
\end{corollary}

\begin{proof}
Let $w_r$ be the full weights with $s_r=\lambda C_r-\beta D_r$.
Define the (unnormalized) sums
\begin{align}
  A_s := \sum_{g\in G_{\mathrm{eff}}} e^{s_g},
  \quad
  B_s := \sum_{b\in B_{\mathrm{eff}}} e^{s_b},
  \quad \nonumber\\
  C_s := \sum_{u\in U_{\mathrm{eff}}} e^{s_u},
  \quad
  Z_s := A_s + B_s + C_s.
\end{align}
Then $w_r=e^{s_r}/Z_s$ and $W_B:=\sum_{b\in B_{\mathrm{eff}}}w_b = B_s/Z_s$.
For any $b\in B_{\mathrm{eff}}$ and $g\in G_{\mathrm{eff}}$, using $C_b\le 1$, $C_g\ge 0$ and \eqref{eq:gap-D},
\[
  s_b - s_g
  =
  \lambda(C_b-C_g) - \beta(D_b-D_g)
  \le
  \lambda - \beta\gamma_{\mathrm{eff}},
\]
hence
\[
  e^{s_b} \le \exp\!\big(-\beta\gamma_{\mathrm{eff}}+\lambda\big)\,e^{s_g}.
\]
Taking $\min$ over $g$ and summing over $b$ yields
\begin{align}
  B_s
  &\le
  |B_{\mathrm{eff}}|\,\exp\!\big(-\beta\gamma_{\mathrm{eff}}+\lambda\big)\,
  \min_{g\in G_{\mathrm{eff}}} e^{s_g} \nonumber\\
  &\le
  \frac{|B_{\mathrm{eff}}|}{|G_{\mathrm{eff}}|}
  \exp\!\big(-\beta\gamma_{\mathrm{eff}}+\lambda\big)\,
  \sum_{g\in G_{\mathrm{eff}}} e^{s_g} \nonumber\\
  &=
  \frac{|B_{\mathrm{eff}}|}{|G_{\mathrm{eff}}|}
  \exp\!\big(-\beta\gamma_{\mathrm{eff}}+\lambda\big)\,A_s.
\end{align}
Therefore, with $R_s:=B_s/A_s$,
\begin{align}
  R_s &\le \frac{|B_{\mathrm{eff}}|}{|G_{\mathrm{eff}}|}\exp\!\big(-\beta\gamma_{\mathrm{eff}}+\lambda\big)\nonumber\\
  &\le \frac{1-\rho_{\mathrm{eff}}}{\rho_{\mathrm{eff}}}\exp\!\big(-\beta\gamma_{\mathrm{eff}}+\lambda\big),
\end{align}
where we used $|G_{\mathrm{eff}}|\ge \rho_{\mathrm{eff}}K$ and
$|B_{\mathrm{eff}}|\le K-|G_{\mathrm{eff}}|\le (1-\rho_{\mathrm{eff}})K$.
Moreover, since $Z_s\ge A_s+B_s$,
\begin{align}
  W_B&=\frac{B_s}{Z_s}\le \frac{B_s}{A_s+B_s}
  = \frac{R_s}{1+R_s}\le R_s \nonumber\\
  &\le \frac{1-\rho_{\mathrm{eff}}}{\rho_{\mathrm{eff}}}\exp\!\big(-\beta\gamma_{\mathrm{eff}}+\lambda\big).
\end{align}

Finally, define
\[
  W_G:=\sum_{g\in G_{\mathrm{eff}}}w_g,\qquad
  W_U:=\sum_{u\in U_{\mathrm{eff}}}w_u,
\]
so $W_G+W_B+W_U=1$. By the same triangle-inequality argument as in the
robust-aggregation proof,
\begin{align}
  \|q-p^*\|_2
  &\le
  \varepsilon_{\mathrm{eff}} W_G + \Delta_{\max}(W_B+W_U)\nonumber\\
  &=
  \varepsilon_{\mathrm{eff}} + (\Delta_{\max}-\varepsilon_{\mathrm{eff}})(W_B+W_U).
\end{align}
Plugging in the bound on $W_B$ gives \eqref{eq:margin-robust-bound-with-U}.
\end{proof}

\section{Human Annotation and Ethical Considerations}
\label{app:human_annot_ethics}

This appendix reports the human-in-the-loop procedures used in our study. All human involvement in this work concerns expert \emph{evaluation} and \emph{revision} of model-generated drafts, and does not involve any new patient data collection.

\subsection{Instructions Given to Participants}
\label{app:human_instructions}

\subsubsection{Quality Assessment of Model-Generated Drafts}
\label{app:task1_quality_assess}

We ask dermatology experts to review a \textbf{900-case core set} and rate the quality of Gemini-generated initial drafts.

\noindent\textbf{Instruction.}
Please review the provided dermatology image and the corresponding AI-generated report.
Using a 0--5 Likert scale, rate the following two dimensions:

\begin{itemize}
  \item \textbf{Morphological Fidelity:} Are the described clinical features (e.g., color, border, lesion type) fully consistent with the visual evidence in the image?
  \item \textbf{Reasoning Validity:} Is the chain-of-thought reasoning logically sound and properly grounded in visual evidence from the image?
\end{itemize}

\noindent\textbf{Score definition.}
5 indicates fully accurate and logically rigorous; 0 indicates severe errors such as major misdiagnosis or hallucinated features.

\subsubsection{Gold Standard Manual Revision for the Core Set}
\label{app:task2_gold_revision}

Experts revise model-generated drafts using a dedicated web interface.
\begin{figure}[h]
    \centering
    \includegraphics[width=0.75\linewidth]{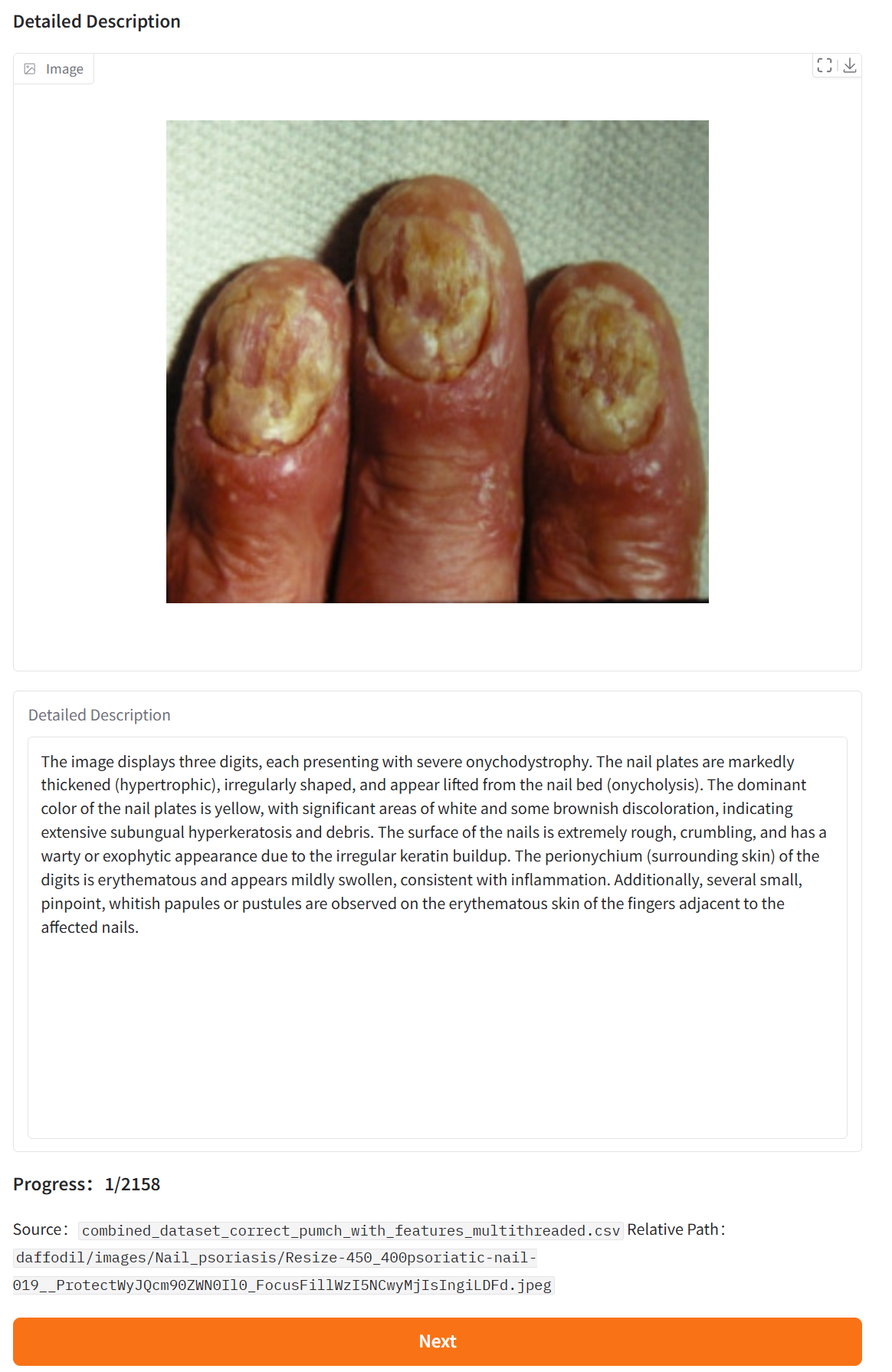}
    \caption{An example of web interface used to get .}
    \label{fig:placeholder}
\end{figure}

\noindent\textbf{Instruction.}
The text box contains an AI-generated draft. Please perform the following:

\begin{enumerate}
  \item \textbf{Line-by-line revision:} Compare against the original image and manually correct terminology errors, missing key features, or reasoning gaps.
  \item \textbf{Bottleneck verification:} Ensure the revised \texttt{<morph>} JSON strictly follows the Derm7pt/SkinCon schema.
  \item \textbf{Final approval:} The revised content should represent the clinical \emph{gold-standard} answer for this case.
\end{enumerate}

\subsubsection{Human Sanity Check for LLM-as-a-Judge}
\label{app:task3_judge_sanity}

For 20 randomly sampled cases, experts evaluate whether the Judge (Gemini-2.5-Pro) provides reasonable scores and feedback.

\noindent\textbf{Instruction.}
Please review the \textbf{model output}, \textbf{reference answer}, and the \textbf{AI Judge}'s score and feedback.

\begin{itemize}
  \item \textbf{Task:} Rate (0--5) whether the AI Judge's evaluation is reasonable.
  \item \textbf{Reasonableness criteria:} The score should be objective, and the feedback should point out key medical differences.
  \item \textbf{Acceptance threshold:} Scores $\ge$ 3 are considered acceptable.
\end{itemize}

\subsubsection{Human Performance Baseline}
\label{app:task4_human_baseline}

To obtain the ``Human Performance'' results, we randomly sample \textbf{100 cases per task} and ask experts to complete the benchmark \textbf{without any AI assistance}.

\noindent\textbf{Instruction.}
Please independently complete DermoBench evaluation tasks as in clinical practice, \textbf{without referencing any AI hints}:

\begin{enumerate}
  \item \textbf{MCQA tasks:} Select the most likely diagnosis from 4-choice or 25-choice options.
  \item \textbf{Hierarchical diagnosis:} Perform step-wise selection along the diagnosis tree path (Superclass $\rightarrow$ Subclass).
  \item \textbf{Open-ended description:} Write a detailed morphological examination report without viewing any reference answer.
\end{enumerate}

\subsection{Recruitment, Compensation, and Consent}
\label{app:recruit_payment_consent}

\textbf{Recruitment and qualifications.}
We invited and engaged two dermatology clinicians via targeted online outreach. Both participants have relevant clinical experience in dermatology.

\textbf{Compensation.}
Participants were compensated at approximately 100 RMB per hour, following local norms for medical professional consulting, which we consider adequate to reflect the value of expert labor.

\textbf{Annotator consent.}
All participating clinicians signed an agreement acknowledging that their revision, annotation, and rating outputs would be used for open research purposes in developing and evaluating our dermatology MLLMs and benchmark.

\subsection{Data Consent, Release Policy, and Ethics Review}
\label{app:data_ethics}

\textbf{Open datasets and intended use.}
This work uses only publicly released, de-identified dermatology datasets. We follow the licenses and intended research use specified by the original dataset providers.
Experts may view the original images during evaluation and revision; however, \textbf{we do not redistribute or release the original images}.
We release only derived artifacts (e.g., prompts, annotations, benchmark splits, and evaluation outputs), and users should obtain images from the original sources.

\textbf{Ethics review.}
We do not collect any new patient data and only use de-identified, publicly available datasets; the expert annotation activities are minimal-risk. Therefore, ethics board approval was not required under our institutional policy.

\section{Visualization}
The following pictures provide additional visualizations and qualitative case studies to better understand the data characteristics of DermoBench.

\begin{figure*}[ht]
  \centering
  
  \includegraphics[width=\linewidth]{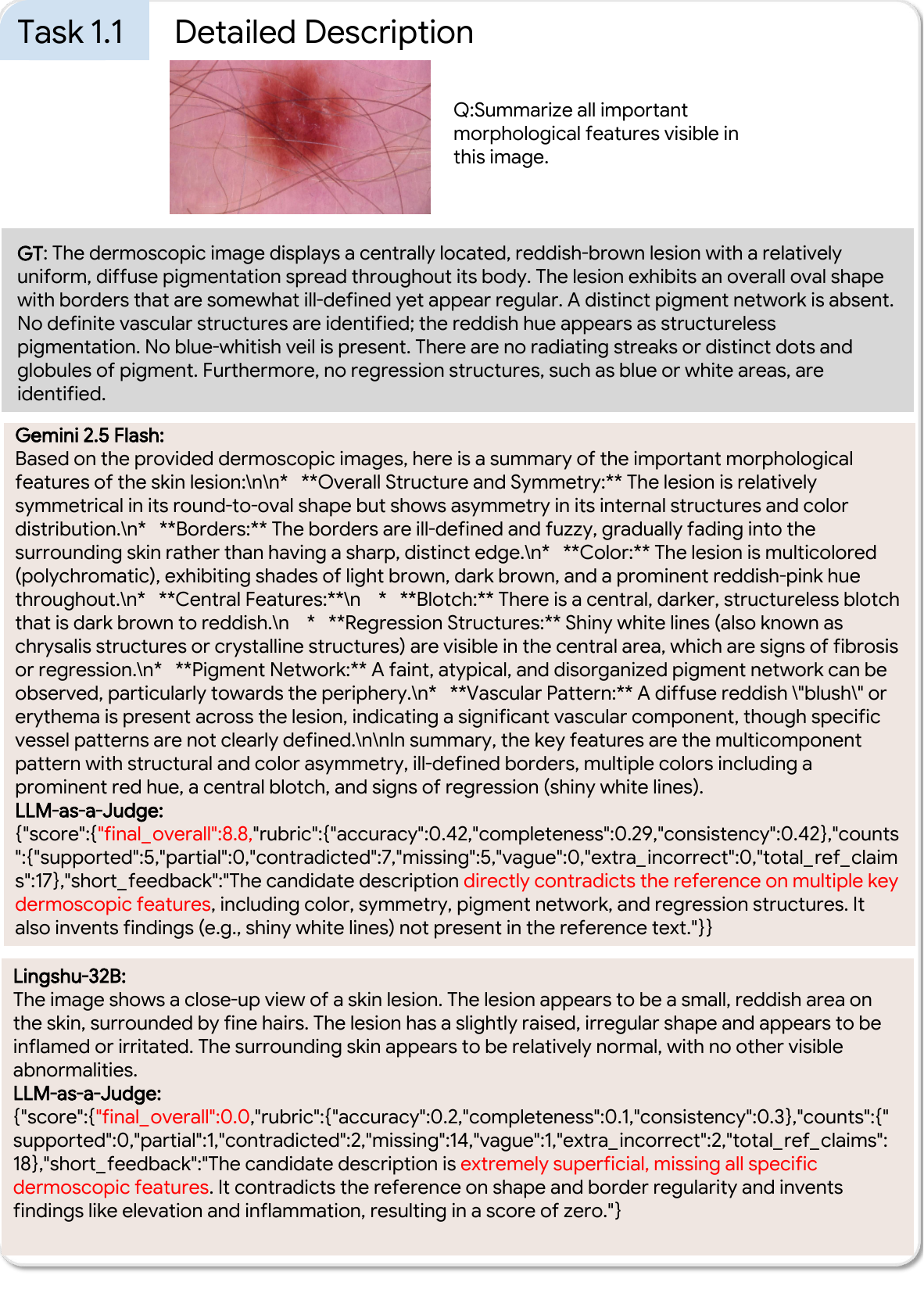}
  \caption{Case study.}
\end{figure*}
\begin{figure*}[ht]
  \centering
  
  \includegraphics[width=\linewidth]{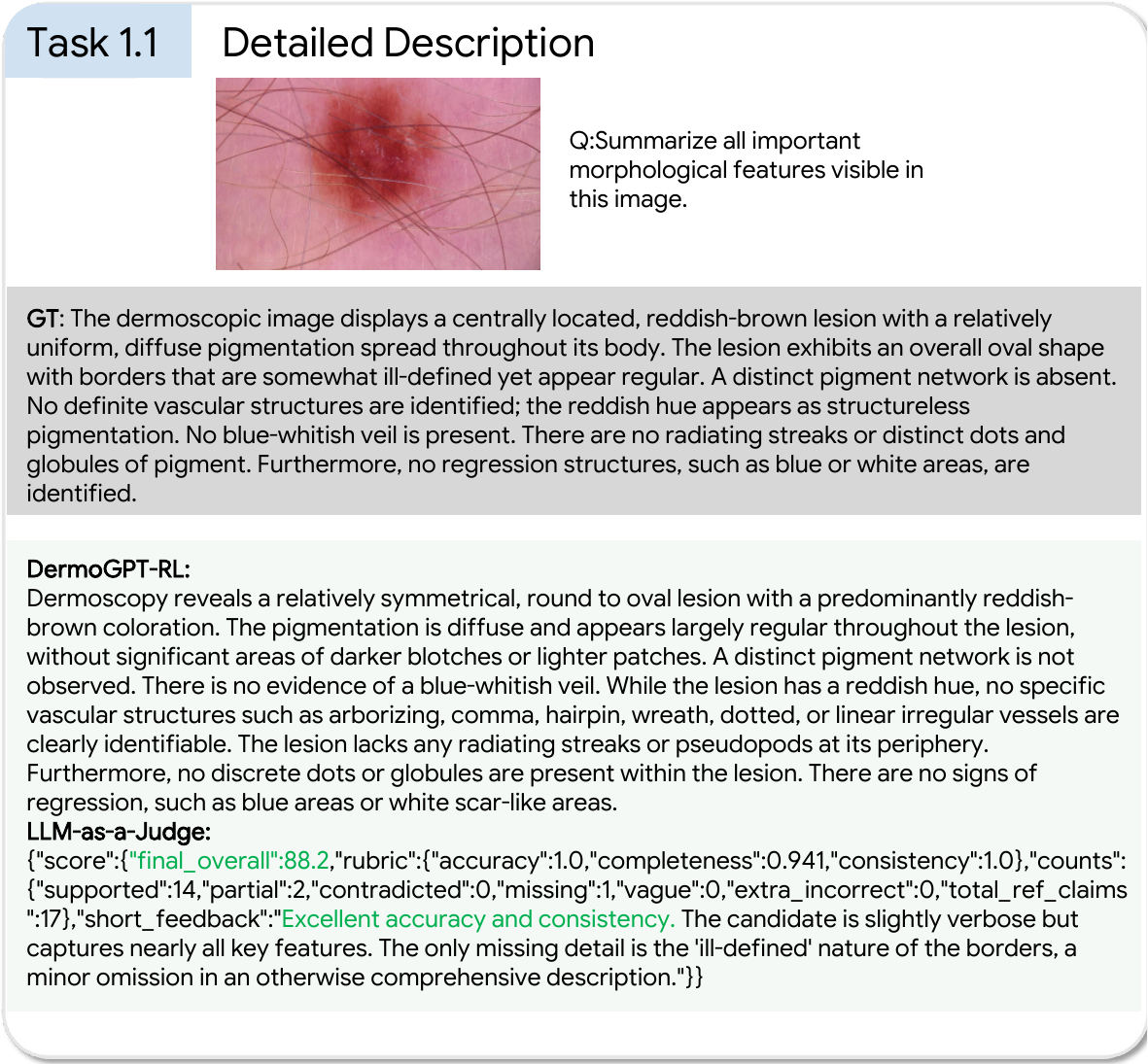}
  \caption{Case study.}
\end{figure*}
\begin{figure*}[ht]
  \centering
  
  \includegraphics[width=\linewidth]{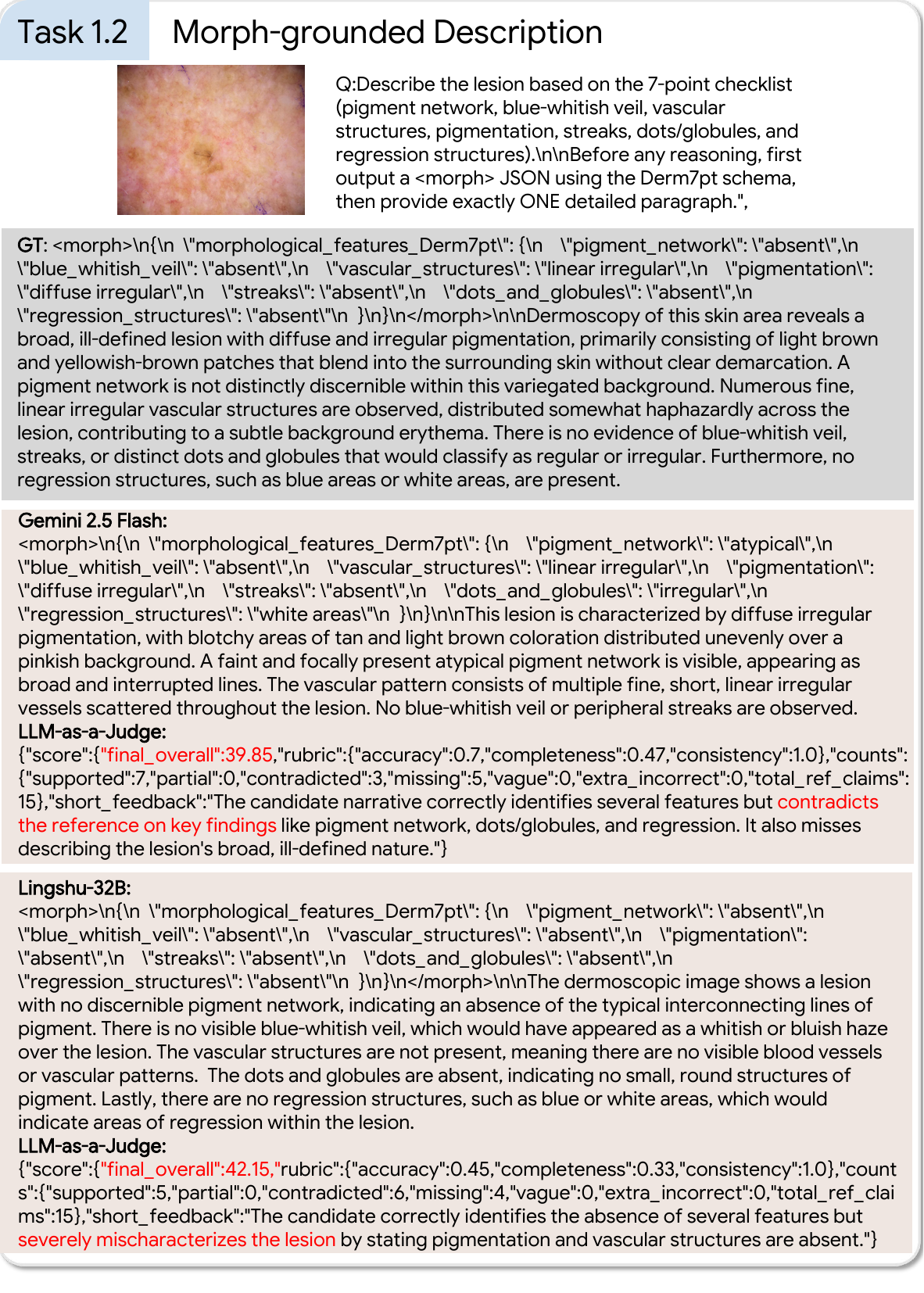}
  \caption{Case study.}
\end{figure*}
\begin{figure*}[ht]
  \centering
  
  \includegraphics[width=\linewidth]{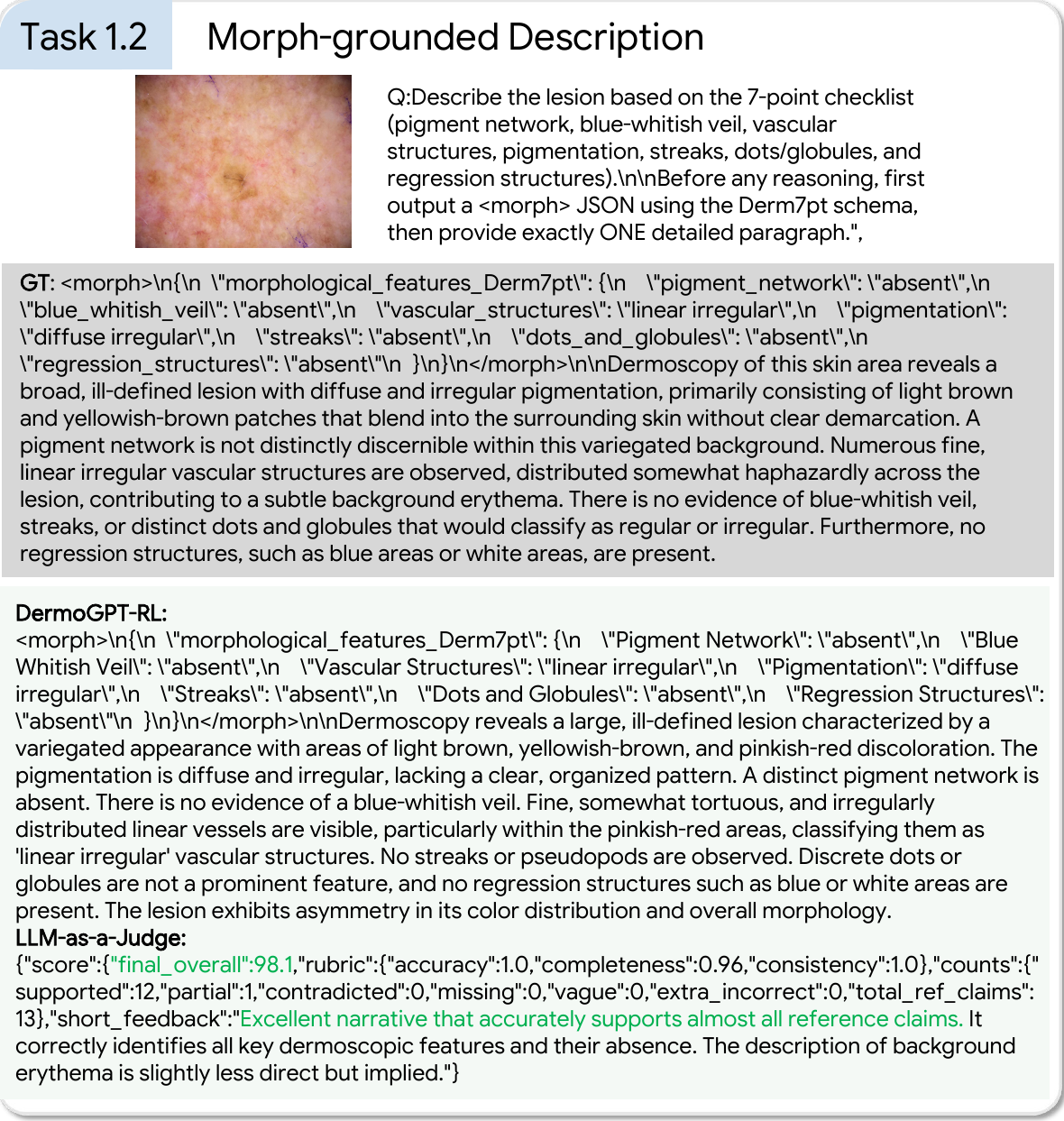}
  \caption{Case study.}
\end{figure*}
\begin{figure*}[ht]
  \centering
  
  \includegraphics[width=\linewidth]{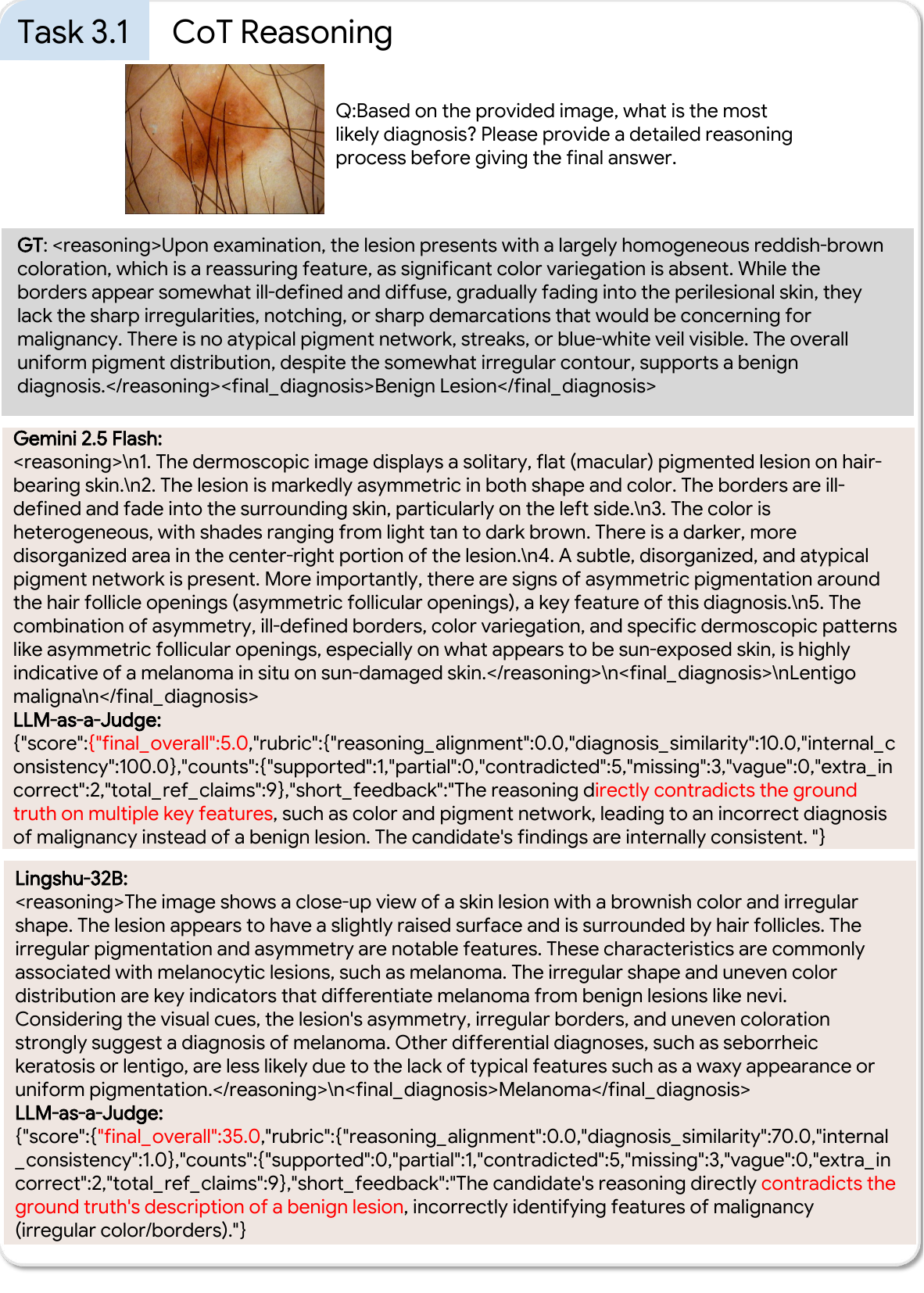}
  \caption{Case study.}
\end{figure*}
\begin{figure*}[ht]
  \centering
  
  \includegraphics[width=\linewidth]{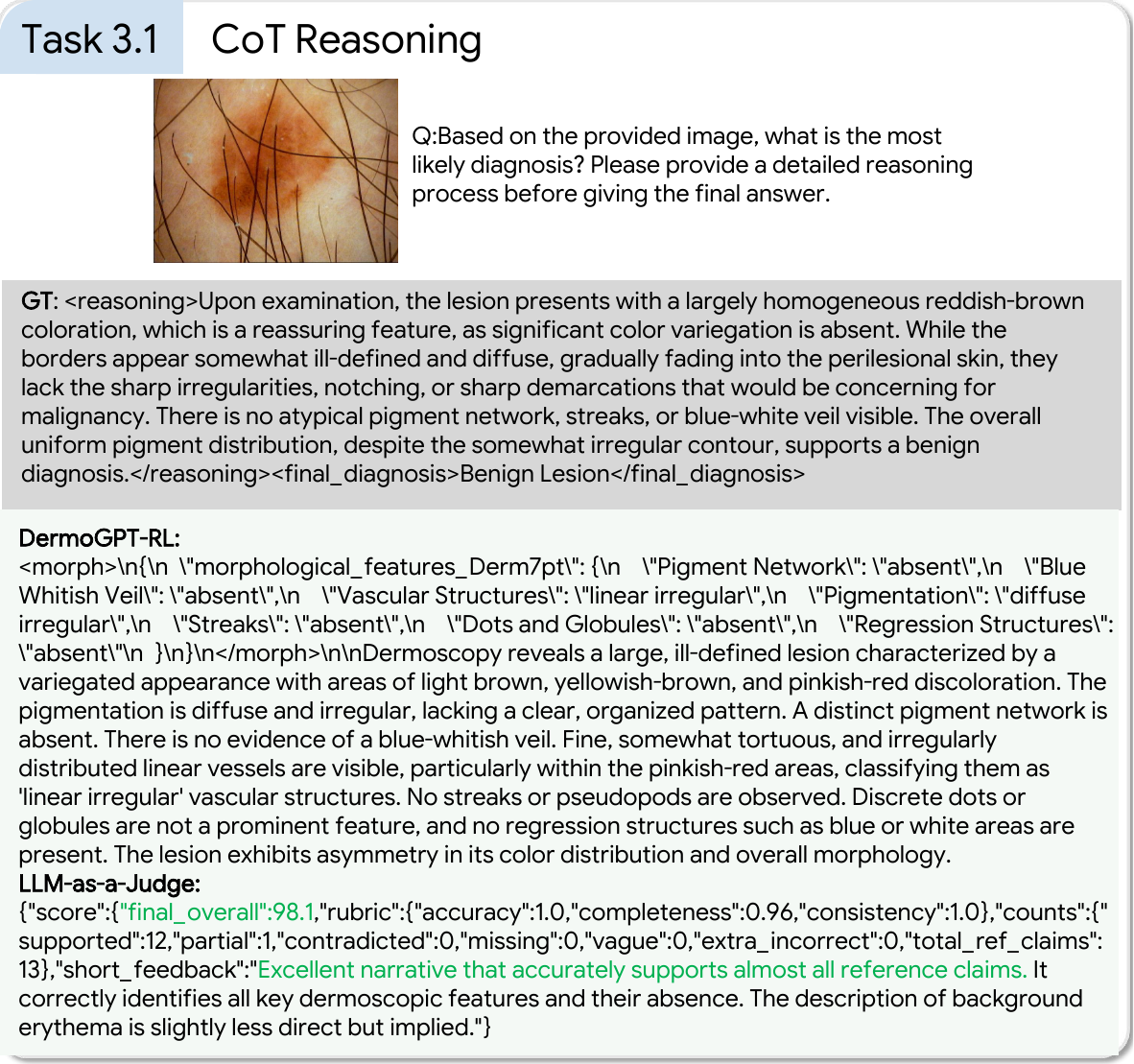}
  \caption{Case study.}
\end{figure*}
\begin{figure*}[ht]
  \centering
  
  \includegraphics[width=\linewidth]{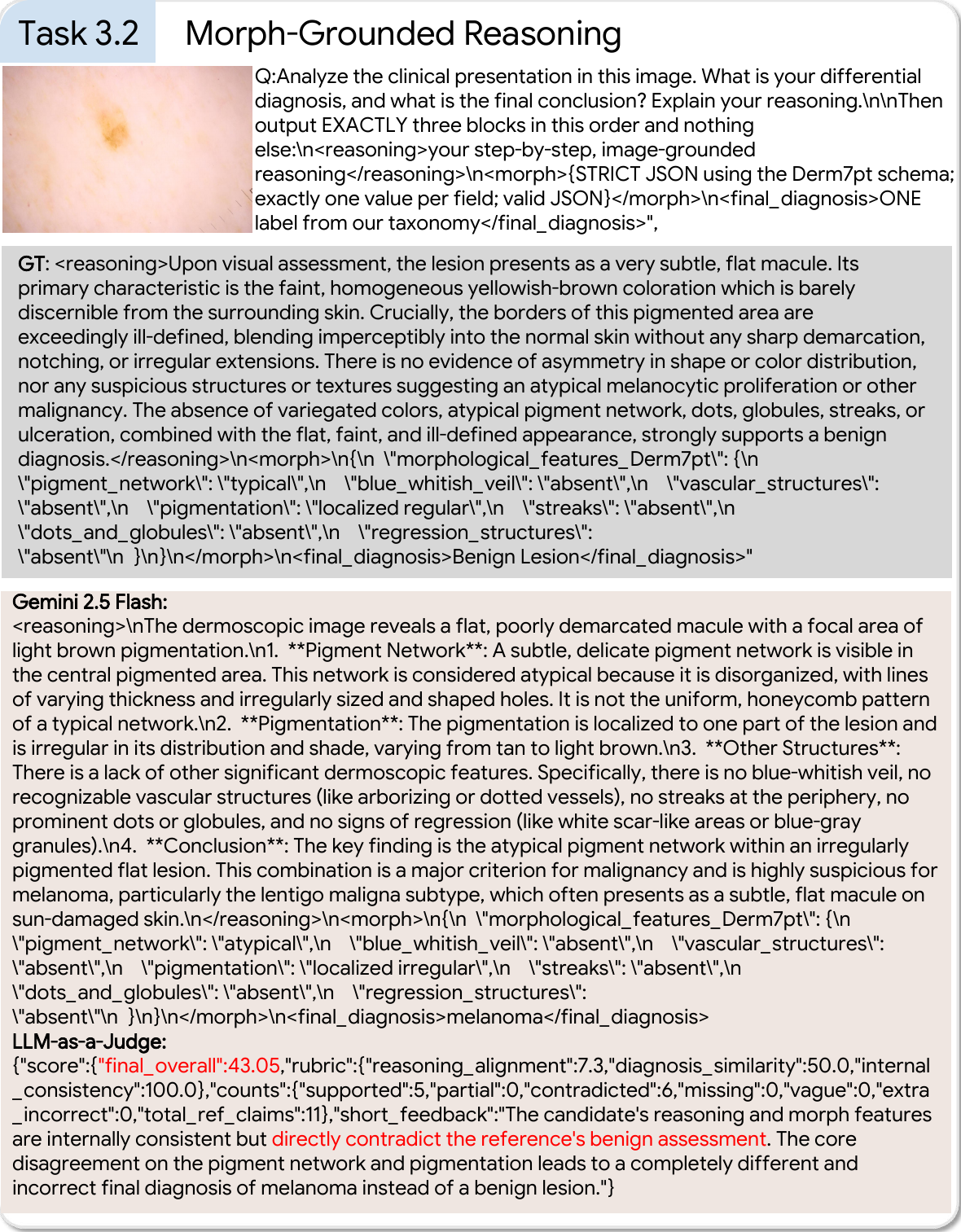}
  \caption{Case study.}
\end{figure*}
\begin{figure*}[ht]
  \centering
  
  \includegraphics[width=\linewidth]{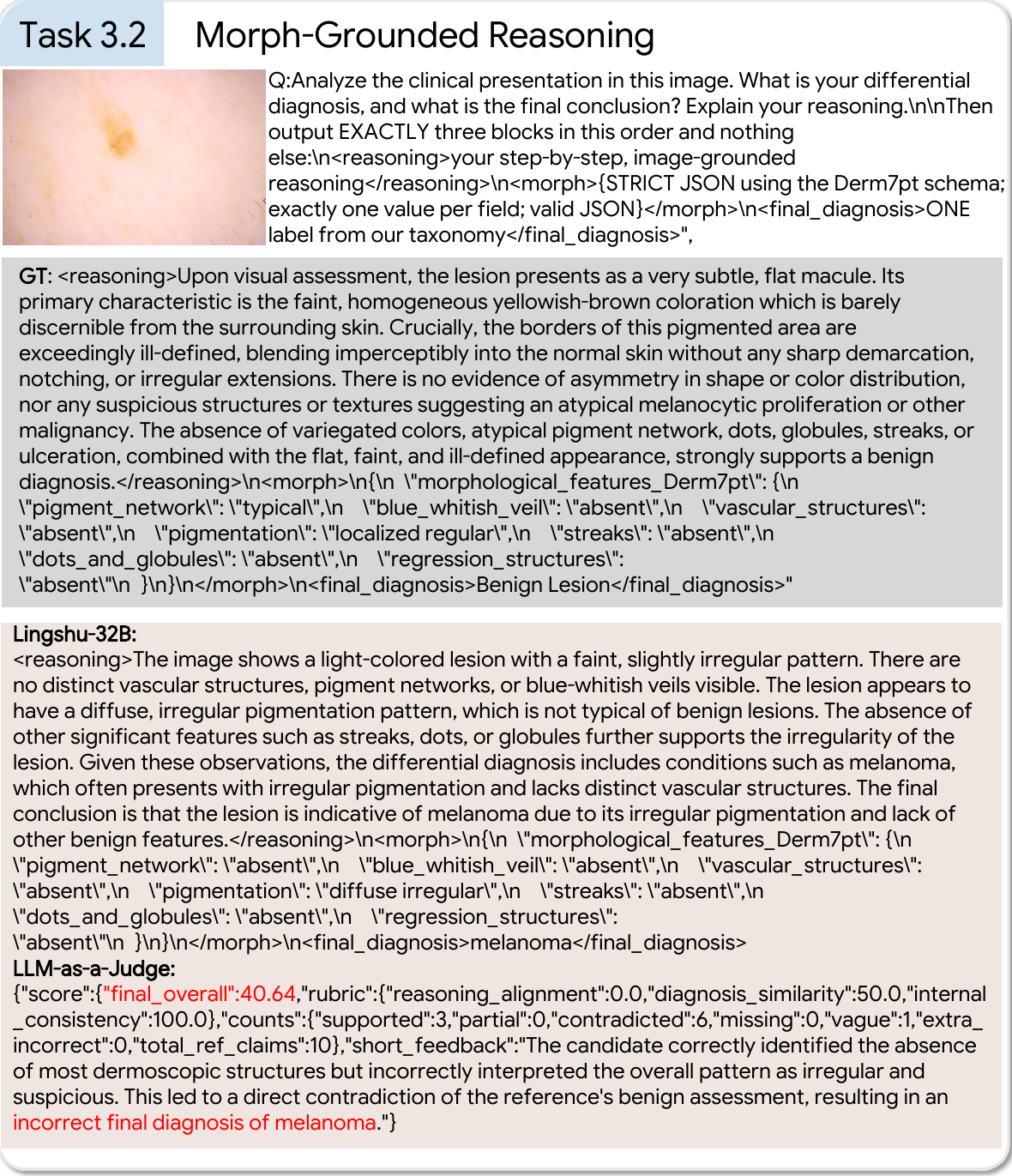}
  \caption{Case study.}
\end{figure*}
\begin{figure*}[ht]
  \centering
  
  \includegraphics[width=\linewidth]{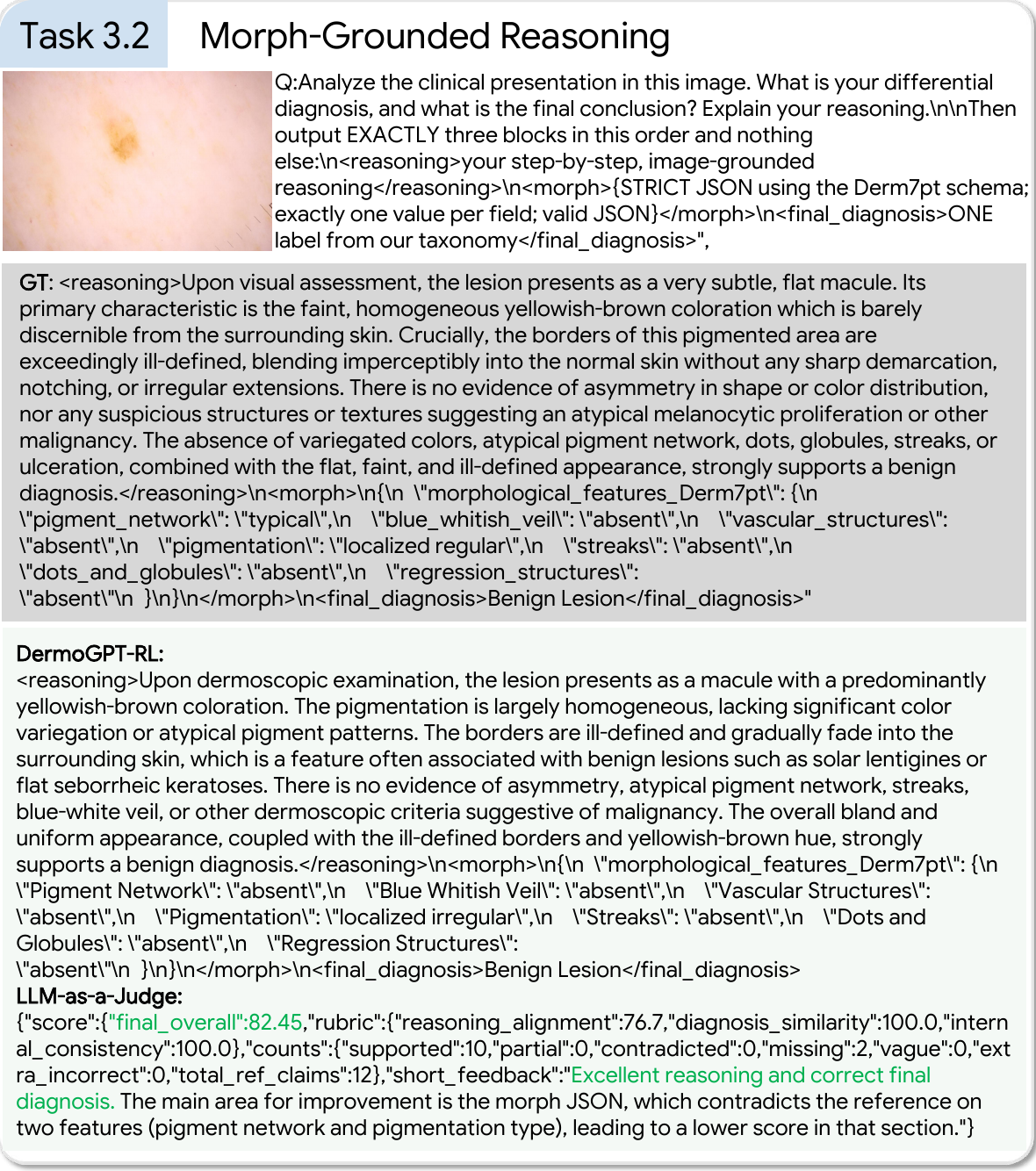}
  \caption{Case study.}
\end{figure*}

\end{document}